\documentclass[sn-mathphys]{sn-jnl}
\jyear{2021}%

\theoremstyle{thmstyleone}%
\newtheorem{theorem}{Theorem}%  meant for continuous numbers
\newtheorem{proposition}{Proposition}% 

\usepackage{booktabs}
\usepackage{arydshln}
\usepackage{array} % for m{..} column type (vertical centering in table cells)

\theoremstyle{thmstyletwo}%

\theoremstyle{thmstylethree}%

\usepackage{amsmath,amssymb} 
\usepackage[capitalize,noabbrev]{cleveref}
\crefname{section}{Sec.}{Secs.}
\crefname{figure}{Fig.}{Figs.}
\crefname{table}{Tab.}{Tabs.}
\crefname{equation}{Eq.}{Eqs.}
\crefname{algorithm}{Alg.}{Algs.}

\usepackage{color}
\usepackage{multirow}
\usepackage{graphicx}
\usepackage{comment}
\usepackage{changepage}
\usepackage{wrapfig}
\usepackage{colortbl} 
\usepackage{arydshln} 
\usepackage{amsthm}
\usepackage{bm}
\usepackage{rotating}

\definecolor{darkgreen}{RGB}{5,102,8}
\usepackage{subcaption}
\usepackage{overpic}
\usepackage{times}

\makeatletter
\let\svtiny\tiny
\def\tiny{\svtiny\footnotesize\color{gray}}
\makeatother

\usepackage{siunitx}
\newcommand{\std}[1]{{\tiny\normalfont
$\pm$#1}}
\newcommand{\stdp}[1]{\std{\phantom{0}#1}}

\newcommand{\myPara}[1]{%
  \noindent\textbf{#1}%
}

\begin{document}

% \textcolor{darkgreen}{xxxxxxxxxxxxxxxx}

\title[ ]{Brain-inspired hierarchical modularity for general continual learning}
%Brain-inspired hierarchical modularity for learning from dynamic experience


\author[1,2,3]{\fnm{Hongwei} \sur{Yan}}%\email{wly19@tsinghua.org.cn}
\equalcont{These authors contributed equally to this work.}

\author[2,4]{\fnm{Kanglei} \sur{Zhou}}%\email{xxzhang1993@gmail.com}
\equalcont{These authors contributed equally to this work.}

\author[2,4]{\fnm{Qi} \sur{Cheng}}

\author[2,4]{\fnm{Weiyi} \sur{Dong}}

\author[2,4]{\fnm{Chunyan} \sur{Lan}}

\author[1,2,3]{\fnm{Guanglong} \sur{Sun}}

\author[1,2,3]{\fnm{Jun} \sur{Zhou}}

\author[5]{\fnm{Qian} \sur{Li}}

\author[1,2,3]{\fnm{Yi} \sur{Zhong}}

\author*[2,4]{\fnm{Liyuan} \sur{Wang}}\email{liyuanwang@tsinghua.edu.cn}

\affil[1]{School of Life Sciences, Tsinghua University, Beijing, China}

\affil[2]{IDG/McGovern Institute for Brain Research, Tsinghua University, Beijing, China}

\affil[3]{Tsinghua-Peking Center for Life Sciences, Beijing, China}

\affil[4]{Department of Psychological and Cognitive Sciences, Tsinghua University, Beijing, China}

\affil[5]{Zhongshan School of Medicine, Sun Yat-sen University Shenzhen Campus, Shenzhen, China}


\abstract{
Continual learning, the ability to learn from sequential experience while retaining and adapting prior knowledge, is central to intelligent systems operating in changing environments. However, conventional continual learning is typically studied with offline task-wise training and clear task boundaries, leaving a substantial gap from general continual learning under online, uncertain, and evolving data streams. In this regime, intelligent systems must separate conflicting experience to reduce interference while integrating compatible experience to promote generalization. Inspired by the organization of the \emph{Drosophila} learning and memory system, we identify a hierarchical modular principle that coordinates both functions through expert specialization and ensemble integration. We instantiate this principle as lightweight modular adaptation of pretrained foundation models, combining brain-inspired random expansion for expert routing and diversified modular integration across spatial and temporal scales. Across visual recognition, vision-language understanding, ego-exo video understanding, and embodied vision-language-action learning, our method consistently improves learning under online and uncertain data streams, with gains exceeding 50 percentage points over replay-free alternatives in embodied manipulation. These findings support hierarchical modularity as a biologically grounded path for learning from dynamic experience.
}




\maketitle

\section{Introduction}\label{sec:intro}
%%%%%%%%%%%%%%%%%%%%%%%%%%%%%%%%%

Continual learning (CL)~\cite{wang2024comprehensive,de2021continual} is a defining process through which intelligence learns, develops, and accumulates knowledge over time. In biological organisms~\cite{davis2023learning,li2020connectome}, learning from sequential experience supports immediate responses to environmental change and progressive development throughout the lifespan, enabling long-term adaptation to changing conditions. A similar capability is increasingly central to artificial intelligence (AI): moving beyond intelligence acquired primarily from static, human-curated data requires systems that can continue to learn from their own experience, despite catastrophic forgetting~\cite{mcclelland1995there,wang2024comprehensive} and loss of plasticity~\cite{wang2021afec,dohare2024loss}. Recent perspectives on an ``era of experience''~\cite{silver2025era,lecun2022path,hughes2024open} envision increasingly general agents whose capabilities emerge through persistent interaction with the external world. Emerging directions on self-improving agents~\cite{zhang2026darwin} and test-time training~\cite{zweiger2026self, behrouz2026titans} similarly point towards systems that continue to refine their behaviour and internal knowledge after deployment.

Most existing AI studies, however, formulate CL in a simplified conventional regime, typically within a narrow task setting and with largely offline training, clear task boundaries, and auxiliary task identities~\cite{zhou2025asal,wang2025hide,wang2022coscl}. These assumptions have enabled substantial progress through synaptic regularization~\cite{kirkpatrick2017overcoming,wang2021afec}, memory replay~\cite{buzzega2020dark,zhou2024magr}, and dynamic architecture~\cite{wang2025hide,wang2023hierarchical}. Real-world experience instead arrives online under uncertain, overlapping, and evolving distributions across diverse models, modalities, and application scenarios. We consider this broader regime as \emph{general continual learning} (GCL)~\cite{de2021continual,moon2023online}. Its central challenge extends beyond retaining past knowledge to a more fundamental question: \textbf{how should learning be organized as data distributions evolve over time?} Conflicting experience should be separated to reduce interference, whereas compatible experience should be integrated to exploit shared structure and promote generalization.

Biological organisms naturally learn under such dynamic conditions, providing a useful reference for GCL. Among model organisms, \emph{Drosophila} is particularly tractable because its learning circuits combine rich adaptive behaviour with increasingly detailed anatomical characterization from whole-brain connectomes and cross-connectome cell typing~\cite{modi2020drosophila,davis2023learning,li2020connectome,winding2023connectome,lin2024network,schlegel2024whole}. In the olfactory learning and memory system, sparse, largely random projections expand sensory representations in Kenyon cells and support pattern separation~\cite{caron2013random,honegger2011cellular,aso2014mushroom,dasgupta2017neural}, while downstream learning and memory are distributed across differentiated compartments with distinct spatial and temporal characteristics~\cite{aso2014neuronal,aso2016dopaminergic,cohn2015coordinated,handler2019distinct,cervantes2013system}. Computationally, we relate these biological mechanisms to two classical paradigms of modular machine learning: \emph{mixture-of-experts} (MoE)~\cite{mu2025comprehensive,jacobs1991adaptive} and \emph{ensemble learning} (EL)~\cite{dong2020survey,hansen2002neural}. MoE promotes specialization across dissimilar distributions to reduce interference, whereas EL integrates diversified information over related distributions to promote generalization. Importantly, the \emph{Drosophila} system coordinates these MoE- and EL-like functions through a hierarchical learning and memory organization rather than deploying them independently~\cite{wang2023incorporating,wang2025convergent}. This hierarchical coordination provides a biological reference for jointly organizing specialization and integration in GCL.

Here we propose FlyGCL, a unified brain-inspired framework for GCL with pretrained foundation models. In \emph{Drosophila}, learning and memory operate downstream of relatively stable sensory processing. Analogously, FlyGCL retains the pretrained backbone as a stable representational substrate and organizes lightweight, parameter-efficient learning downstream. Brain-inspired random expansion of pretrained representations improves instance-level routing among specialized experts, while differentiated adaptive modules and multi-timescale predictions introduce complementary diversity across spatial and temporal dimensions for ensemble integration. These components realize hierarchical coordination between routing-based specialization and spatial-temporal integration. This modular design accommodates different lightweight learning modules and is broadly applicable across pretrained backbones and learning settings. Computational analyses further characterize the complementary gains of specialization and integration, and the benefit of organizing them over stable pretrained representations (Methods).

We evaluate FlyGCL across diverse forms of real-world continual experience, spanning visual recognition, vision-language understanding, ego-exo video understanding, and embodied vision-language-action learning, all under online and uncertain data streams (Supplementary~\cref{tab:protocol_details,tab:method_optimization_config}). FlyGCL consistently improves CL performance across these scenarios and pretrained models. The gains are particularly pronounced in embodied vision-language-action learning. Across spatial, object-centric, goal-conditioned, and long-horizon manipulation, FlyGCL achieves final average success rates of 83.1\%, 86.1\%, 94.5\%, and 79.1\%, respectively, exceeding the strongest replay-free CL baseline on each benchmark by 49.5--58.8 percentage points. Together, these results support hierarchical modularity as a biologically grounded principle for organizing learning from dynamic experience.

\begin{figure}
    \centering
    \begin{overpic}[width=\linewidth,clip,trim=0 0 0 0]{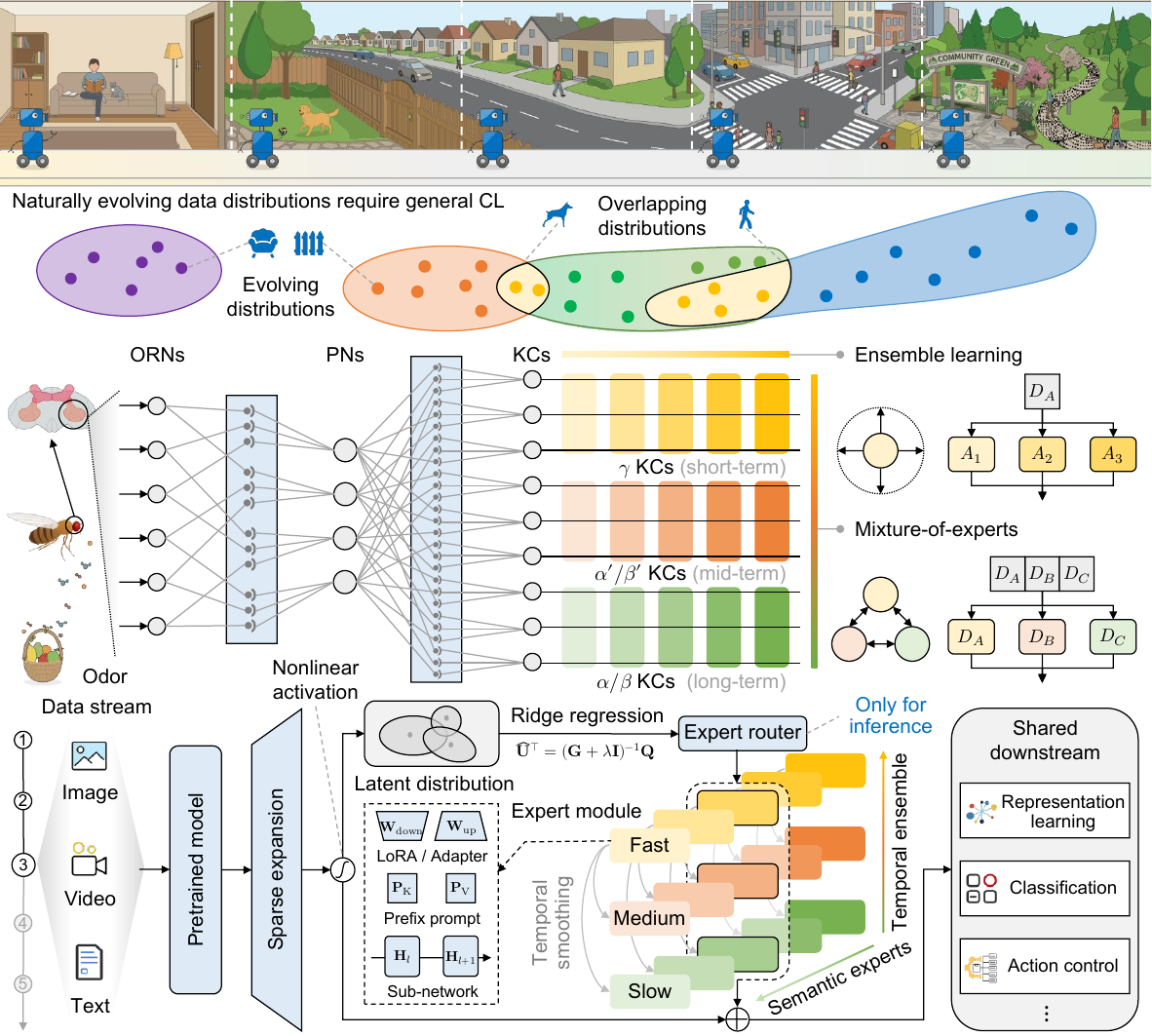}
        \put(0,88.5){\sf\footnotesize\textbf{a}}
        \put(0,59){\sf\footnotesize\textbf{b}}
        \put(0,28){\sf\footnotesize\textbf{c}}
    \end{overpic}
    \caption{
        \textbf{Brain-inspired hierarchical modular framework for general continual learning (GCL).}
        \subref{fig:teaser-a}: Real-world environments induce online and uncertain streams, where data distributions evolve and previously seen concepts may recur over time.
        \subref{fig:teaser-b}: Biological inspiration from the \textit{Drosophila} olfactory learning and memory system. Olfactory receptor neurons (ORNs) project to projection neurons (PNs) and are sparsely expanded into Kenyon cells (KCs), motivating hierarchical modular learning. The upper branch illustrates ensemble learning, where multiple readouts $A_1,A_2,A_3$ share the same distribution $D_A$, while the lower branch illustrates mixture-of-experts, where different distributions $D_A,D_B,D_C$ are assigned to specialized experts.
        \subref{fig:teaser-c}: Brain-inspired framework for GCL. Online data streams are processed by a pretrained model and adapted through sparse expansion with nonlinear activation, expert routing, and multi-timescale expert integration, where routing is solved by closed-form ridge regression in the latent space. The resulting representations support diverse downstream tasks. 
}
    \label{fig:teaser}
    \phantomsubcaption\label{fig:teaser-a}
    \phantomsubcaption\label{fig:teaser-b}
    \phantomsubcaption\label{fig:teaser-c}
\end{figure}

\section{Results}\label{sec:result}

We study \emph{general continual learning} (GCL) under online, uncertain, and evolving data streams across diverse models, modalities, and application scenarios. Its central challenge is to organize incoming experience by separating conflicting distributions to reduce interference while integrating compatible ones to promote generalization (\cref{fig:teaser-a}). FlyGCL addresses this challenge through brain-inspired hierarchical modularity. We first examine its biological and computational basis, and then evaluate its generality under perceptual and embodied learning scenarios.

\subsection{Biological and computational basis of hierarchical modular framework}

\begin{figure}
    \centering
    
    \begin{overpic}[width=\linewidth]{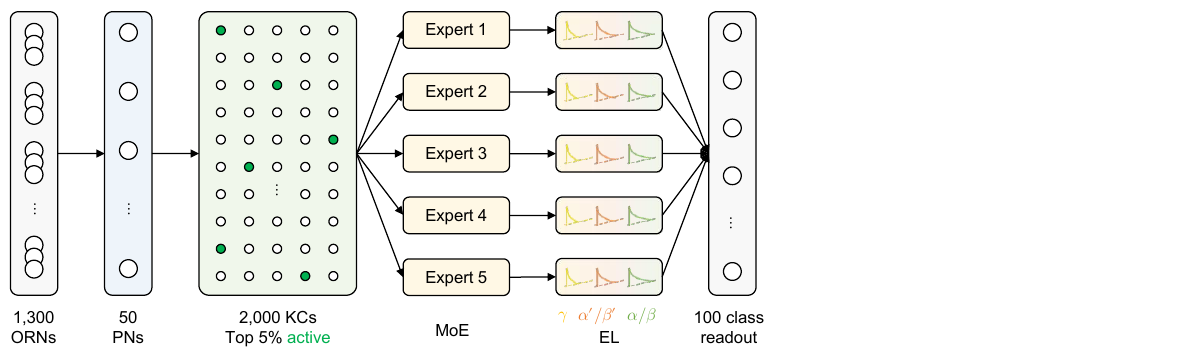}
        \put(66.5,0.5){\includegraphics[width=0.33\linewidth]{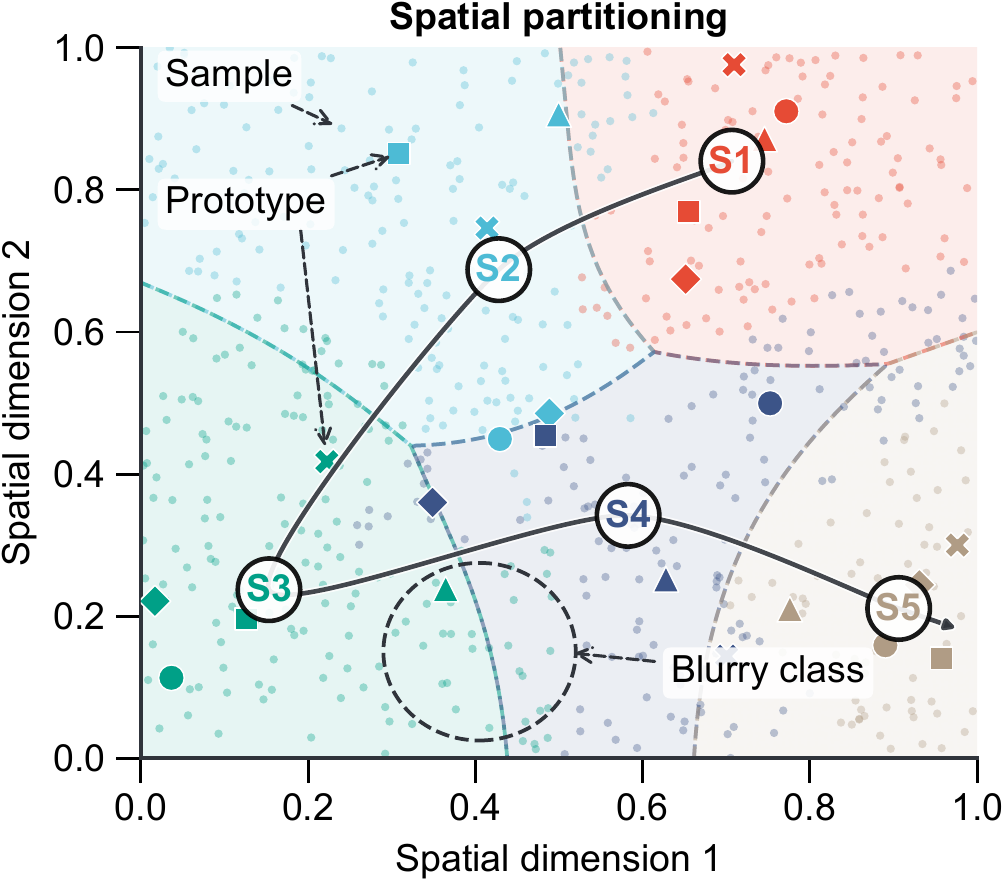}}
    \end{overpic}
    \\[2.0mm]
    \begin{overpic}[width=\linewidth]{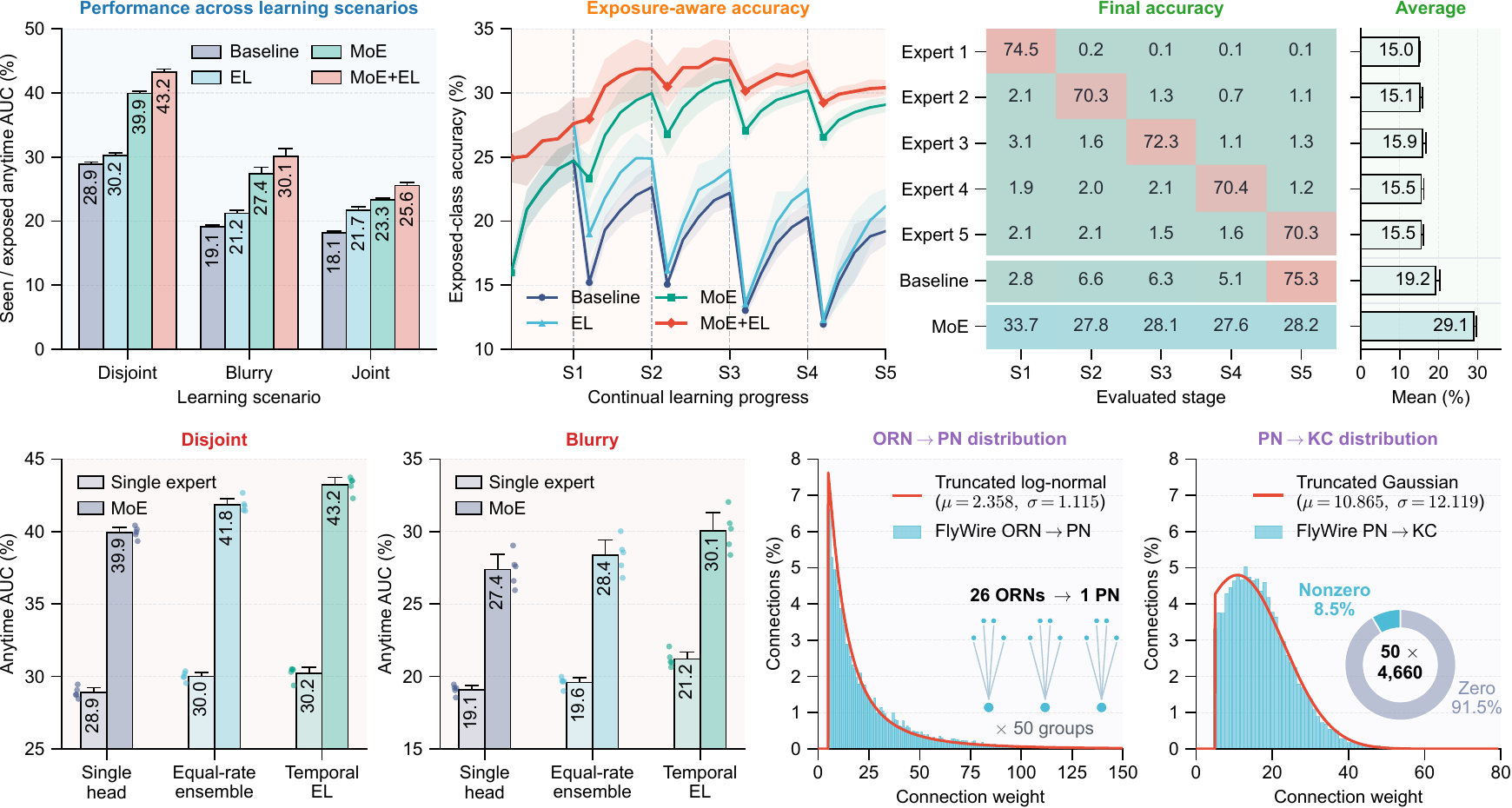}
        \put(-1.0,83){\sf\footnotesize\textbf{a}}
        \put(66,83){\sf\footnotesize\textbf{b}}
        \put(-1.0,53){\sf\footnotesize\textbf{c}}
        \put(29,53){\sf\footnotesize\textbf{d}}
        \put(58,53){\sf\footnotesize\textbf{e}}
        \put(-1.0,24){\sf\footnotesize\textbf{f}}
        \put(24,24){\sf\footnotesize\textbf{g}}
        \put(49,24){\sf\footnotesize\textbf{h}}
        \put(75,24){\sf\footnotesize\textbf{i}}
    \end{overpic}
    \caption{
    \textbf{Biologically grounded evaluation of brain-inspired hierarchical modularity.}
    \subref{fig:theory-a}: \textit{Drosophila}-inspired architecture with ORN-PN compression, sparse PN-KC expansion, spatially differentiated experts, and temporal ensemble learning. ORNs, olfactory receptor neurons; PNs, projection neurons; KCs, Kenyon cells; MoE, mixture-of-experts; EL, ensemble learning.
    \subref{fig:theory-b}: Spatial partitioning of 100 classes into five regions according to prototype proximity.
    \subref{fig:theory-c}: Area under the exposed-class anytime accuracy curve (AUC) of the baseline, EL, MoE, and MoE+EL under disjoint, blurry, and joint settings.
    \subref{fig:theory-d}: Exposed-class accuracy of the four methods over the blurry setting.
    \subref{fig:theory-e}: Final accuracy of individual experts, baseline, and MoE across five spatial regions.
    \subref{fig:theory-f} and \subref{fig:theory-g}: Anytime AUC with a single head, equal-rate ensemble, and temporal EL, with or without MoE, under disjoint and blurry settings, respectively.
    \subref{fig:theory-h}: FlyWire ORN-PN connection weights and their truncated log-normal approximation; inset, 50 groups with 26 ORNs converging onto one PN.
    \subref{fig:theory-i}: FlyWire PN-KC connection weights and their truncated Gaussian approximation; inset, connection sparsity.
    All results are averaged over five independent runs; error bars and shaded areas denote 95\% confidence intervals.
    }
    \label{fig:theory}
    \phantomsubcaption\label{fig:theory-a}
    \phantomsubcaption\label{fig:theory-b}
    \phantomsubcaption\label{fig:theory-c}
    \phantomsubcaption\label{fig:theory-d}
    \phantomsubcaption\label{fig:theory-e}
    \phantomsubcaption\label{fig:theory-f}
    \phantomsubcaption\label{fig:theory-g}
    \phantomsubcaption\label{fig:theory-h}
    \phantomsubcaption\label{fig:theory-i}
\end{figure}
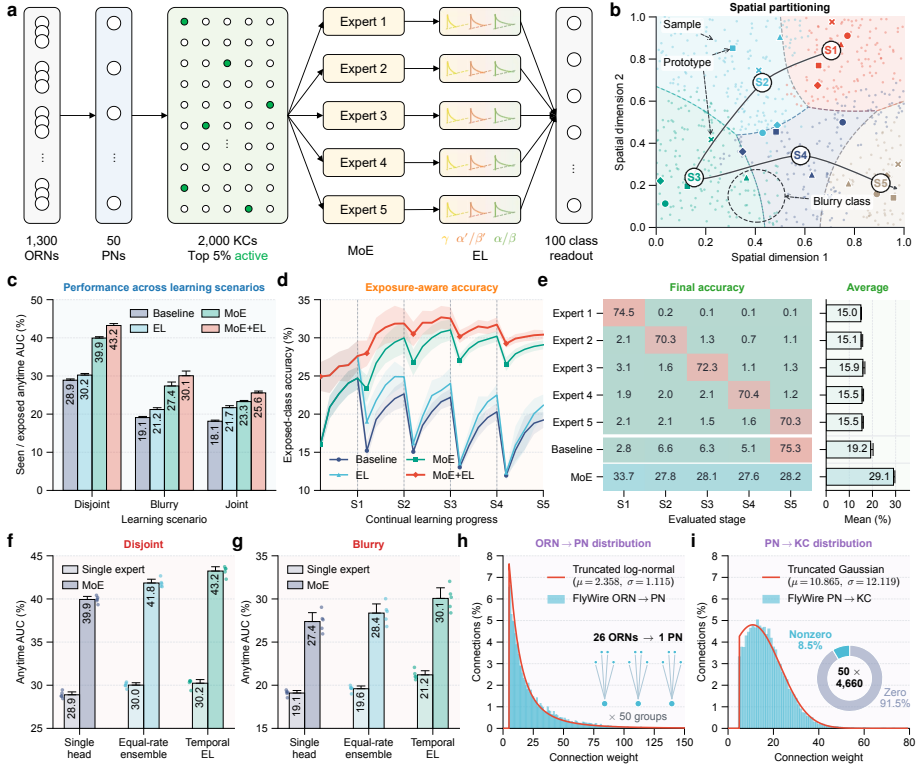

The \emph{Drosophila} olfactory system provides a compact biological model of learning from continuously varying sensory experience (\cref{fig:teaser-b}). Odor signals are encoded by around 1300 olfactory receptor neurons (ORNs) of 50 groups and around 150 projection neurons (PNs) of 50 types~\cite{wang2021evolving}, and then transmitted through sparse, largely random PN-KC projections into a substantially expanded population of around 2,000 Kenyon cells (KCs)~\cite{caron2013random,honegger2011cellular,fulton2024common}. This transformation produces sparse, distributed representations that reduce overlap between odor patterns and support pattern separation~\cite{aso2014mushroom,dasgupta2017neural}. Downstream of this relatively stable sensory representation, learning and memory are organized through the $\gamma$, $\alpha'/\beta'$, and $\alpha/\beta$ KC lobes, each partitioned into five anatomically differentiated compartments (a total of 15 modules across three lobes) regulated by distinct dopaminergic and output pathways~\cite{aso2014neuronal,aso2016dopaminergic,cohn2015coordinated}. These compartments provide spatial diversification within each lobe, while the three lobes contribute preferentially to short-term memory, intermediate memory consolidation, and long-term memory, respectively~\cite{handler2019distinct,cervantes2013system}. The mushroom body therefore organizes learning and memory along two complementary axes: spatial compartmentalization within KC lobes and temporal differentiation across them.

We interpret this biological organization computationally as a hierarchy of two complementary paradigms in modular machine learning (\cref{fig:teaser-b,fig:teaser-c}). At the first level, sparse random expansion separates sensory representations and enables selective recruitment of downstream pathways, providing a biological analogue of \emph{mixture-of-experts} (MoE) routing for reducing interference between dissimilar distributions. At the second level, differentiated compartments provide parallel spatial memory pathways, while the $\gamma$, $\alpha'/\beta'$, and $\alpha/\beta$ lobes span progressively longer memory timescales. Coordinating information across these spatial and temporal dimensions resembles \emph{ensemble learning} (EL), which exploits diversity and integration to improve generalization over related distributions. Computationally, MoE-like specialization separates conflicting experience, whereas EL-like integration combines compatible information across spatial and temporal memory components. Their hierarchical coordination yields the central principle of GCL: separating conflicting experience while integrating compatible experience.

FlyGCL instantiates this biological organization in pretrained foundation models (\cref{fig:teaser-c}, Methods). The pretrained backbone serves as a relatively stable representational substrate, analogous to upstream sensory processing, while brain-inspired random expansion of its representations enables instance-level routing among specialized experts. Diversified adaptive components provide spatial variation, whereas prediction heads with different effective memory windows provide temporal variation in each expert. Fast predictions emphasize recent experience, slow predictions preserve information over longer timescales, and intermediate heads bridge the two, forming a computational analogue of short-term, consolidation-related, and long-term memory. This hierarchical design can be implemented with lightweight adaptation such as prompts, adapters, or LoRA~\cite{lester2021power, rebuffi2017learning, hu2021lora}, allowing FlyGCL to operate across different pretrained models and learning scenarios. Computational analysis further supports the complementary roles of specialization and integration, and the benefit of coordinating them on stable pretrained representations (\cref{fig:teaser-c}, Methods).

To examine the hierarchical modular principle in a controlled biologically grounded setting, we construct a continual olfactory learning model that follows the population scale and modular organization of the \emph{Drosophila} olfactory system (\cref{fig:theory-a}, Supplementary~\cref{sec:app_bio_simulation}). Following prior task-driven models of olfactory learning~\cite{wang2021evolving,shen2023reducing}, the network contains 1,300 ORNs, 50 types of PNs, and 2,000 KCs. ORN-PN and PN-KC connections are sampled using statistics derived from FlyWire~\cite{dorkenwald2022flywire} (\cref{fig:theory-h,fig:theory-i}), and only the 5\% most active KCs are retained for each odor. The sensory pathway remains fixed, restricting CL to downstream learning and memory components. We organize these components along the same spatial-temporal hierarchy: five parallel experts represent memory pathways differentiated by spatial regions, while three temporal heads capture fast, intermediate, and slow effective memory scales. We generate 100 classes from prototypes in a 50-dimensional sensory space and assign nearby prototypes to five equally sized spatial regions. Each online data stream contains 50,000 unique samples presented over five stages, ranging from strictly separated classes in Disjoint to increasing cross-stage overlap in Blurry and Joint (50\% and 100\% classes overlap, respectively).

We compare four targeted baselines that isolate the two dimensions of this organization: a naive baseline with a single adaptive pathway; MoE with five spatially differentiated experts and an expert router based on accumulated stage prototypes; EL with three prediction heads operating at distinct effective timescales; and the hierarchical model combining five experts with three temporal heads each (\cref{fig:theory-a}). Across stream configurations, MoE and EL provide complementary benefits, while their hierarchical combination consistently performs best (\cref{fig:theory-c,fig:theory-d}). Under the disjoint setting, the area under the curve (AUC) performance increases from 28.9\% for the baseline to 39.9\% with MoE and 43.2\% with MoE+EL. Under the blurry setting, MoE+EL reaches 30.1\%, compared with 19.1\% for the baseline, and remains consistently stronger over the course of learning. The same ordering holds under the joint setting, where MoE+EL also outperforms either component alone.

The expert analysis provides direct evidence of specialization (\cref{fig:theory-e}). Each expert is most accurate in its corresponding region, whereas its accuracy is low elsewhere. Routing these specialized experts produces a mean final accuracy of 29.1\% across regions, compared with 19.2\% for the shared baseline. Temporal diversity provides an additional consistent gain (\cref{fig:theory-f,fig:theory-g}). For example, under the blurry setting with MoE, anytime AUC increases from 27.4\% with a single head to 28.4\% with three equal-rate heads and 30.1\% when the heads use different learning rates. This progression holds both with and without MoE in the disjoint and blurry settings, supporting distinct contributions from expert specialization and temporal integration.

\subsection{General Continual Learning for Visual and Vision-Language Perception} %700

Real-world intelligent systems continuously encounter changing perceptual experience, from evolving visual concepts to multimodal observations grounded in language. Continual visual recognition and continual vision-language learning therefore provide representative scenarios for studying how models preserve shared structure while adapting to distribution-specific changes over time. Although both scenarios have been widely studied in conventional CL, existing efforts largely rely on offline and disjoint task sequences. We revisit them under more realistic online and blurry data streams.

\begin{figure}[!h]
    \centering
    \vspace{-0.2cm}
    \begin{overpic}[width=\linewidth,clip,trim=0 450 0 0]{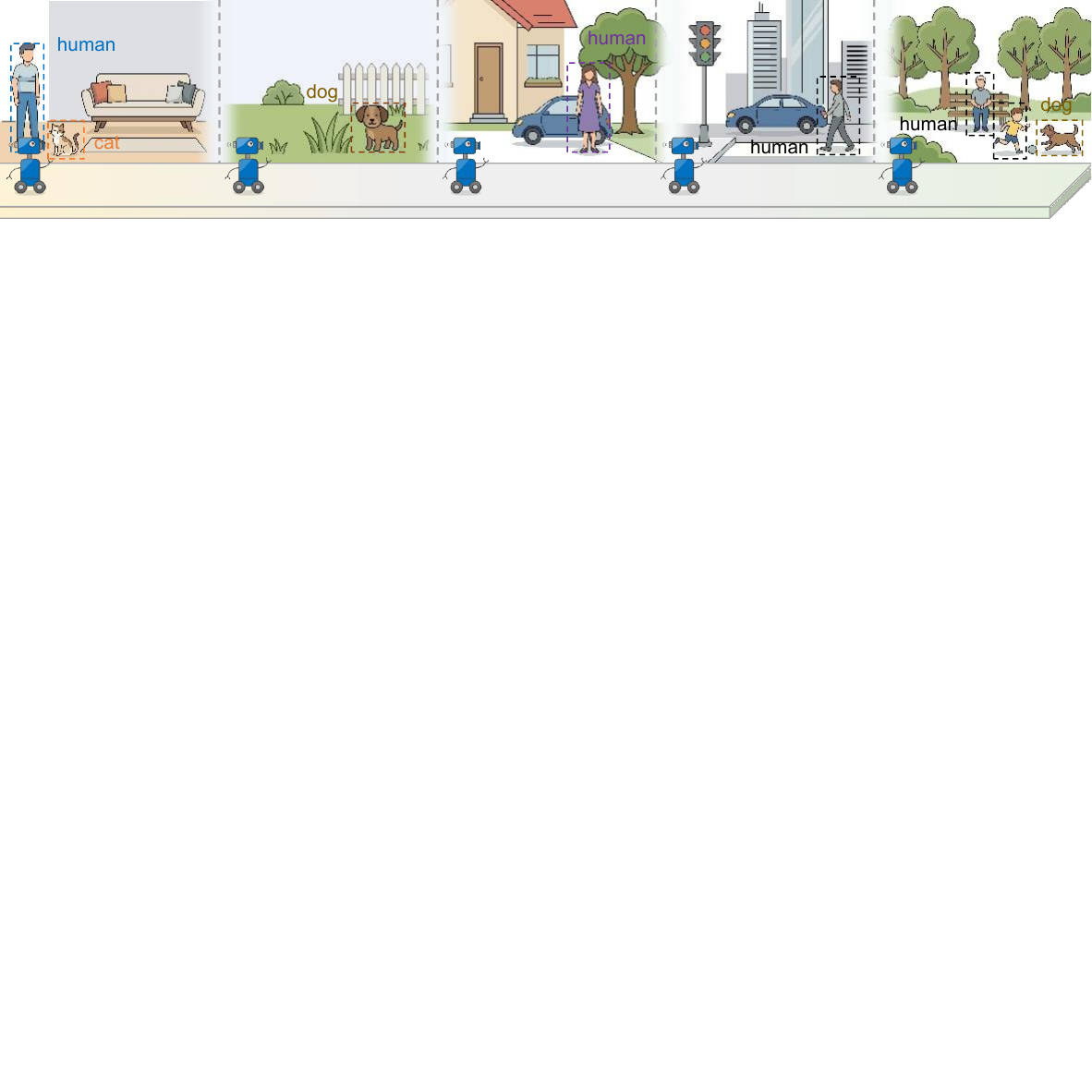}
    \end{overpic}
    % \begin{overpic}[width=\linewidth]{figs/image_hw.pdf}
    % \end{overpic}
    \\[1.5mm]
    \begin{overpic}[width=\linewidth]{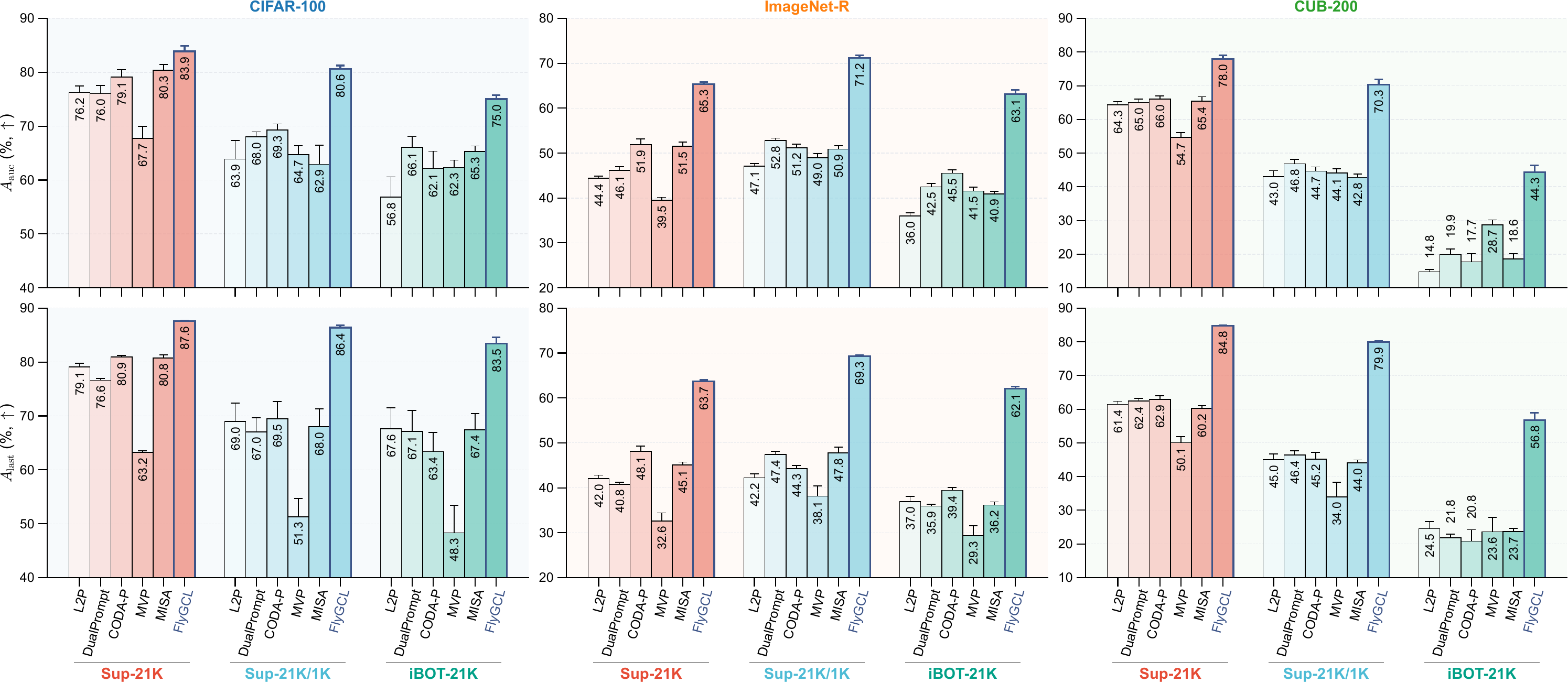}
    \end{overpic}
    \\[1.5mm]
    \begin{overpic}[width=\linewidth]{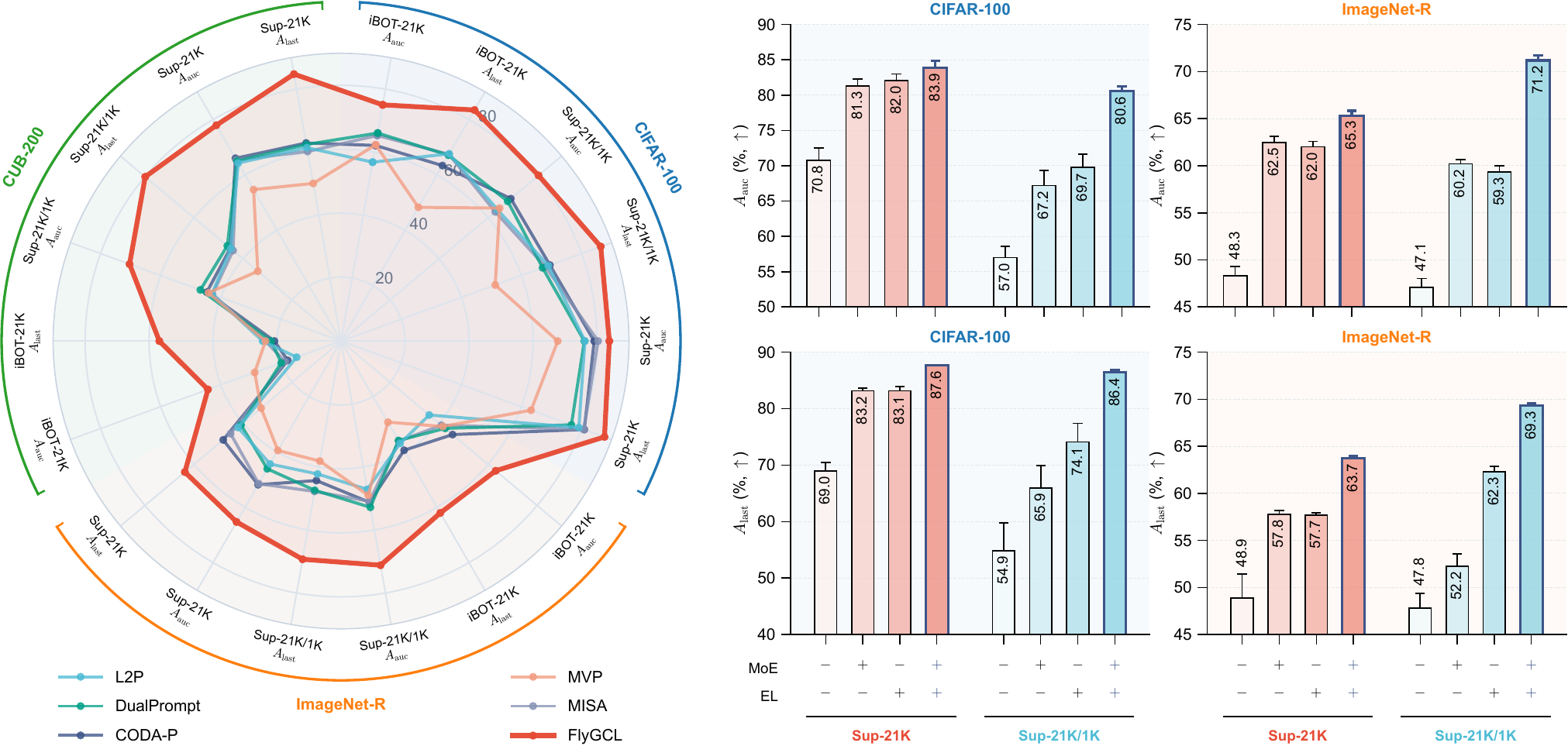}
        \put(0,112){\sf\footnotesize\textbf{a}}
        \put(0,91.75){\sf\footnotesize\textbf{b}}
        \put(0,45.5){\sf\footnotesize\textbf{c}}
        \put(44,45.5){\sf\footnotesize\textbf{d}}
    \end{overpic}
    \caption{
    \textbf{Results on continual visual recognition benchmarks.}
    \subref{fig:img-a}: Illustration of real-world vision scenarios, where agents continuously encounter diverse objects across evolving sessions.
    \subref{fig:img-b}: Overall performance comparison with state-of-the-art baselines.
    \subref{fig:img-c}: Unified performance summary across datasets, pretrained models, and evaluation metrics.
    \subref{fig:img-d}: Ablation analysis of mixture-of-experts (MoE) and ensemble learning (EL). 
    In \subref{fig:img-b}--\subref{fig:img-d}, FlyGCL uses LoRA modules as adaptive experts. All results are averaged over five independent runs; error bars indicate the standard error of the mean. Complete results for prompt-, adapter-, and LoRA-based instantiations are reported in Supplementary~\cref{tab:img,tab:img_ablation}.
    }
    \label{fig:img}
    \phantomsubcaption\label{fig:img-a}
    \phantomsubcaption\label{fig:img-b}
    \phantomsubcaption\label{fig:img-c}
    \phantomsubcaption\label{fig:img-d}
\end{figure}

\myPara{Continual Visual Recognition.}
We first evaluate FlyGCL on continual visual recognition under online and blurry data streams, where samples from newly introduced and previously observed classes are probabilistically interleaved~\cite{moon2023online,kang2025advancing}, producing uncertain and evolving distributions over time (\cref{fig:img-a}). We consider CIFAR-100~\cite{krizhevsky2009learning}, ImageNet-R~\cite{hendrycks2021many}, and CUB-200~\cite{wah2011caltech} datasets, spanning generic object recognition, distribution-shifted concepts, and fine-grained categories. Comparisons include representative pretrained-based CL methods, such as L2P~\cite{wang2022learning}, DualPrompt~\cite{wang2022dualprompt}, and CODA-Prompt~\cite{smith2023coda}, as well as online CL methods MVP~\cite{moon2023online} and MISA~\cite{kang2025advancing}. We report average anytime performance ($A_{\mathrm{auc}}$) and final average performance ($A_{\mathrm{last}}$). Across these benchmarks, FlyGCL consistently achieves the strongest final and anytime performance (\cref{fig:img-b,fig:img-c}, Supplementary~\cref{tab:img}): in the primary comparison, its $A_{\mathrm{auc}}$/$A_{\mathrm{last}}$ exceed the strongest baseline by 3.5\%/6.9\%, 13.8\%/18.7\%, and 12.6\%/24.6\% on CIFAR-100, ImageNet-R, and CUB-200, respectively. The performance gains are particularly clear on ImageNet-R and CUB-200, where distribution shifts and fine-grained distinctions place greater demands on selective adaptation.

We next test whether this advantage depends on the pretrained representation or the adaptation interface. The primary comparison covers three backbone settings: a model pretrained on ImageNet-21K (Sup-21K), a model pretrained on ImageNet-21K and subsequently adapted to ImageNet-1K (Sup-21K/1K)~\cite{russakovsky2015imagenet,ridnik2021imagenet,dosovitskiy2020image}, and a self-supervised iBOT model pretrained on ImageNet-21K (iBOT-21K)~\cite{zhou2021image} (\cref{fig:img-b}). We additionally evaluate self-supervised checkpoints pretrained on ImageNet-1K, including iBOT~\cite{zhou2021image}, DINO~\cite{caron2021emerging}, and MoCo v3~\cite{chen2021empirical} (Supplementary~\cref{tab:img}). Across these settings, FlyGCL supports prompt-, adapter-, and LoRA-based adaptation~\cite{lester2021power,li2021prefix,rebuffi2017learning,hu2021lora}, and retains strong performance across the resulting combinations. This consistency shows that the proposed hierarchical modularity is not tied to a particular pretrained representation or tuning interface.

Finally, we isolate the roles of the two components of FlyGCL design. Removing either MoE or EL reduces performance across datasets and pretrained representations, whereas their combination performs best (\cref{fig:img-d}, Supplementary~\cref{tab:img_ablation}). Expert routing provides differentiated adaptation paths for heterogeneous visual distributions, while temporal ensemble integration stabilizes predictions as related experience recurs. These results support that hierarchical specialization and integration provide complementary gains for continual visual recognition.

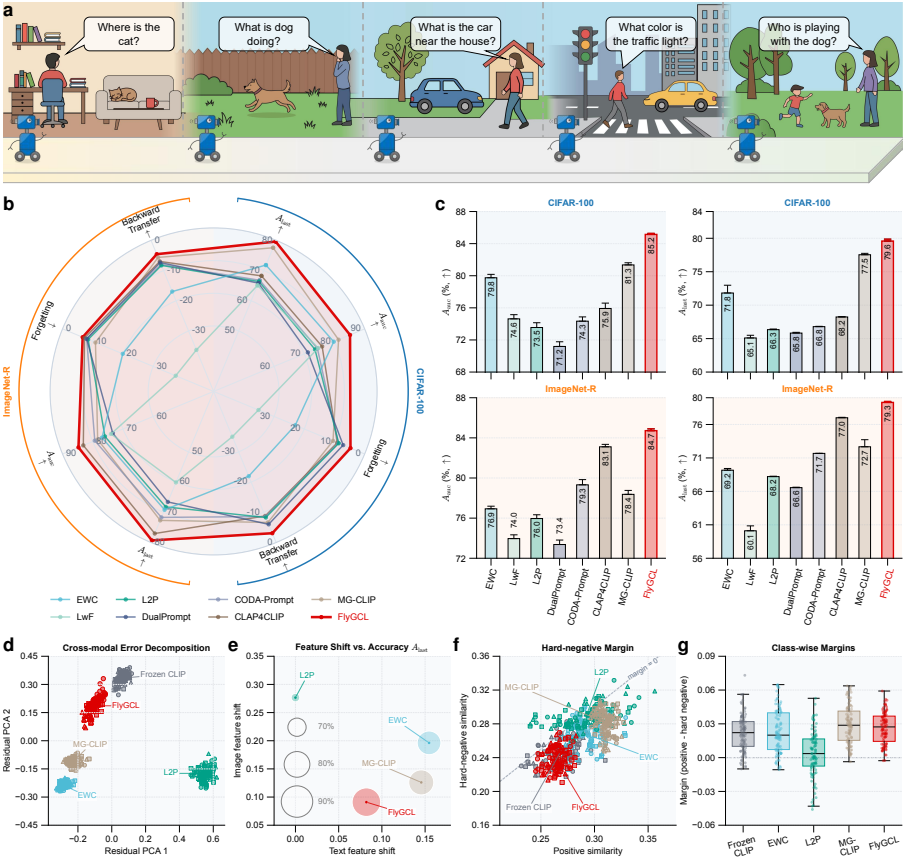
\begin{figure}
    \centering
    \vspace{-0.2cm}
    \begin{overpic}[width=\linewidth,clip,trim=0 450 0 0]{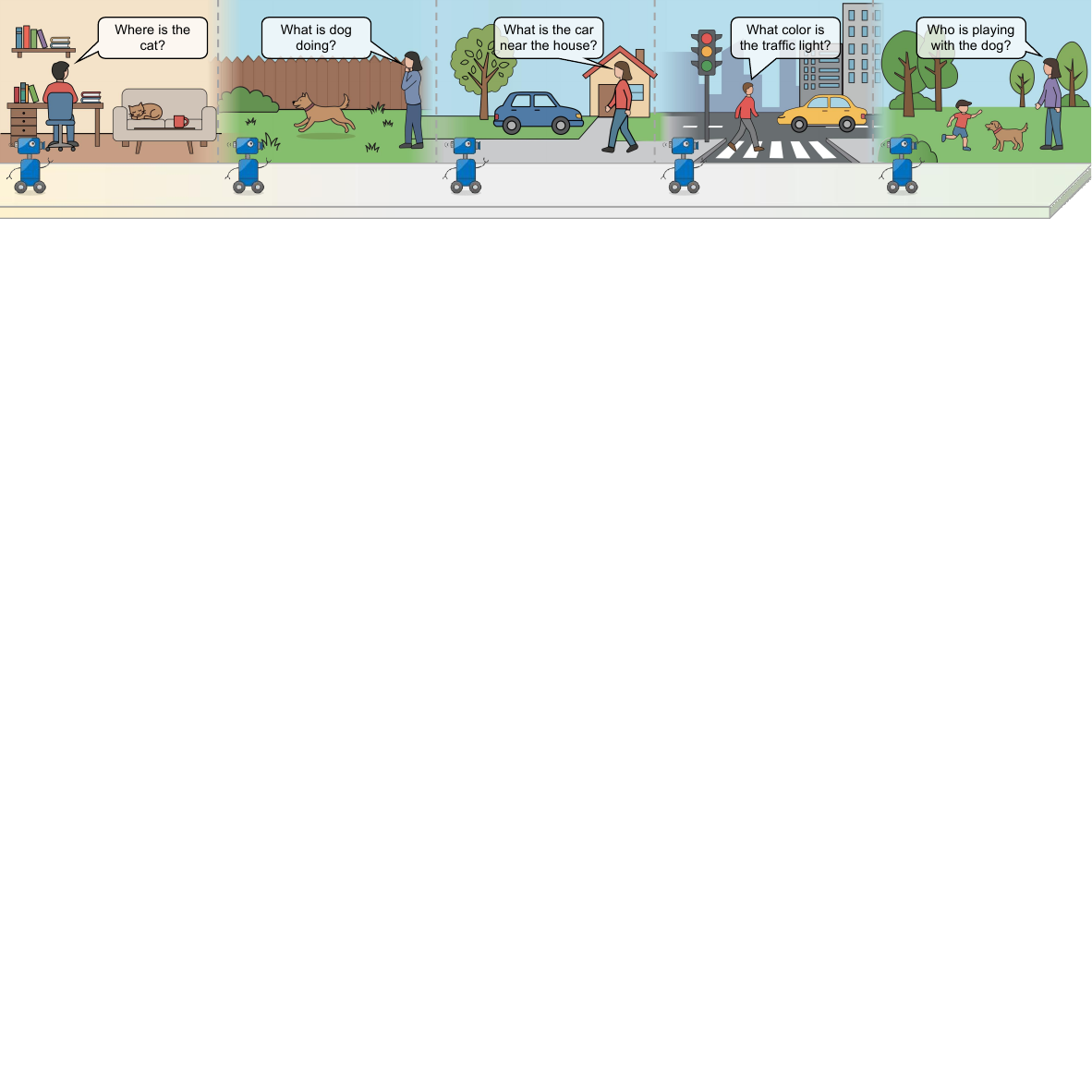}
    \end{overpic}
    \\[1.5mm]
    \begin{overpic}[width=\linewidth]{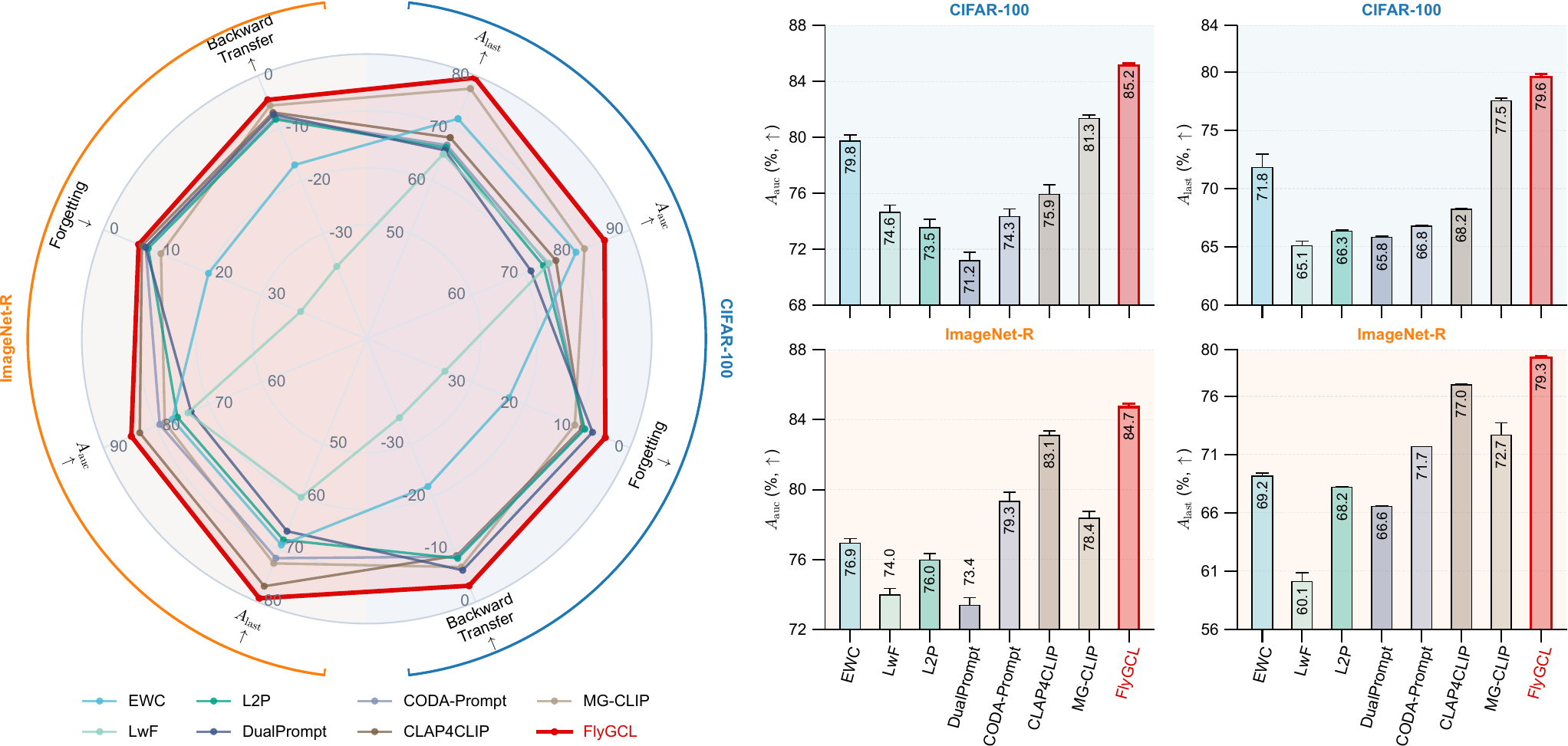}
    \end{overpic}
    \\[3.0mm]
    \begin{overpic}[width=\linewidth]{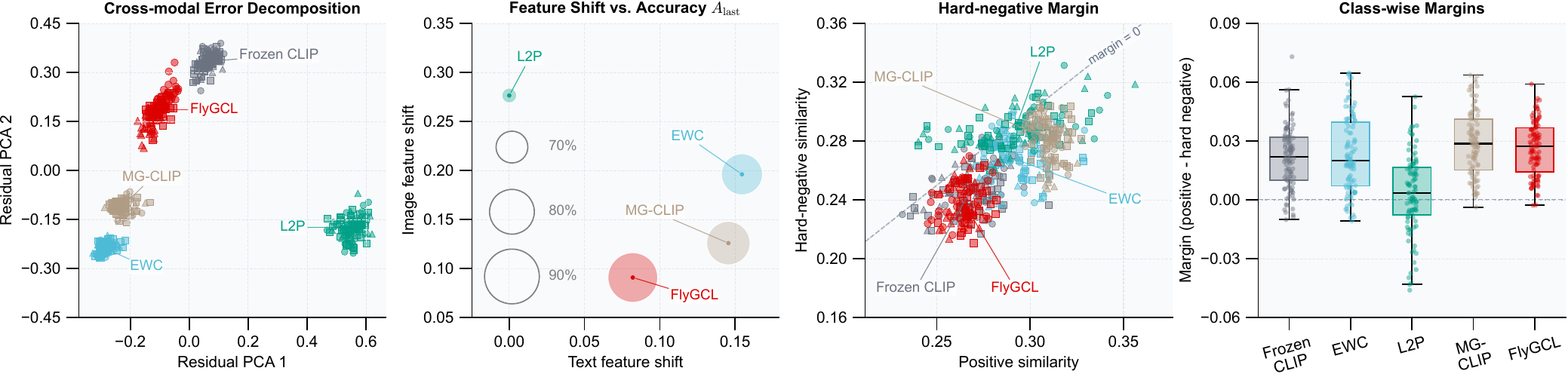}
    \put(0,93.5){\sf\footnotesize\textbf{a}}
    \put(0,72){\sf\footnotesize\textbf{b}}
    \put(48,72){\sf\footnotesize\textbf{c}}
    \put(0,23){\sf\footnotesize\textbf{d}}
    \put(25,23){\sf\footnotesize\textbf{e}}
    \put(50,23){\sf\footnotesize\textbf{f}}
    \put(75,23){\sf\footnotesize\textbf{g}}
    \end{overpic}
    \caption{
    \textbf{Results on continual vision-language benchmarks.}
    \subref{fig:vl-a}: Illustration of real-world vision-language scenarios, where an embodied agent continuously encounters evolving visual contexts and language queries.
    \subref{fig:vl-b}: Unified performance summary on CIFAR-100 and ImageNet-R using $A_{\rm auc}$, $A_{\rm last}$, forgetting, and backward transfer.
    \subref{fig:vl-c}: Performance comparison on CIFAR-100 and ImageNet-R. 
    \subref{fig:vl-d}: Principal component analysis (PCA) projection of representations on CIFAR-100.
    \subref{fig:vl-e}: Vision-language feature drift relative to frozen CLIP; bubble size denotes $A_{\rm last}$.
    \subref{fig:vl-f}: Hard-negative separation by positive-pair and hard-negative similarities.
    \subref{fig:vl-g}: Class-wise vision-language margin, computed as positive-pair similarity minus hard-negative similarity.
    All results are averaged over three independent runs; error bars indicate the standard error of the mean. Detailed numerical results are reported in Supplementary~\cref{tab:vl}.
    }
    \label{fig:vl}
    \phantomsubcaption\label{fig:vl-a}
    \phantomsubcaption\label{fig:vl-b}
    \phantomsubcaption\label{fig:vl-c}
    \phantomsubcaption\label{fig:vl-d}
    \phantomsubcaption\label{fig:vl-e}
    \phantomsubcaption\label{fig:vl-f}
    \phantomsubcaption\label{fig:vl-g}
\end{figure}

\myPara{Continual Vision-Language Learning.}
We next extend GCL to vision-language models, where CL must accommodate evolving visual concepts while preserving the cross-modal semantic structure acquired during large-scale pretraining (\cref{fig:vl-a}). Compared with visual recognition, CL introduces an additional challenge: changes in visual representations may disrupt their correspondence with language and weaken the shared semantic space that supports multimodal generalization. We evaluate FlyGCL with pretrained CLIP~\cite{radford2021learning} on CIFAR-100 and ImageNet-R, comparing with classical CL methods such as EWC~\cite{kirkpatrick2017overcoming} and LwF~\cite{li2017learning}, as well as CLIP-based CL methods including CLAP4CLIP~\cite{jha2024clap4clip} and MG-CLIP~\cite{huang2025mind}. FlyGCL achieves the strongest overall performance across both benchmarks (\cref{fig:vl-b,fig:vl-c}, Supplementary~\cref{tab:vl}), reaching $A_{\mathrm{auc}}$/$A_{\mathrm{last}}$ of 85.2/79.6\% on CIFAR-100 and 84.7/79.3\% on ImageNet-R, while also maintaining favorable forgetting and backward transfer. These results show that hierarchical modularity remains effective when GCL extends from visual prediction to continually evolving cross-modal representations.

We examine how different continual learners reshape the pretrained image-text representation space. Principal component analysis (PCA) visualization of the representation residuals reveals distinct adaptation trajectories relative to frozen CLIP (\cref{fig:vl-d}). FlyGCL exhibits a more constrained representation shift than competing approaches, while retaining the highest final performance. The joint decomposition of image- and text-feature shifts shows the same pattern (\cref{fig:vl-e}): FlyGCL remains closer to the pretrained cross-modal space without sacrificing adaptation accuracy. This balance is consistent with the hierarchical design of FlyGCL, where selective expert adaptation accommodates distribution-specific changes while temporal integration limits unnecessary displacement of the pretrained representation.

We further analyze whether this preservation extends to the semantic relationships that underpin image-text recognition. For each image, we compare its similarity to the matched text with that of the hardest negative text (\cref{fig:vl-f}). CL generally narrows this separation relative to frozen CLIP, whereas FlyGCL better preserves the positive-to-hard-negative margin. The class-wise analysis confirms that this advantage is broadly distributed across semantic categories rather than driven by a small subset of classes (\cref{fig:vl-g}). The representation- and similarity-level analyses indicate that FlyGCL adapts to evolving visual distributions while better retaining the pretrained image-text geometry, providing a mechanistic explanation for its stronger continual vision-language learning performance.

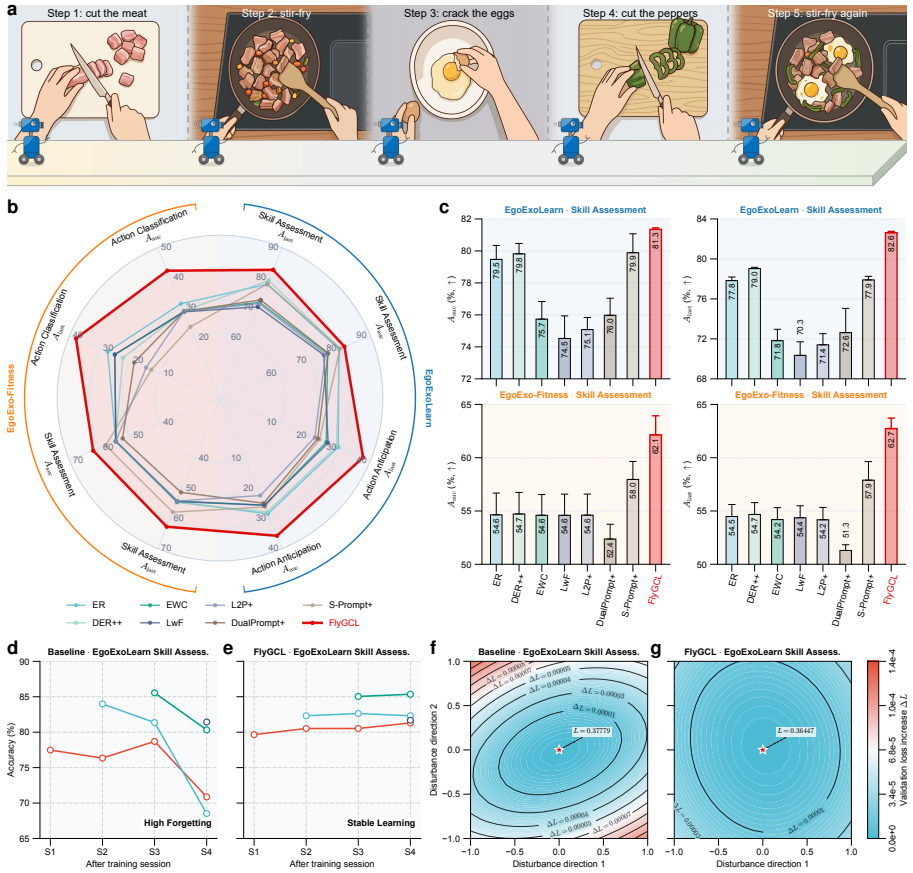
\begin{figure}
    \centering
    \vspace{-0.2cm}
    \begin{overpic}[width=\linewidth,clip,trim=0 450 0 0]{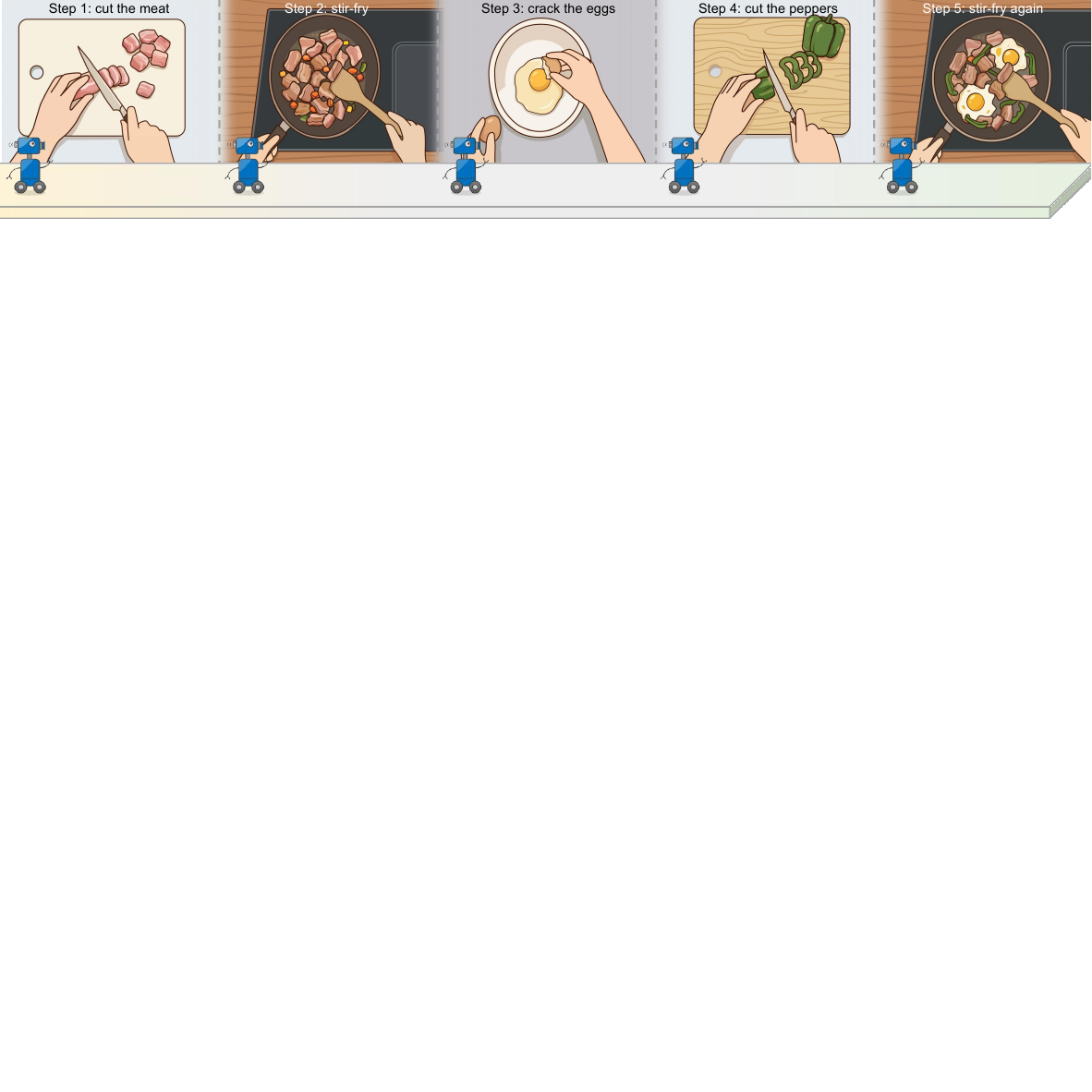}%
    \end{overpic}%
    % \begin{overpic}[width=\linewidth]{figs/egovideo_hw.pdf}%
    % \end{overpic}%
    \\[1.5mm]
    \begin{overpic}[width=\linewidth]{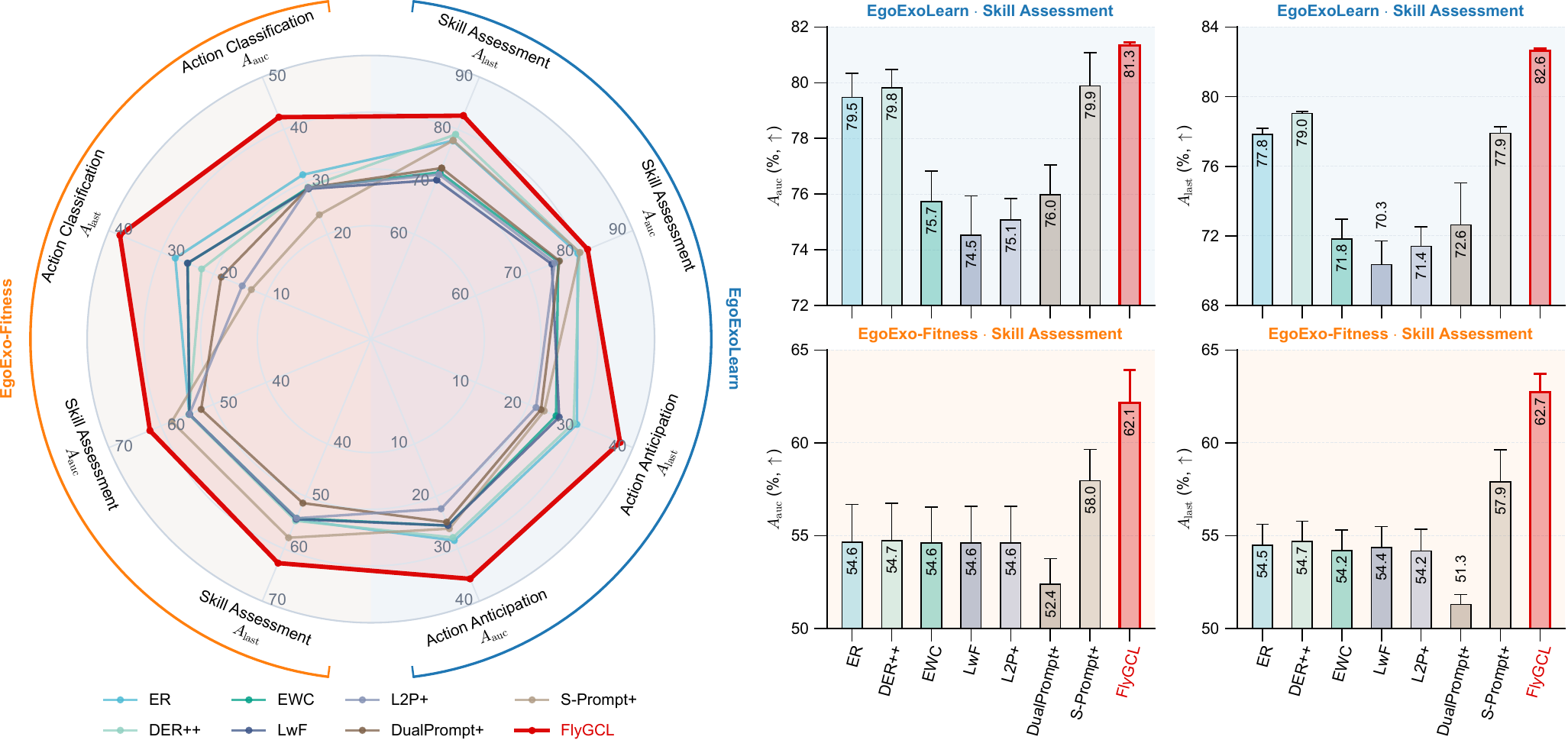}
    \end{overpic}
    \\[3.0mm]
    \begin{overpic}[width=\linewidth]{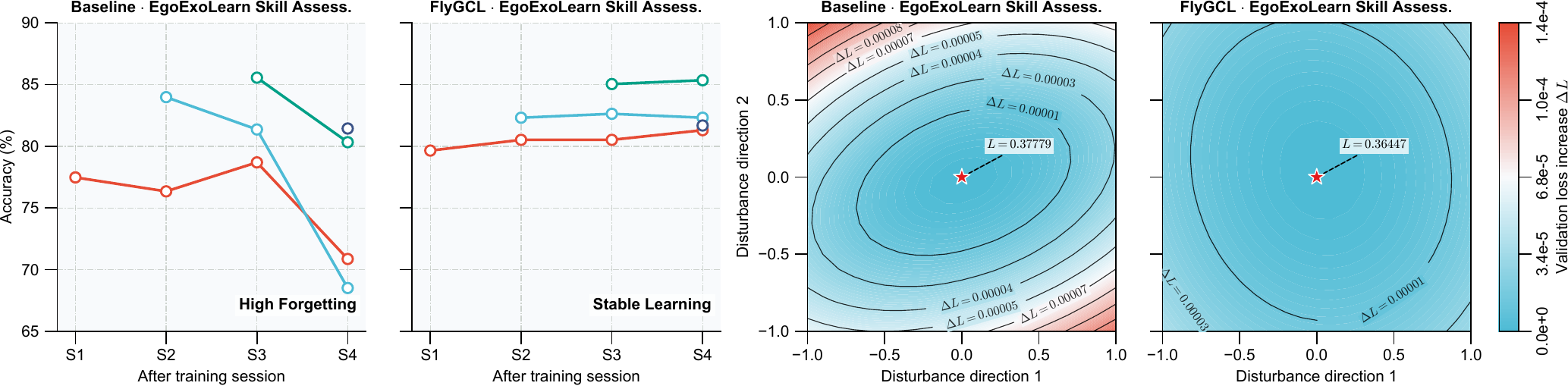}
    \put(0,95){\sf\footnotesize\textbf{a}}
    \put(0,73){\sf\footnotesize\textbf{b}}
    \put(48,73){\sf\footnotesize\textbf{c}}
    \put(0,24){\sf\footnotesize\textbf{d}}
    \put(24,24){\sf\footnotesize\textbf{e}}
    \put(47,24){\sf\footnotesize\textbf{f}}
    \put(71.5,24){\sf\footnotesize\textbf{g}}
    \end{overpic}
    \caption{
    \textbf{Results on continual ego-exo video understanding benchmarks.}
    \subref{fig:egovideo-a}: Illustration of continual ego-exo video understanding, where embodied agents encounter evolving activities and viewpoints over time.
    \subref{fig:egovideo-b}: Unified performance summary on EgoExoLearn and EgoExo-Fitness using $A_{\rm auc}$ and $A_{\rm last}$. For skill assessment on EgoExoLearn and EgoExo-Fitness, results are averaged over the Relation Network (RN)- and Triplet Loss (TL)-based ego-exo model variants \cite{huang2024egoexolearn}.
    \subref{fig:egovideo-c}: Performance comparison for continual skill assessment on EgoExoLearn and EgoExo-Fitness.
    \subref{fig:egovideo-d}, \subref{fig:egovideo-e}: Stage-wise accuracy of baseline and FlyGCL on EgoExoLearn.
    \subref{fig:egovideo-f}, \subref{fig:egovideo-g}: Local loss landscapes under parameter perturbations for baseline and FlyGCL.
    All results are averaged over three independent runs; error bars indicate the standard error of the mean. Detailed numerical results are reported in Supplementary~\cref{tab:ce4l_skill,tab:fitness_skill,tab:ce4l_anticipation,tab:fitness_action_classification}.
    }
    \label{fig:egovideo}
    \phantomsubcaption\label{fig:egovideo-a}
    \phantomsubcaption\label{fig:egovideo-b}
    \phantomsubcaption\label{fig:egovideo-c}
    \phantomsubcaption\label{fig:egovideo-d}
    \phantomsubcaption\label{fig:egovideo-e}
    \phantomsubcaption\label{fig:egovideo-f}
    \phantomsubcaption\label{fig:egovideo-g}
\end{figure}

\subsection{General Continual Learning for Embodied Perception and Action} % 700

Real-world intelligent systems must perceive changing environments while learning continuously from embodied and human-centered experience. Continual ego-exo video understanding and continual vision-language-action learning represent two important settings in this direction, spanning the interpretation of evolving human activities and the acquisition of action policies through multimodal interaction. These scenarios remain relatively underexplored in conventional CL, despite their direct relevance to long-running intelligent agents. Their video and interaction streams are inherently online, temporally continuous, and distributionally blurry, closely matching the uncertain and evolving experience targeted by GCL. We therefore examine whether FlyGCL can extend from perception to embodied understanding and action.

\myPara{Continual Ego-Exo Video Understanding.}
We extend GCL to continual ego-exo video understanding~\cite{yan2026ce4l}, where agents must continuously learn evolving activities from heterogeneous first- and third-person observations (\cref{fig:egovideo-a}). Compared with image-based GCL, continual video learning introduces additional temporal and viewpoint variations: the same activity can exhibit substantially different visual dynamics across egocentric and exocentric views, while subject, skill, and motion patterns evolve continuously throughout the data stream. We evaluate FlyGCL on EgoExoLearn~\cite{huang2024egoexolearn} and EgoExo-Fitness~\cite{li2024egoexo}, covering continual skill assessment, action anticipation, and action classification. Across these settings, FlyGCL achieves consistently strong overall performance on both $A_{\rm auc}$ and $A_{\rm last}$ (\cref{fig:egovideo-b,fig:egovideo-c}, Supplementary~\cref{tab:ce4l_skill,tab:fitness_skill,tab:ce4l_anticipation,tab:fitness_action_classification}), showing that the brain-inspired hierarchical modularity remains effective when GCL extends from static visual recognition to temporally evolving and cross-view video representations.

The performance gains are consistent across distinct forms of embodied video understanding. On EgoExoLearn skill assessment, FlyGCL reaches $A_{\rm auc}$/$A_{\rm last}$ of $81.3\%$/$82.6\%$, outperforming replay-, regularization-, and prompt-based methods. On EgoExo-Fitness skill assessment, FlyGCL reaches $62.1\%$/$62.7\%$, compared with $58.0\%$/$57.9\%$ for the strongest competing method in the main comparison (\cref{fig:egovideo-c}). Additional evaluations on sequence verification and guidance-based execution verification (Supplementary~\cref{tab:fitness_verification,tab:fitness_guidance_verification}) further demonstrate that the benefit of hierarchical modularity extends across distinct video tasks and output spaces.

We further examine how different continual learners evolve as new ego-exo sessions arrive. The stage-wise trajectories (\cref{fig:egovideo-d,fig:egovideo-e}) reveal substantially different forgetting patterns: the sequential fine-tuning baseline progressively degrades on previously learned sessions, exhibiting substantial forgetting on EgoExoLearn skill assessment, whereas FlyGCL largely preserves earlier performance and maintains stable learning throughout the continual stream. The local loss landscapes provide complementary evidence (\cref{fig:egovideo-f,fig:egovideo-g}): FlyGCL converges to a flatter basin and exhibits lower sensitivity to parameter perturbations than the sequential fine-tuning baseline. These results suggest that hierarchical modularity improves not only average continual performance but also the robustness of the learned solution, with specialized experts accommodating heterogeneous temporal and viewpoint-specific patterns while temporal integration limits destructive interference as related embodied experience recurs.

\begin{figure}[!h]
    \centering
    \begin{overpic}[width=\linewidth,clip,trim=0 450 0 0]{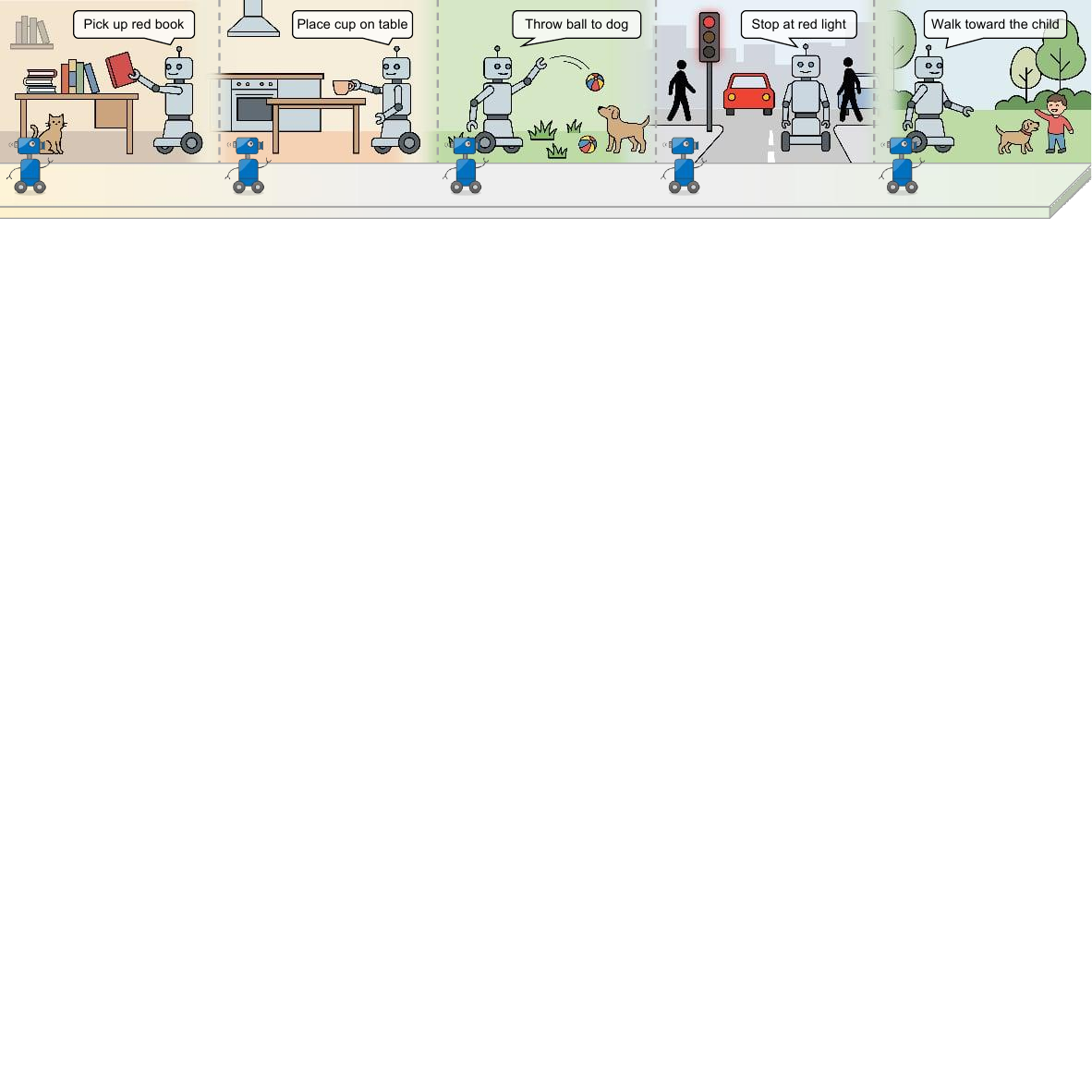}
    \end{overpic}
    \\[1.5mm]
    % \begin{overpic}[width=\linewidth]{figs/vla_online_combined.pdf}
    % \end{overpic}
    \begin{overpic}[width=\linewidth]{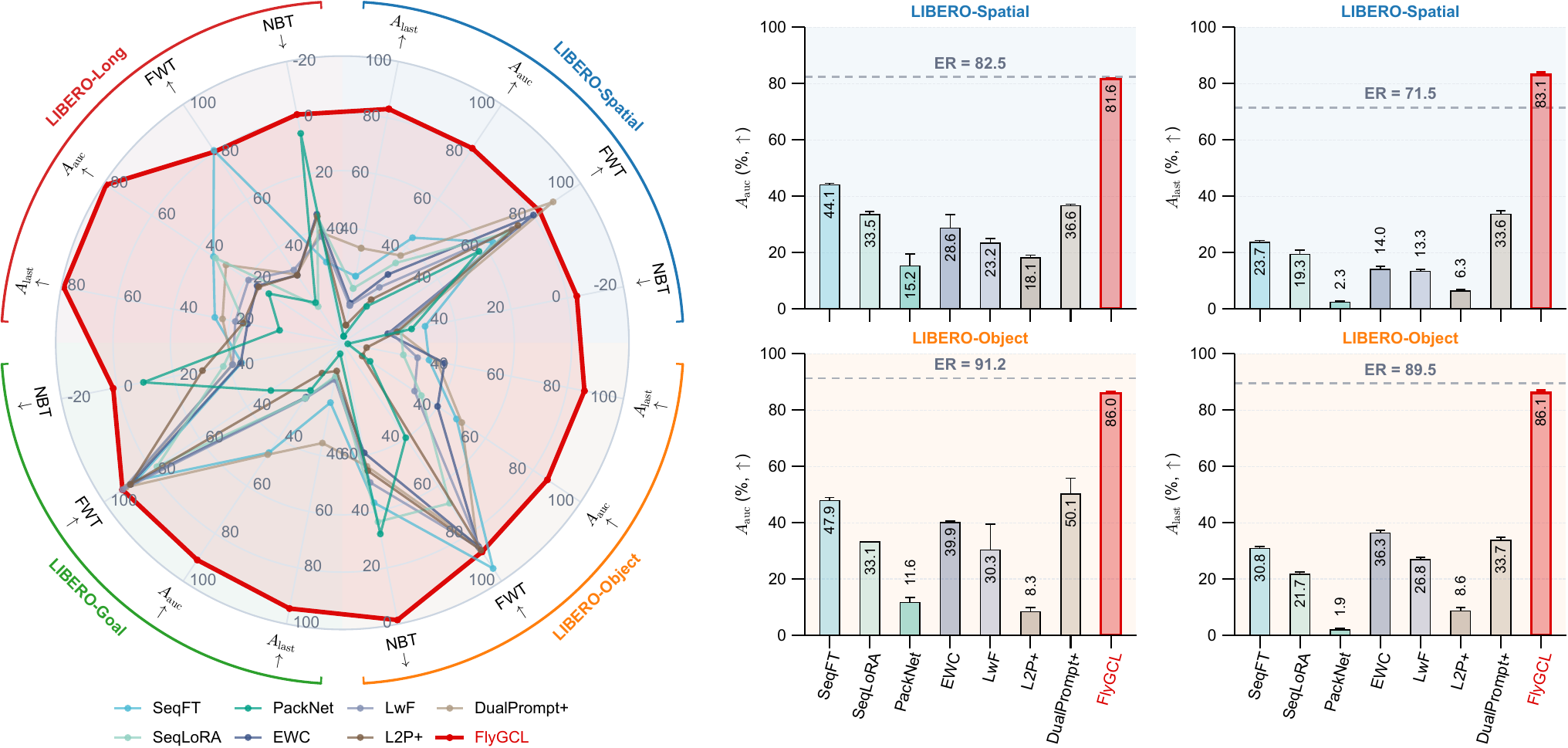}
    \end{overpic}
    % \\[1.5mm]
    % \begin{overpic}[width=\linewidth]{figs/offline_libero_spatial_object.pdf}
    % \end{overpic}
    \\[1.5mm]
    \begin{overpic}[width=\linewidth]{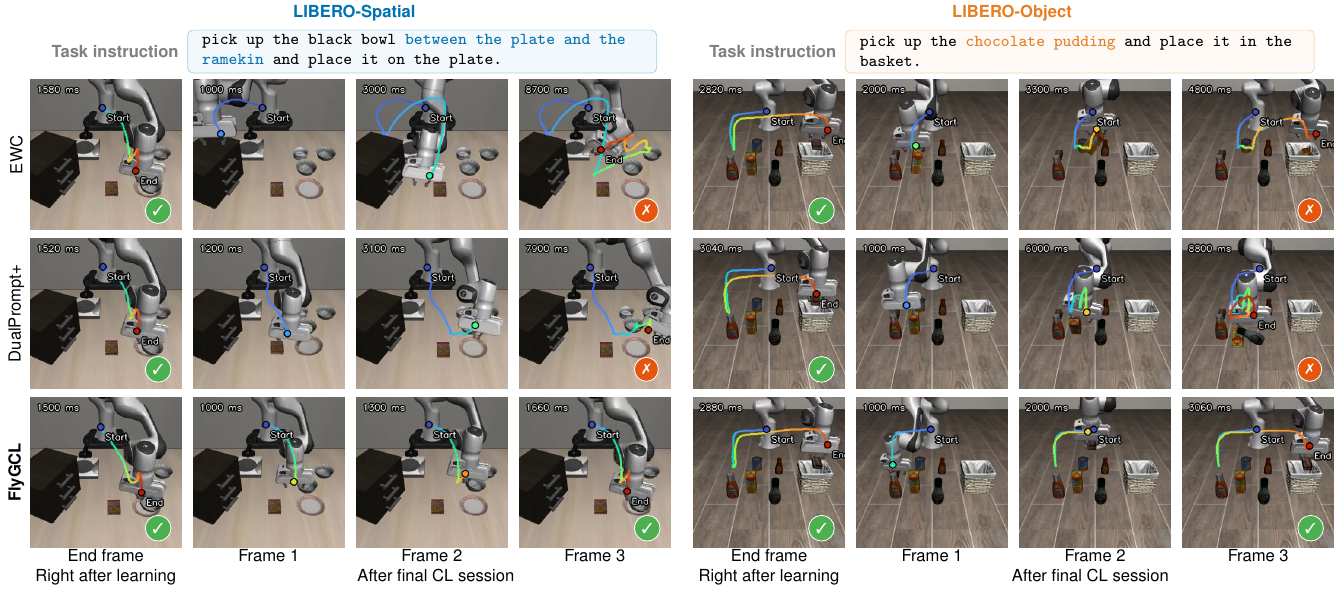}
        \put(0,113){\sf\footnotesize\textbf{a}}
        \put(0,91){\sf\footnotesize\textbf{b}}
        \put(45,91){\sf\footnotesize\textbf{c}}
        \put(0,42){\sf\footnotesize\textbf{d}}
        \put(51,42){\sf\footnotesize\textbf{e}}
    \end{overpic}
    \caption{
    \textbf{Results on continual vision-language-action benchmarks.}
    \subref{fig:vla-a}: Illustration of continual embodied interaction, where agents follow evolving language instructions across changing environments and tasks.
    \subref{fig:vla-b}: Unified performance summary across LIBERO benchmarks and CL metrics.
    \subref{fig:vla-c}: Performance comparison with state-of-the-art baselines on LIBERO-Spatial and LIBERO-Object using $A_{\rm auc}$ and $A_{\rm last}$.
    \subref{fig:vla-d}: Qualitative rollouts on LIBERO-Spatial, comparing task execution immediately after learning and after the final CL session for EWC, DualPrompt+, and FlyGCL.
    \subref{fig:vla-e}: Qualitative rollouts on LIBERO-Object under the same protocol.
    Results are averaged over three runs; error bars denote the standard error of the mean.
    }
    \label{fig:vla}
    \phantomsubcaption\label{fig:vla-a}
    \phantomsubcaption\label{fig:vla-b}
    \phantomsubcaption\label{fig:vla-c}
    \phantomsubcaption\label{fig:vla-d}
    \phantomsubcaption\label{fig:vla-e}
\end{figure}

\myPara{Continual Vision-Language-Action Learning.}
We further extend GCL to continual vision-language-action learning, where embodied agents must connect visual observations and language instructions to actions as environments, goals, and interaction states evolve over time (\cref{fig:vla-a}). Unlike recognition or representation learning, continual updates in this setting affect an entire interaction policy. The agent must preserve visuolinguistic grounding while adapting its visuomotor behaviour to new experience, even as related situations recur without clear task boundaries. Changes in perception or action generation can alter subsequent states and compound over an interaction trajectory, placing demands on both stable representation learning and policy adaptation. We evaluate FlyGCL on LIBERO-Spatial, -Object, -Goal, and -Long~\cite{liu2023libero} under the primary online GCL stream and a complementary offline protocol. In addition to EWC, LwF, L2P+, and DualPrompt+, we compare with sequential fine-tuning (SeqFT), LoRA fine-tuning (SeqLoRA), and PackNet~\cite{mallya2018packnet}.

In the online setting, task distributions overlap and recur throughout the stream, requiring the agent to acquire new visuomotor behaviours while retaining earlier ones. On LIBERO-Spatial and LIBERO-Object, FlyGCL reaches $A_{\rm last}$/$A_{\rm auc}$ of 83.1\%/81.6\% and 86.1\%/86.0\%, exceeding the strongest replay-free baselines by 49.5\%/37.5\% and 49.8\%/35.9\%, respectively (\cref{fig:vla-b,fig:vla-c}). The same pattern holds for goal-conditioned and long-horizon manipulation (Extended Data~\cref{fig:vla_online_ext-a,fig:vla_online_ext-b} and Supplementary~\cref{tab:libero_gcl}). FlyGCL also maintains strong forward and backward transfer across the four suites, indicating that it can incorporate new interaction patterns without sacrificing previously acquired behaviours. These results extend the benefit of hierarchical modularity from perceptual representations to instruction-conditioned action policies.

We also evaluate the four suites under the offline protocol, in which each task is trained for multiple passes before the learner proceeds to the next one. FlyGCL remains consistently strong across spatial, object-centric, goal-conditioned, and long-horizon manipulation (Extended Data~\cref{fig:vla_offline} and Supplementary~\cref{tab:libero_offline}). This result shows that its effectiveness is not limited to rapid online transitions. It also applies when task changes are more structured but each task still contains continuously varying visual states, action trajectories, and interaction outcomes.

The rollout analyses further show how this stability affects task execution (\cref{fig:vla-d,fig:vla-e}, Extended Data~\cref{fig:vla_online_ext-c,fig:vla_online_ext-d}). We compare FlyGCL with regularization-based EWC and task-expert-based DualPrompt+. After subsequent updates, both methods can lose a critical part of behaviours that they execute successfully immediately after learning. Their failures range from spatial grounding and object selection, as illustrated by DualPrompt+ missing the target position of the black bowl and EWC selecting the wrong object instead of the chocolate pudding, to incomplete goal execution and long-horizon action sequences. In the latter cases, EWC pursues an incorrect goal state, while DualPrompt+ completes the placement step but fails to close the drawer. FlyGCL more consistently preserves complete task execution across subsequent sessions, consistent with expert routing separating task-specific visuomotor changes and temporal integration retaining useful information across recurring experience.

\section{Discussion}\label{sec:discussion}

This work identifies hierarchical modularity as a unified principle for GCL under online, uncertain, and evolving data streams. Beyond preserving past knowledge, our results highlight a broader requirement: learning should be organized according to the relationships among incoming experiences. Inspired by the organization of olfactory learning and memory in \emph{Drosophila}, FlyGCL coordinates expert specialization and ensemble integration to separate conflicting experience while integrating compatible experience. This design consistently improves GCL across visual recognition, vision-language understanding, ego-exo video understanding, and embodied vision-language-action learning, while remaining effective across diverse pretrained representations and parameter-efficient adaptation mechanisms. Some simplified brain-inspired components were explored in our earlier conference paper~\cite{yan2026flyprompt}. The present work extends them into a biologically grounded hierarchical framework, supported by controlled olfactory modeling and evaluated across substantially broader models, modalities, and learning scenarios (Supplementary Information). These results establish hierarchical specialization and integration as an effective way to organize learning from dynamic experience.

From an AI perspective, GCL connects CL to a broader transition from intelligence acquired from static data to intelligence that develops through experience. Recent perspectives on the ``era of experience''~\cite{silver2025era}, autonomous machine intelligence~\cite{lecun2022path}, and open-ended learning~\cite{hughes2024open} similarly envision agents that continually extend their capabilities through interaction with the external world. Such agents must not only accumulate experience, but also determine what should be reused, separated, or integrated as distributions change. GCL provides a concrete learning paradigm for this process by bringing CL closer to the online, uncertain, and evolving conditions faced by real-world agents. This requirement arises in long-running systems such as embodied robots, autonomous vehicles, personalized assistants, and healthcare or scientific monitoring systems, where perception, knowledge, and behaviour must be continually updated without predefined task boundaries. FlyGCL provides a biologically grounded realization of this idea by organizing learning according to relationships among evolving experiences.

The biological implications are equally important. Recent whole-brain connectomes and cross-connectome cell typing have provided increasingly detailed maps of the \emph{Drosophila} nervous system~\cite{winding2023connectome,lin2024network,schlegel2024whole}, but how this anatomical organization supports learning from changing experience remains less understood. Our computational model offers a functional interpretation of the olfactory learning and memory system by linking sparse expansion and differentiated memory pathways to specialization and integration. Controlled olfactory experiments further show that their hierarchical coordination improves learning as experience shifts from disjoint to increasingly recurrent distributions. These results suggest testable roles for the underlying circuit organization: differentiated pathways may reduce interference between dissimilar experiences, whereas coordinated integration may preserve shared structure and improve generalization across related experiences. In this way, computational modelling can complement connectomics by linking anatomical organization to functional principles of learning and memory.

More broadly, our study exemplifies a bidirectional NeuroAI framework in which biological organization inspires machine-learning principles, while computational models generate testable hypotheses for neuroscience. The hierarchical organization of \emph{Drosophila} olfactory learning and memory motivates a unified view of MoE and EL as complementary computational paradigms for specialization and integration in GCL. In turn, our results suggest specific biological predictions: sparse expansion and differentiated downstream pathways should improve separation and reduce interference between dissimilar experiences; memory pathways operating across distinct spatial and temporal scales should contribute complementary information when related experiences recur; and their hierarchical coordination should be particularly beneficial when conflicting and compatible experiences coexist over time. Extending these principles to embodied GCL also aligns with the NeuroAI vision of an ``embodied Turing test'', in which intelligence develops through continuous sensorimotor interaction with the world~\cite{zador2023neuroai}.

Several directions remain beyond the scope of this study. Our biological model abstracts the overall organization of the \emph{Drosophila} olfactory learning and memory system, leaving richer neuromodulatory dynamics, biological processes underlying memory consolidation, and behavioural feedback for future computational and experimental investigation. On the AI side, our benchmarks extend GCL towards online, multimodal, and embodied experience, whereas longer-term open-ended interaction may additionally require active exploration, changing objectives, and dynamic allocation of learning resources. FlyGCL also concentrates plasticity in lightweight modules over relatively stable pretrained representations. Extending hierarchical specialization and integration to deeper learning within large foundation models remains an important direction. Looking forward, extending these principles across richer biological mechanisms, open-ended experience, and deeper model plasticity may help advance a more general science of continually developing intelligence.

\section{Methods}\label{sec:method}

\subsection{Problem Formulation}

\myPara{General Continual Learning.}
Conventional CL~\cite{wang2024comprehensive,wang2021afec,wang2025hide} often studies well-separated tasks under largely offline training, with previous-task data unavailable during subsequent updates. In contrast, GCL considers single-pass, non-stationary streams with uncertain and blurry data distributions, without clear task boundaries or reliable task identities~\cite{koh2021online,de2021continual,moon2023online}. Formally, the data stream consists of $T$ sessions, $\mathcal{S}=\{\mathcal{D}_1,\mathcal{D}_2,\ldots,\mathcal{D}_T\}$, where each session is given by $\mathcal{D}_t=\{(\mathbf{o}_{t,i},\mathbf{u}_{t,i})\}_{i=1}^{N_t}$. Here, $\mathbf{o}_{t,i}\in\mathcal{O}_t$ denotes an observation, and $\mathbf{u}_{t,i}\in\mathcal{U}_t$ denotes its associated learning signal. This notation provides a unified description of the GCL settings studied in this work: $\mathbf{o}$ may be instantiated as an image, an image-text pair, a video clip, an ego-exo multi-view sequence, or an embodied visual-language state, while $\mathbf{u}$ may correspond to a class label, semantic target, retrieval correspondence, temporal annotation, quality score, action, trajectory, or task-success signal. A standard model consists of a backbone $f_{\theta}(\cdot)$ and an output module $g_{\psi}(\cdot)$. For notational convenience, we denote the resulting predictor as $F_{\theta,\psi}=g_{\psi}\circ f_{\theta}$, i.e., $\hat{\mathbf{u}}=F_{\theta,\psi}(\mathbf{o})=g_{\psi}(f_{\theta}(\mathbf{o}))$. The learning objective is to obtain a unified mapping from $\mathcal{O}=\bigcup_{t=1}^{T}\mathcal{O}_t$ to $\mathcal{U}=\bigcup_{t=1}^{T}\mathcal{U}_t$, while learning online and retaining previously acquired knowledge.

At session $t$, samples are drawn from a local distribution $(\mathbf{o}_{t,i},\mathbf{u}_{t,i})\sim P_t(\mathbf{o},\mathbf{u})$, where $P_t$ can evolve gradually or abruptly, overlap with previous distributions, and recur later in the stream. In classification-based GCL, such overlap is commonly instantiated by the Si-Blurry setting~\cite{kang2025advancing,moon2023online}, which decomposes the global label space into a disjoint subset $\mathcal{Y}^{D}$ and a blurry subset $\mathcal{Y}^{B}$, with $\mathcal{Y}=\mathcal{Y}^{D}\cup\mathcal{Y}^{B}$ and $\mathcal{Y}^{D}\cap\mathcal{Y}^{B}=\varnothing$. Classes in $\mathcal{Y}^{D}$ are primarily associated with specific sessions, whereas classes in $\mathcal{Y}^{B}$ can reappear across multiple sessions. The disjoint ratio $r_D=\lvert\mathcal{Y}^{D}\rvert/\lvert\mathcal{Y}\rvert$ controls the degree of session-specific separation, while the blurry sample ratio $r_B$ controls the frequency of recurring samples. Beyond classification, the same principle also applies more broadly, where semantic concepts, temporal states, skill levels, or action patterns may be partially shared across sessions.

\myPara{Specialization and Integration.}
GCL requires determining when incoming knowledge should be shared and when it should be separated. Local distributions may share task-relevant structure in representations, semantics, temporal dynamics, viewpoints, or behaviours, while differing in gradients, decision boundaries, temporal alignments, or action policies. Related distributions should therefore share information to promote generalization, whereas conflicting distributions should be separated to reduce interference. A fully shared predictor $F_{\theta,\psi}$ can exploit common structure but is vulnerable to interference, whereas fully isolated predictors preserve distribution-specific knowledge at the cost of useful transfer.
We formalize this trade-off with modular machine learning. Let $\Omega=\{\omega_1,\omega_2,\ldots,\omega_K\}$ denote a set of trainable modules, such as prompts, adapters, LoRA branches, task heads, video-specific modules, or policy modules. When modules have their own output heads, we denote them by $\Psi=\{\psi_1,\psi_2,\ldots,\psi_K\}$. Given a representation $\mathbf{h}=f_{\theta}(\mathbf{o})$, a routing function produces module weights $\pi(\mathbf{o})=r_{\eta}(\mathbf{h})\in\Delta^{K-1}$. We use $F_{\theta,\Omega,\Psi}$ to denote the stream-level modular predictor induced by the backbone, modules, output heads, and routing or aggregation mechanism.

MoE supports specialization by assigning inputs to adaptive modules. For example, with $k^{*}=\arg\max_k \pi_k(\mathbf{o})$, the routed prediction can be written as
\begin{equation}
    F_{\mathrm{MoE}}(\mathbf{o})
    =
    F_{\theta,\omega_{k^{*}},\psi_{k^{*}}}(\mathbf{o}),
    \label{eq:form_moe}
\end{equation}
which reduces interference by limiting updates and predictions to the selected module. EL instead integrates multiple predictors,
\begin{equation}
    F_{\mathrm{EL}}(\mathbf{o})
    =
    \mathcal{A}
    \left(
    F_{\theta,\omega_1,\psi_1}(\mathbf{o}),
    \ldots,
    F_{\theta,\omega_K,\psi_K}(\mathbf{o})
    \right),
    \label{eq:form_el}
\end{equation}
where $\mathcal{A}(\cdot)$ denotes an aggregation function. This improves robustness and generalization for similar distributions by combining diverse but related predictions.

However, directly combining MoE and EL is non-trivial because they favor different distributional structures~\cite{wang2025convergent}. EL promotes generalization when predictors capture identical or closely related distributions~\cite{zhu2023towards}, whereas effective MoE routing relies on structured separation among heterogeneous distributions~\cite{chen2022towards}. In GCL, compatible and conflicting distributions may coexist: overly exclusive routing can suppress transfer among related distributions, whereas overly broad ensembling can mix incompatible modules and reintroduce interference. GCL therefore requires hierarchical coordination between specialization and integration, separating conflicting distributions while integrating compatible ones.

\subsection{Theoretical Analysis}

\myPara{Decomposing Hierarchical GCL Risk.}
The preceding formulation shows that effective GCL requires coordinating specialization and integration under uncertain and evolving data streams. We analyze this requirement using the stream-level modular predictor $F_{\theta,\Omega,\Psi}$ defined above, which comprises the backbone $f_{\theta}$, trainable modules $\Omega$, output modules $\Psi$, and the associated routing or aggregation operations. Let $F^{\star}$ denote the ideal stream-level predictor obtained if the latent structure of the local distributions $\{P_t\}_{t=1}^{T}$ were known. We define the expected GCL risk as
\begin{equation}
    \mathcal{R}_{\mathrm{GCL}}(F_{\theta,\Omega,\Psi})
    =
    \frac{1}{T}
    \sum_{t=1}^{T}
    \mathbb{E}_{(\mathbf{o},\mathbf{u})\sim P_t}
    \left[
    \ell\left(F_{\theta,\Omega,\Psi}(\mathbf{o}),\mathbf{u}\right)
    \right],
    \label{eq:gcl_risk}
\end{equation}
where $\ell(\cdot)$ denotes the task-specific loss. Because $F^{\star}$ is unavailable in CL, the practical objective is to approximate this ideal predictor from the observed stream $\mathcal{S}$ under non-stationary distribution shifts.

\begin{theorem}[Hierarchical Decomposition of GCL Risk]
\label{thm:gcl_decomposition}
Let $F_{\theta,\Omega,\Psi}$ be a hierarchical modular predictor learned from $\mathcal{S}=\{\mathcal{D}_1,\ldots,\mathcal{D}_T\}$, where $\mathcal{D}_t=\{(\mathbf{o}_{t,i},\mathbf{u}_{t,i})\}_{i=1}^{N_t}$. Under an additive excess-risk decomposition relative to $F^{\star}$, the expected GCL risk is upper bounded by
\begin{equation}
    \mathcal{R}_{\mathrm{GCL}}(F_{\theta,\Omega,\Psi})
    \lesssim
    \mathcal{R}_{\mathrm{fit}}(F_{\theta,\Omega,\Psi})
    +
    \underbrace{\mathcal{E}_{\mathrm{sep}}(F_{\theta,\Omega,\Psi})}_{\text{separation error}}
    +
    \underbrace{\mathcal{E}_{\mathrm{int}}(F_{\theta,\Omega,\Psi})}_{\text{integration error}}
    +
    \underbrace{\mathcal{C}_{\mathrm{coord}}(F_{\theta,\Omega,\Psi})}_{\text{coordination cost}} .
    \label{eq:gcl_decomposition}
\end{equation}
Here, the empirical fitting term is defined on the observed stream as
\begin{equation}
    \mathcal{R}_{\mathrm{fit}}(F_{\theta,\Omega,\Psi})
    =
    \frac{1}{T}
    \sum_{t=1}^{T}
    \frac{1}{N_t}
    \sum_{i=1}^{N_t}
    \ell\left(F_{\theta,\Omega,\Psi}(\mathbf{o}_{t,i}),\mathbf{u}_{t,i}\right).
    \label{eq:gcl_fit}
\end{equation}
The remaining terms are excess-risk components: $\mathcal{E}_{\mathrm{sep}}(F_{\theta,\Omega,\Psi})$ is induced by insufficient separation of conflicting local distributions, $\mathcal{E}_{\mathrm{int}}(F_{\theta,\Omega,\Psi})$ is induced by insufficient integration of compatible predictions, and $\mathcal{C}_{\mathrm{coord}}(F_{\theta,\Omega,\Psi})$ is the additional cost introduced by coordinating routing and integration.
\end{theorem}

The proof is provided in Supplementary~\cref{sec:app_proof_gcl_decomposition}.
The three terms characterize distinct failure modes of hierarchical modular learning. The separation error $\mathcal{E}_{\mathrm{sep}}$ captures residual interference when samples with conflicting gradients, decision boundaries, temporal alignments, or policies share parameters. The integration error $\mathcal{E}_{\mathrm{int}}$ captures the residual estimation gap when related local distributions or compatible predictors are treated in isolation. The coordination cost $\mathcal{C}_{\mathrm{coord}}$ arises from jointly performing routing and integration, and is defined as the positive excess loss of the hierarchical modular predictor relative to the ideal stream-level predictor:
\begin{equation}
    \mathcal{C}_{\mathrm{coord}}(F_{\theta,\Omega,\Psi})
    =
    \frac{1}{T}
    \sum_{t=1}^{T}
    \mathbb{E}_{(\mathbf{o},\mathbf{u})\sim P_t}
    \left[
    \ell\left(F_{\theta,\Omega,\Psi}(\mathbf{o}),\mathbf{u}\right)
    -
    \ell\left(F^{\star}(\mathbf{o}),\mathbf{u}\right)
    \right]_{+},
    \label{eq:gcl_coord}
\end{equation}
where $[\cdot]_{+}$ denotes the positive part. This term accounts for over-separation of related distributions, over-integration of incompatible modules, routing errors, and aggregation mismatch. Effective GCL therefore requires coordinating separation and integration rather than treating routing and aggregation as independent operations.

\myPara{Separation and Integration Gains.}
The decomposition in \cref{eq:gcl_decomposition} indicates that hierarchical modular learning is beneficial only when the gains from separation and integration outweigh the additional cost of coordinating them. We define the routing gain and ensemble gain as
\begin{equation}
    \Delta_{\mathrm{route}}
    =
    \mathcal{E}_{\mathrm{sep}}(F_{\theta,\psi})
    -
    \mathcal{E}_{\mathrm{sep}}(F_{\mathrm{MoE}}),
    \quad
    \Delta_{\mathrm{ens}}
    =
    \mathcal{E}_{\mathrm{int}}(F_{\mathrm{single}})
    -
    \mathcal{E}_{\mathrm{int}}(F_{\mathrm{EL}}),
    \label{eq:route_ens_gain}
\end{equation}
where $F_{\theta,\psi}$ is the fully shared predictor, $F_{\mathrm{MoE}}$ is the routed modular predictor in \cref{eq:form_moe}, $F_{\mathrm{single}}$ denotes a predictor using a single adaptive path, and $F_{\mathrm{EL}}$ is the integrated predictor in \cref{eq:form_el}. 
$\Delta_{\mathrm{route}}$ measures the reduction in separation error obtained by routing conflicting samples to different adaptive modules, while $\Delta_{\mathrm{ens}}$ measures the reduction in integration error obtained by aggregating compatible predictors.

\begin{proposition}[Separation and Integration Gains]
\label{prop:sep_int_gain}
Assume that routing provides a non-negative separation gain $\Delta_{\mathrm{route}}\geq 0$ and aggregation provides a non-negative integration gain $\Delta_{\mathrm{ens}}\geq 0$. Then the GCL risk of a hierarchical modular predictor $F_{\theta,\Omega,\Psi}$ satisfies
\begin{equation}
\begin{aligned}
    \mathcal{R}_{\mathrm{GCL}}(F_{\theta,\Omega,\Psi})
    \lesssim
    & \mathcal{R}_{\mathrm{fit}}(F_{\theta,\Omega,\Psi})
    +
    \mathcal{E}_{\mathrm{sep}}(F_{\theta,\psi})
    +
    \mathcal{E}_{\mathrm{int}}(F_{\mathrm{single}}) \\
    & -
    \Delta_{\mathrm{route}}
    -
    \Delta_{\mathrm{ens}}
    +
    \mathcal{C}_{\mathrm{coord}}(F_{\theta,\Omega,\Psi}).
    \label{eq:hierarchical_gain}
\end{aligned}
\end{equation}
Consequently, hierarchical modular learning improves the risk bound whenever
\begin{equation}
    \Delta_{\mathrm{route}}
    +
    \Delta_{\mathrm{ens}}
    >
    \mathcal{C}_{\mathrm{coord}}(F_{\theta,\Omega,\Psi}).
    \label{eq:hierarchical_condition}
\end{equation}
\end{proposition}

The proof is provided in Supplementary~\cref{sec:app_proof_sep_int_gain}.
\cref{prop:sep_int_gain} makes the trade-off in hierarchical modular learning explicit. Routing reduces interference by separating conflicting local distributions, whereas integration improves robustness by combining compatible predictions. The two operations are nevertheless coupled: overly exclusive routing may separate related distributions, while overly broad aggregation may mix incompatible modules. Effective GCL therefore requires coordinating separation and integration rather than optimizing either operation in isolation.

\myPara{Pretraining-Supported Coordination.}
The condition in \cref{eq:hierarchical_condition} shows that the benefit of hierarchical modular learning depends on both the gains from routing and aggregation and the cost of coordinating them. This coordination cost is particularly relevant in GCL, where the distinction between related and conflicting local distributions may be uncertain. We formalize this effect by separating the probability of an imperfect modular decision from its resulting prediction penalty.

\begin{proposition}[Pretraining-Supported Coordination]
\label{prop:pretrain_coord}
Let $\epsilon_{\mathrm{route}}$ denote the degree of routing error, $\epsilon_{\mathrm{agg}}$ denote the mismatch introduced by aggregating modular predictions, and $\Delta_{\mathrm{mis}}(f_{\theta})$ denote the prediction penalty of assigning an input to a suboptimal module under the representation induced by $f_{\theta}$. Then the coordination cost is bounded by
\begin{equation}
    \mathcal{C}_{\mathrm{coord}}(F_{\theta,\Omega,\Psi})
    \leq
    \epsilon_{\mathrm{route}}
    \Delta_{\mathrm{mis}}(f_{\theta})
    +
    \epsilon_{\mathrm{agg}}.
    \label{eq:coord_pretrain}
\end{equation}
Moreover, if a pretrained backbone provides a more stable and semantically organized representation than a representation learned from scratch on the online stream, then
\begin{equation}
    \Delta_{\mathrm{mis}}(f_{\theta}^{\mathrm{pre}})
    <
    \Delta_{\mathrm{mis}}(f_{\theta}^{\mathrm{scratch}}),
    \label{eq:pretrain_mismatch}
\end{equation}
which reduces the coordination cost in \cref{eq:coord_pretrain}.
\end{proposition}

The proof is provided in Supplementary~\cref{sec:app_proof_pretrain_coord}.
\cref{prop:pretrain_coord} explains how pretrained representations reduce the cost of imperfect coordination in hierarchical GCL. A routing error is more harmful when the selected module produces predictions that differ substantially from the appropriate one, whereas this penalty is reduced when the pretrained backbone maps related inputs into a stable semantic space. Pretrained representations therefore contribute not only positive transfer and resistance to forgetting, but also greater robustness to imperfect routing and integration. This supports hierarchical MoE-EL as a lightweight modular design over stable pretrained foundation models.

\myPara{Implications for FlyGCL.}
The analysis above identifies three requirements for GCL: separating conflicting local distributions, integrating compatible predictions, and limiting the coordination cost between them. These considerations motivate hierarchical modularity as the design principle of FlyGCL, particularly over stable pretrained representations. We next instantiate this principle in a brain-inspired GCL framework and describe its architecture and optimization.

\subsection{FlyGCL Model}

\myPara{Model Overview.}
FlyGCL is a unified brain-inspired framework for GCL with pretrained foundation models. Its design follows the organization of the \emph{Drosophila} olfactory learning and memory system at three functional levels. Sparse random expansion provides a distributed representation analogous to the PN-KC pathway and supports selective recruitment of downstream pathways. Spatially differentiated experts provide parallel memory pathways, while output heads with different update timescales capture complementary short- and long-term information. Therefore, stable representation, expert routing, and temporal integration form a computational hierarchy for separating conflicting experience and integrating compatible predictions.

Pretraining-based CL commonly introduces parameter-efficient tuning components, such as prompts, adapters, and LoRA, as lightweight experts over a pretrained backbone. Each expert provides a specialized learning pathway parameterized by $\bm{\omega}$ and accumulates spatially differentiated knowledge, producing outputs $f_{\bm{\theta}}(\bm{x};\bm{\omega})$. Under single-pass blurry data streams, expert-based learning must determine both the appropriate expert for each input and how to maintain reliable predictions under limited and imbalanced online supervision. FlyGCL addresses these challenges through random-expanded analytic routing and temporal ensemble integration within each routed expert. Let $f_{\bm{\theta}}(\cdot)$ denote a pretrained backbone and $\bm{h}=f_{\bm{\theta}}(\bm{x})\in\mathbb{R}^{d}$ its representation. The trainable expert pool is $\Omega=\{\bm{\omega}_1,\bm{\omega}_2,\ldots,\bm{\omega}_K\}$, and the prediction of expert $E_k$ is denoted by $F_{\bm{\theta},\bm{\omega}_k,\bm{\psi}_k}(\bm{x})$, where $\bm{\psi}_k$ is its output head.

\myPara{Random-Expanded Analytic Router.} To improve expert routing, FlyGCL uses a random-expanded analytic router inspired by sparse expansion in the fruit fly mushroom body. Given $\bm{h}=f_{\bm{\theta}}(\bm{x})$, we apply a fixed random projection followed by nonlinear activation:
\begin{equation}
    \bm{\varphi}(\bm{x})
    =
    \sigma\left(f_{\bm{\theta}}(\bm{x})\bm{R}\right)
    =
    \sigma(\bm{h}\bm{R})
    \in \mathbb{R}^{M},
\end{equation}
where $\bm{R}\in\mathbb{R}^{d\times M}$ is a random matrix, $M>d$, and $\sigma(\cdot)$ is an element-wise activation function. The expanded feature $\bm{\varphi}(\bm{x})$ is used for instance-level expert routing rather than final prediction, preserving the flexibility of downstream adaptive experts.

During online training, for each incoming batch $\mathcal{B}_i$ from session $t$, we compute the expanded feature matrix $\bm{\Phi}_i\in\mathbb{R}^{B\times M}$ and update two statistics:
\begin{equation}
    \bm{G}
    \leftarrow
    \bm{G}+\bm{\Phi}_i^{\top}\bm{\Phi}_i,
    \quad
    \bm{Q}
    \leftarrow
    \bm{Q}+\bm{\Phi}_i^{\top}\bm{C}_t ,
\end{equation}
where $\bm{G}\in\mathbb{R}^{M\times M}$ captures second-order feature correlations, $\bm{Q}\in\mathbb{R}^{M\times K}$ stores expert-wise feature statistics, and $\bm{C}_t\in\mathbb{R}^{B\times K}$ denotes the expert assignment target for the current session. The router matrix $\bm{U}\in\mathbb{R}^{K\times M}$ is obtained by the closed-form ridge solution
\begin{equation}
    \widehat{\bm{U}}^{\top}
    =
    (\bm{G}+\lambda\bm{I})^{-1}\bm{Q},
\end{equation}
where $\lambda>0$ is the regularization parameter. At inference time, the routing score and selected expert are computed as
\begin{equation}
    \bm{s}(\bm{x})
    =
    \bm{\varphi}(\bm{x})\widehat{\bm{U}}^{\top},
    \quad
    \hat{E}(\bm{x})
    =
    \arg\max_{k\leq K} s_k(\bm{x}).
\end{equation}
The routed prediction is
\begin{equation}
    F_{\rm route}(\bm{x})
    =
    F_{\bm{\theta},\bm{\omega}_{\hat{E}},\bm{\psi}_{\hat{E}}}(\bm{x}).
\end{equation}

After routing, the selected expert is updated using the task-specific learning signal. For a sample $(\bm{x}_i,\bm{y}_i)$ assigned to expert $\hat{E}$, the online objective is
\begin{equation}
    \mathcal{L}_i
    =
    \ell_i
    \left(
    F_{\bm{\theta},\bm{\omega}_{\hat{E}},\bm{\psi}_{\hat{E}}}(\bm{x}_i),
    \bm{y}_i
    \right),
\end{equation}
where $\ell_i$ can be instantiated as a classification, contrastive, regression, ranking, verification, imitation, or policy-learning loss. Because the router is updated through accumulated statistics and solved in closed form, it avoids iterative router training and is well suited to the single-pass dynamic data streams.

\myPara{Temporal Ensemble-based Experts.}
The prediction of a routed expert depends not only on its learned representation, but also on the stability of its output head as the data distribution evolves. FlyGCL therefore equips each expert with multiple output heads operating at different effective timescales. For expert $E_k$ that accumulates spatially differentiated knowledge, we maintain an online head $\bm{\psi}^{(0)}_k$ and $n$ shadow heads $\{\bm{\psi}^{(j)}_k\}_{j=1}^{n}$ updated with different exponential moving average (EMA) rates. For a linear output head $\bm{\psi}=(\bm{W},\bm{b})$, the $j$-th EMA head of expert $E_k$ is updated as
\begin{equation}
    \bm{W}_{k}^{(j)}
    \leftarrow
    \alpha_j\bm{W}_{k}^{(j)}
    +(1-\alpha_j)\bm{W},
    \quad
    \bm{b}_{k}^{(j)}
    \leftarrow
    \alpha_j\bm{b}_{k}^{(j)}
    +(1-\alpha_j)\bm{b}.
\end{equation}
Different EMA rates $\alpha_j$ induce distinct effective memory timescales: faster-updating heads prioritize recent observations, resembling short-term memory in the $\gamma$ lobe of \emph{Drosophila}, whereas progressively slower heads integrate information over longer timescales, yielding more stable decision boundaries under distribution shift and paralleling the more persistent memory supported by the $\alpha'/\beta'$ and $\alpha/\beta$ lobes.

At inference, after the analytic router selects expert $\hat{E}$, FlyGCL computes predictions from the online and EMA heads of this expert and aggregates them:
\begin{equation}
    F_{\rm FlyGCL}(\bm{x})
    =
    \mathcal{A}
    \left(
    F_{\bm{\theta},\bm{\omega}_{\hat{E}},\bm{\psi}_{\hat{E}}^{(0)}}(\bm{x}),
    F_{\bm{\theta},\bm{\omega}_{\hat{E}},\bm{\psi}_{\hat{E}}^{(1)}}(\bm{x}),
    \ldots,
    F_{\bm{\theta},\bm{\omega}_{\hat{E}},\bm{\psi}_{\hat{E}}^{(n)}}(\bm{x})
    \right),
\end{equation}
where $\mathcal{A}(\cdot)$ combines the head predictions, using normalized softmax weights when weighted temporal aggregation is adopted. This temporal ensemble improves decoding robustness without removing the specialization created by expert routing.

For visual recognition settings, we further introduce a lightweight gate to calibrate the integrated temporal output. We construct an analytic class head from the frozen pretrained representation reusing the same fixed random expansion and ridge solution described above, obtaining class evidence $\bm{z}_{\rm an}(\bm{x})$. We calibrate the temporally aggregated class probabilities $\bm{p}_{\rm temp}(\bm{x}) = F_{\rm FlyGCL}(\bm{x})$ according to
\begin{equation}
    z_{{\rm gate},c}(\bm{x})
    =
    \log p_{{\rm temp},c}(\bm{x})
    +
    \lambda_t z_{{\rm an},c}(\bm{x}),
    \qquad\forall c\in\mathcal{C}(\bm{x}),
\end{equation}
where $\lambda_t=\lambda_{\max}(t/T)^2$, $t/T$ denotes the normalized progress through the continual stream, $\lambda_{\max}\geq 0$ is the final calibration strength, and $\mathcal{C}(\bm{x})$ denotes the category set. The analytic evidence enhances classes supported by the stable pretrained representation and suppresses unsupported alternatives, without selecting experts or adding another temporal head. This class-dependent regulation may functionally resemble the modulatory role of dopamine neurons (DANs) over mushroom-body output pathways~\cite{aso2014mushroom,dasgupta2017neural}. The calibration is applied with temporal integration, while expert selection remains separately determined by the analytic router.

\subsection{Experimental Setups}

\myPara{Datasets.}
We evaluate FlyGCL across four CL scenarios: visual recognition, vision-language learning, ego-exo video understanding, and embodied vision-language-action learning. For visual recognition, we use CIFAR-100~\cite{krizhevsky2009learning}, containing 60,000 images from 100 classes (50,000/10,000 training/test images); ImageNet-R~\cite{hendrycks2021many}, containing 30,000 artistic and non-photographic renditions from 200 ImageNet classes; and CUB-200~\cite{wah2011caltech}, containing 11,788 images from 200 fine-grained bird species. For vision-language learning, we construct CLIP-based continual benchmarks on CIFAR-100 and ImageNet-R using the same visual streams, with each class represented by the textual prompt \texttt{a photo of a \{class name\}}. For ego-exo video understanding, we use EgoExoLearn~\cite{huang2024egoexolearn}, which contains 120 hours of egocentric execution and exocentric demonstration videos with gaze and multimodal annotations for cross-view association, planning, and skill assessment, and EgoExo-Fitness~\cite{li2024egoexo}, which provides synchronized ego-exo fitness videos with two-level temporal boundaries and interpretable action-judgement annotations. For embodied vision-language-action learning, we use LIBERO~\cite{liu2023libero}, a benchmark of language-conditioned robotic manipulation comprising LIBERO-Spatial, -Object, -Goal, and -Long. The first three suites each contain 10 tasks emphasizing spatial relations, object-centric manipulation, and goal-conditioned behaviour, respectively, while LIBERO-Long contains 10 long-horizon tasks derived from LIBERO-100. We adopt $r_D=50\%$ and $r_B=30\%$ across all scenarios (except $r_B=10\%$ for visual recognition following MVP~\cite{moon2023online} and MISA~\cite{kang2025advancing}). Sessions are constructed using class-, action-semantic-, or task-level partitions, with category-based partitioning used for the single-category EgoExo-Fitness benchmark. Detailed benchmark setup and implementations are provided in Supplementary~\cref{sec:app_setup}.

\myPara{Evaluation Metrics.}
We report two commonly used CL metrics across all benchmarks:
final average performance $A_{\mathrm{last}}$ and average anytime performance $A_{\mathrm{auc}}$. Let $R_{t,j}$ denote the task-specific performance on session $j$ after the model has learned through session $t$. Depending on the benchmark, $R_{t,j}$ is instantiated as accuracy, ranking accuracy, Top-5 recall, etc. The average anytime performance is defined as
\begin{equation}
    A_{\mathrm{auc}}
    =
    \frac{1}{T}\sum_{t=1}^{T}
    \left(
        \frac{1}{t}\sum_{j=1}^{t}R_{t,j}
    \right),
    \label{eq:metric_auc}
\end{equation}
which measures performance throughout CL. The final average
performance is defined as
\begin{equation}
    A_{\mathrm{last}}
    =
    \frac{1}{T}\sum_{j=1}^{T}R_{T,j},
    \label{eq:metric_last}
\end{equation}
which measures performance retained after learning the entire stream. Benchmark-specific metric instantiations and notation are detailed in Supplementary~\cref{sec:app_metrics}.

\myPara{Baseline Methods.}
We compare FlyGCL with representative CL methods across four application scenarios, with sequential fine-tuning (SeqFT) included as a common baseline throughout. For continual image recognition, we additionally consider the regularization-based methods EWC~\cite{kirkpatrick2017overcoming} and LwF~\cite{li2017learning}, the parameter efficiently tuning methods L2P~\cite{wang2022learning} and DualPrompt~\cite{wang2022dualprompt}, and the online CL methods MVP~\cite{moon2023online} and MISA~\cite{kang2025advancing}. For continual vision-language learning with CLIP-based models, we compare against EWC, LwF, L2P, DualPrompt, CLAP4CLIP~\cite{jha2024clap4clip}, and MG-CLIP~\cite{huang2025mind}. For continual ego-exo video understanding, we include EWC, LwF, L2P+, DualPrompt+, S-Prompt+~\cite{wang2022s}, and the replay-based methods Experience Replay (ER)~\cite{rolnick2019experience} and DER++~\cite{buzzega2020dark}. For these adapter-based variants, we replace the original prompt modules with adapters, with “+” indicating this modification. For continual embodied vision-language-action learning, we compare against EWC, LwF, L2P+, DualPrompt+, ER, and PackNet~\cite{mallya2018packnet}. Collectively, these baselines cover sequential fine-tuning, regularization-based, replay-based, and task-specific CL methods.

\myPara{Implementation Details.}
We implement FlyGCL on pretrained backbones adopted by each benchmark to ensure fair comparison with existing CL methods. For visual recognition, we use Vision Transformer (ViT-B/16) backbones with ImageNet-based pretraining, including supervised ImageNet-21K pretraining (Sup-21K), ImageNet-21K pretraining followed by ImageNet-1K fine-tuning (Sup-21K/1K), and self-supervised iBOT pretraining on ImageNet-21K (iBOT-21K). For vision-language learning, we use the pretrained OpenAI CLIP model with a ViT-B/16 image encoder. For ego-exo video understanding and embodied vision-language-action learning, we follow the backbone, input preprocessing, and evaluation pipeline of the corresponding benchmarks, replacing only the continual adaptation component across methods. All methods use the same data stream, session order, and online update budget. For prompt-, adapter-, and LoRA-based methods, we match the number of trainable modules to FlyGCL whenever applicable, so that comparisons primarily reflect the continual coordination strategy rather than model capacity. Hyperparameters are selected on the validation split of the first stream setting and then kept fixed across sessions. We report the mean and standard error over multiple random seeds. Detailed optimizer settings, learning rates, batch sizes, training budgets, and benchmark-specific implementations are provided in Supplementary~\cref{sec:app_impl}.

\backmatter

\section*{Data Availability}
All benchmark datasets used in this paper are publicly available from their original sources. For continual visual recognition, we use CIFAR-100 (\href{https://www.cs.toronto.edu/~kriz/cifar.html}{https://www.cs.toronto.edu/\textasciitilde kriz/cifar.html}), ImageNet-R (\href{https://github.com/hendrycks/imagenet-r}{https://github.com/hendrycks/imagenet-r}), and CUB-200 (\href{https://www.vision.caltech.edu/datasets/cub_200_2011/}{https://www.vision.caltech.edu/datasets/cub\_200\_2011/}). For continual ego-exo video understanding, we use EgoExoLearn (\href{https://github.com/OpenGVLab/EgoExoLearn}{https://github.com/OpenGVLab/EgoExoLearn}) and EgoExo-Fitness (\href{https://github.com/iSEE-Laboratory/EgoExo-Fitness}{https://github.com/iSEE-Laboratory/EgoExo-Fitness}). For continual vision-language-action learning, we use the LIBERO benchmark suite (\href{https://github.com/Lifelong-Robot-Learning/LIBERO}{https://github.com/Lifelong-Robot-Learning/LIBERO}), including LIBERO-Long, -Spatial, -Object and -Goal. The synthetic odor streams used for the biologically grounded analysis are generated following the procedure described in this paper and will be released with the code.

\section*{Code Availability}
The implementation code, configuration files, and evaluation protocols are available at \url{https://github.com/THU-NeuroML/FlyGCL}.

\section*{Acknowledgments}
This work was supported by the NSFC Project (No.~T2622023, No.~62406160 and No.~62595773), the Beijing Natural Science Foundation (No.~L247011), the Beijing Nova Program (No.~202604841279), and the Beijing Major Science and Technology Project (No.~Z251100008425003).

\section*{Author Contributions Statement}
H.Y., K.Z. and L.W. conceived the project.
H.Y., K.Z. and L.W. designed the computational framework.
H.Y. and K.Z. performed the main experiments, assisted by Q.C. and W.D.
H.Y., J.Z., G.S., Q.L., Y.Z., and L.W. contributed to the biological motivation and interpretation.
H.Y., K.Z. and L.W. analyzed the results.
H.Y., K.Z., and L.W. wrote the paper.
All authors discussed the results and revised the manuscript.
L.W. supervised the project.

\section*{Competing Interests Statement}
The authors declare no competing interests.

\clearpage


\bibliography{sn-bibliography}% common bib file


\clearpage

\renewcommand{\figurename}{Extended Data Fig.}
\setcounter{figure}{0}
\renewcommand{\thefigure}{\arabic{figure}}

\begin{figure}
    \centering
    \begin{overpic}[width=\linewidth,clip,trim=0 0 0 440]{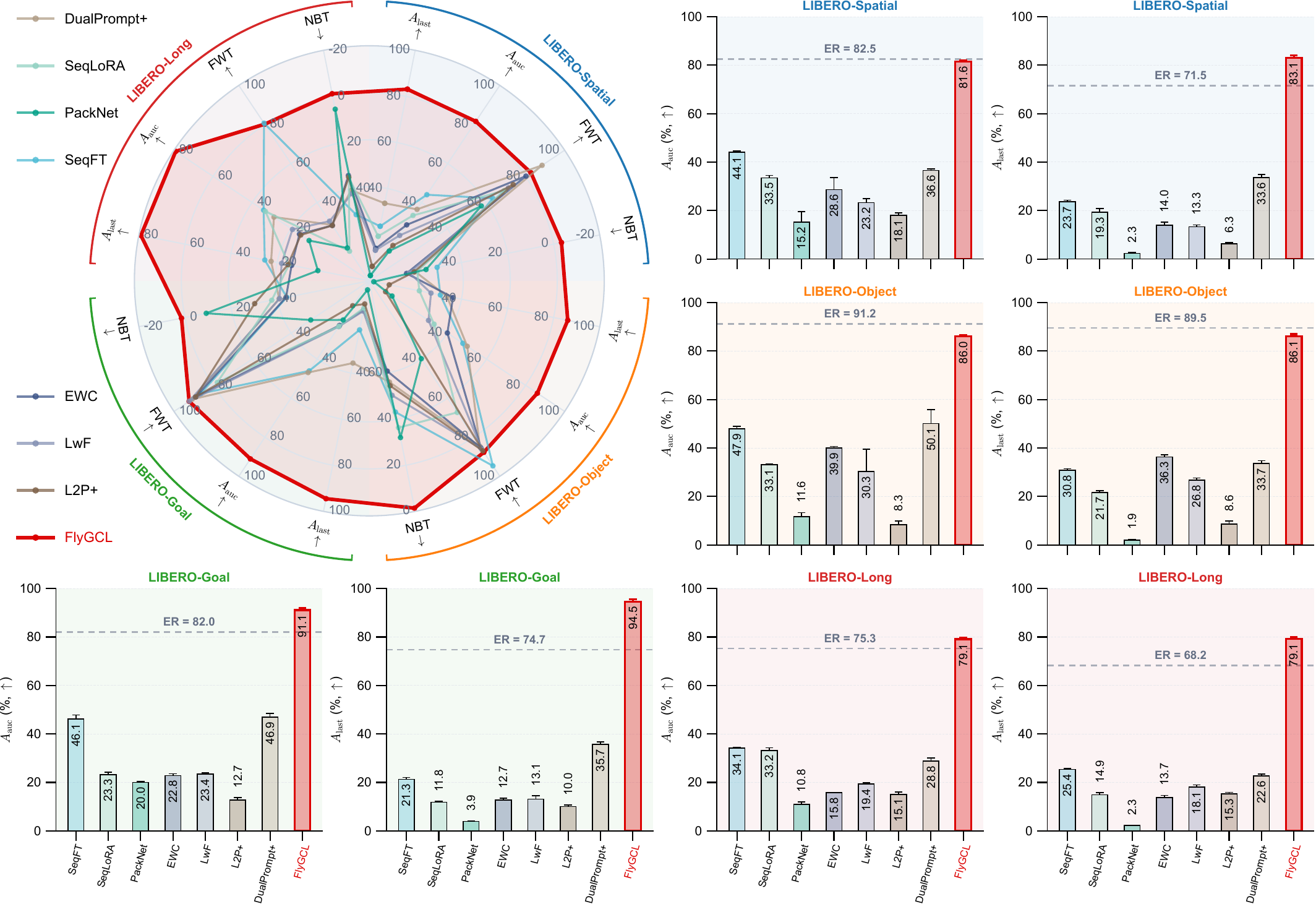}
        \put(0,24){\sf\footnotesize\textbf{a}}
        \put(50,24){\sf\footnotesize\textbf{b}}
    \end{overpic}
    \begin{overpic}[width=\linewidth]{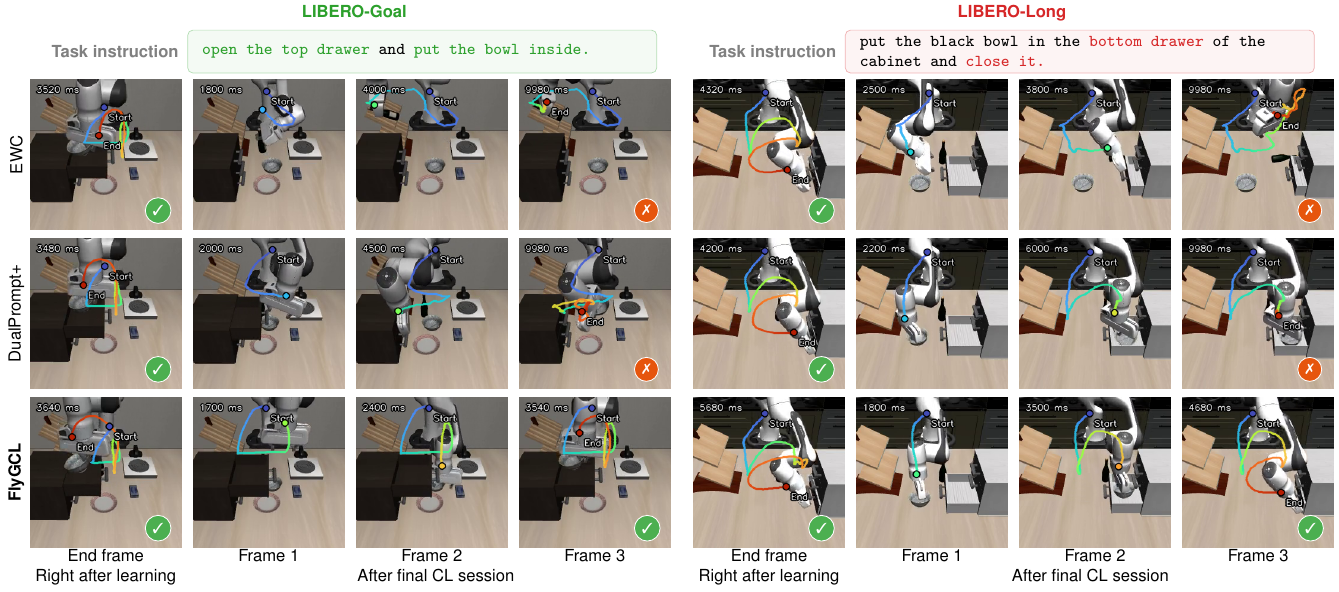}
        \put(0,43){\sf\footnotesize\textbf{c}}
        \put(50,43){\sf\footnotesize\textbf{d}}
    \end{overpic}
    \caption{\textbf{Extended results on continual vision-language-action benchmarks.}
    \subref{fig:vla_online_ext-a}: Performance comparison with state-of-the-art baselines on LIBERO-Goal using $A_{\rm auc}$ and $A_{\rm last}$.
    \subref{fig:vla_online_ext-b}: Performance comparison with state-of-the-art baselines on LIBERO-Long using $A_{\rm auc}$ and $A_{\rm last}$.
    \subref{fig:vla_online_ext-c}: Qualitative rollouts on LIBERO-Goal, comparing task execution immediately after learning and after the final CL session for EWC, DualPrompt+, and FlyGCL.
    \subref{fig:vla_online_ext-d}: Qualitative rollouts on LIBERO-Long under the same protocol.
    Results are averaged over three independent runs; error bars denote the standard error of the mean.
    }
    \label{fig:vla_online_ext}
    \phantomsubcaption\label{fig:vla_online_ext-a}
    \phantomsubcaption\label{fig:vla_online_ext-b}
    \phantomsubcaption\label{fig:vla_online_ext-c}
    \phantomsubcaption\label{fig:vla_online_ext-d}
\end{figure}

\clearpage

\begin{figure}
    \centering
    % \begin{overpic}[width=\linewidth]{figs/vla_offline_combined.pdf}
    % \end{overpic}
    \begin{overpic}[width=\linewidth]{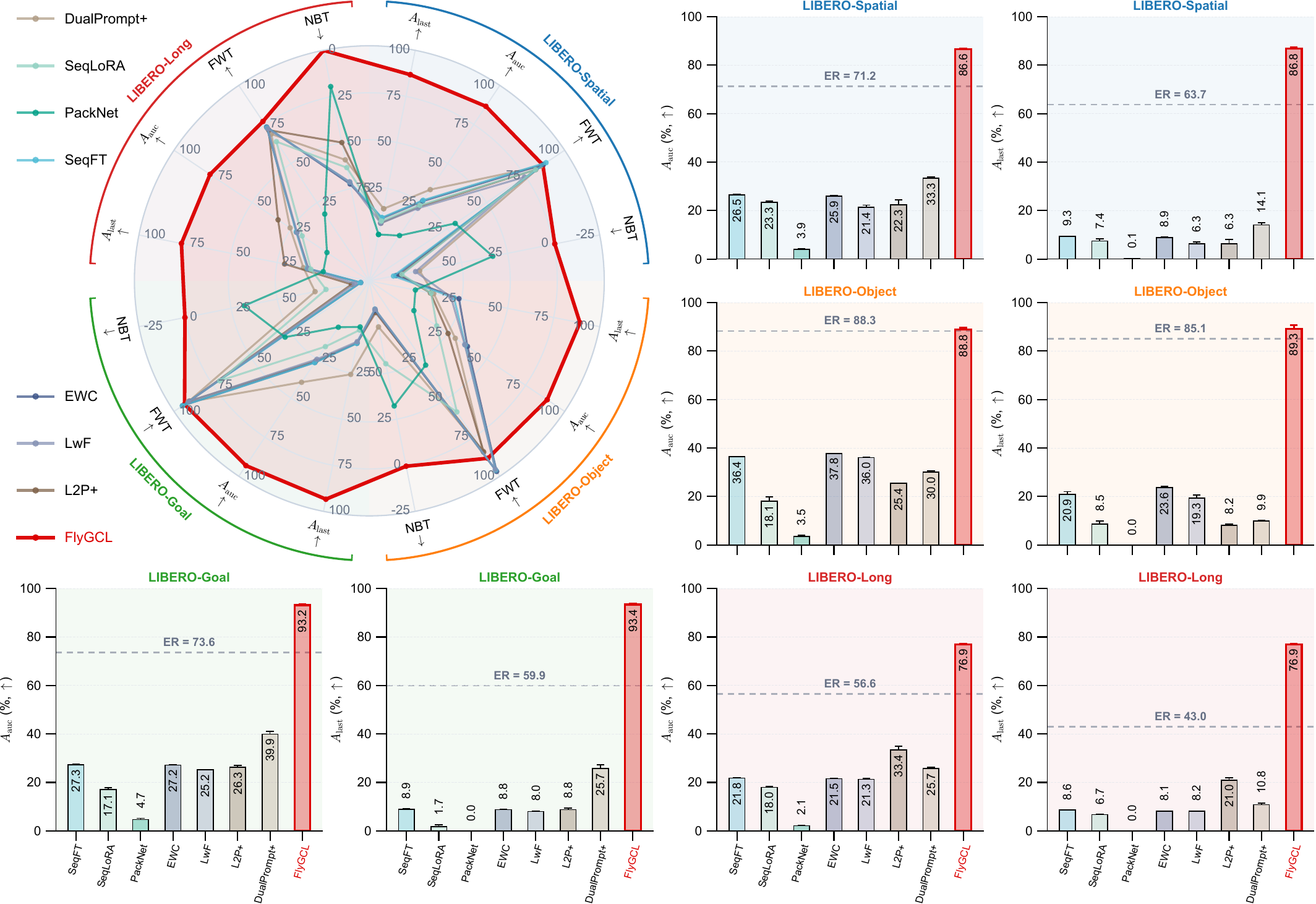}
        \put(0,68){\sf\footnotesize\textbf{a}}
        \put(48,68){\sf\footnotesize\textbf{b}}
        \put(0,25){\sf\footnotesize\textbf{c}}
    \end{overpic}
    \caption{
    \textbf{Offline results on continual vision-language-action benchmarks.}
    \subref{fig:vla_offline-a}: Unified performance summary across LIBERO-Long, LIBERO-Spatial, LIBERO-Goal, and LIBERO-Object under multiple continual learning metrics.
    \subref{fig:vla_offline-b}: Performance comparison with state-of-the-art baselines on LIBERO-Spatial and LIBERO-Object using $A_{\rm auc}$ and $A_{\rm last}$.
    \subref{fig:vla_offline-c}: Performance comparison with state-of-the-art baselines on LIBERO-Goal and LIBERO-Long using $A_{\rm auc}$ and $A_{\rm last}$.
    Results are averaged over three independent runs; error bars denote the standard error of the mean.
    }
    \label{fig:vla_offline}
    \phantomsubcaption\label{fig:vla_offline-a}
    \phantomsubcaption\label{fig:vla_offline-b}
    \phantomsubcaption\label{fig:vla_offline-c}
\end{figure}

\clearpage

\appendix
\renewcommand{\thefigure}{S\arabic{figure}}
\renewcommand{\thetable}{S\arabic{table}}
\renewcommand{\theequation}{S\arabic{equation}}

\section{Proofs}
\subsection{Proof of \cref{thm:gcl_decomposition}}
\label{sec:app_proof_gcl_decomposition}

\begin{proof}
Let $F_{\theta,\Omega,\Psi}$ be the stream-level modular predictor and let $F^{\star}$ be the ideal stream-level predictor. For brevity, we write $F$ for $F_{\theta,\Omega,\Psi}$ in this proof. By definition, the expected GCL risk over the evolving stream is
\begin{equation}
    \mathcal{R}_{\mathrm{GCL}}(F)
    =
    \frac{1}{T}
    \sum_{t=1}^{T}
    \mathbb{E}_{(\mathbf{o},\mathbf{u})\sim P_t}
    \left[
    \ell(F(\mathbf{o}),\mathbf{u})
    \right].
    \label{eq:app_gcl_risk}
\end{equation}
The empirical fitting risk on the observed stream $\mathcal{S}=\{\mathcal{D}_1,\ldots,\mathcal{D}_T\}$ is
\begin{equation}
    \mathcal{R}_{\mathrm{fit}}(F)
    =
    \frac{1}{T}
    \sum_{t=1}^{T}
    \frac{1}{N_t}
    \sum_{i=1}^{N_t}
    \ell(F(\mathbf{o}_{t,i}),\mathbf{u}_{t,i}).
    \label{eq:app_fit_risk}
\end{equation}
The gap between the expected stream risk and the empirical fitting risk can be written as
\begin{equation}
    \mathcal{R}_{\mathrm{GCL}}(F)-\mathcal{R}_{\mathrm{fit}}(F).
    \label{eq:app_residual_gap}
\end{equation}
Under hierarchical modular learning, this residual gap has three sources. First, if conflicting local distributions are assigned to insufficiently separated adaptive modules, their updates and predictions interfere with each other. We denote the corresponding excess risk by $\mathcal{E}_{\mathrm{sep}}(F)$. Second, if compatible local distributions or predictors are treated in isolation, the model loses potential transfer and suffers from larger estimation variance. We denote this excess risk by $\mathcal{E}_{\mathrm{int}}(F)$. Third, even if routing and integration are individually useful, their combination may be imperfect. The positive excess loss of the hierarchical modular predictor over the ideal stream-level predictor is
\begin{equation}
    \mathcal{C}_{\mathrm{coord}}(F)
    =
    \frac{1}{T}
    \sum_{t=1}^{T}
    \mathbb{E}_{(\mathbf{o},\mathbf{u})\sim P_t}
    \left[
    \ell(F(\mathbf{o}),\mathbf{u})
    -
    \ell(F^{\star}(\mathbf{o}),\mathbf{u})
    \right]_{+}.
    \label{eq:app_coord_cost}
\end{equation}
By the assumed additive excess-risk decomposition, the residual gap in \cref{eq:app_residual_gap} is upper bounded by the sum of these three non-negative components:
\begin{equation}
    \mathcal{R}_{\mathrm{GCL}}(F)-\mathcal{R}_{\mathrm{fit}}(F)
    \lesssim
    \mathcal{E}_{\mathrm{sep}}(F)
    +
    \mathcal{E}_{\mathrm{int}}(F)
    +
    \mathcal{C}_{\mathrm{coord}}(F).
    \label{eq:app_residual_bound}
\end{equation}
Adding $\mathcal{R}_{\mathrm{fit}}(F)$ to both sides gives
\begin{equation}
    \mathcal{R}_{\mathrm{GCL}}(F)
    \lesssim
    \mathcal{R}_{\mathrm{fit}}(F)
    +
    \mathcal{E}_{\mathrm{sep}}(F)
    +
    \mathcal{E}_{\mathrm{int}}(F)
    +
    \mathcal{C}_{\mathrm{coord}}(F).
    \label{eq:app_gcl_decomposition}
\end{equation}
Substituting back $F=F_{\theta,\Omega,\Psi}$ yields \cref{eq:gcl_decomposition}. This completes the proof.
\end{proof}

\subsection{Proof of \cref{prop:sep_int_gain}}
\label{sec:app_proof_sep_int_gain}

\begin{proof}
For brevity, we write $F_{\mathrm{hier}}$ for $F_{\theta,\Omega,\Psi}$. From \cref{thm:gcl_decomposition}, the expected GCL risk of a hierarchical modular predictor satisfies
\begin{equation}
    \mathcal{R}_{\mathrm{GCL}}(F_{\mathrm{hier}})
    \lesssim
    \mathcal{R}_{\mathrm{fit}}(F_{\mathrm{hier}})
    +
    \mathcal{E}_{\mathrm{sep}}(F_{\mathrm{hier}})
    +
    \mathcal{E}_{\mathrm{int}}(F_{\mathrm{hier}})
    +
    \mathcal{C}_{\mathrm{coord}}(F_{\mathrm{hier}}).
    \label{eq:app_hier_start}
\end{equation}
By the definition of the routing gain in \cref{eq:route_ens_gain}, we have
\begin{equation}
    \Delta_{\mathrm{route}}
    =
    \mathcal{E}_{\mathrm{sep}}(F_{\theta,\psi})
    -
    \mathcal{E}_{\mathrm{sep}}(F_{\mathrm{MoE}}),
\end{equation}
which gives
\begin{equation}
    \mathcal{E}_{\mathrm{sep}}(F_{\mathrm{MoE}})
    =
    \mathcal{E}_{\mathrm{sep}}(F_{\theta,\psi})
    -
    \Delta_{\mathrm{route}}.
    \label{eq:app_route_gain}
\end{equation}
Since the hierarchical predictor uses routing as its separation mechanism, its separation error is upper bounded by that of the routed modular predictor, i.e.,
\begin{equation}
    \mathcal{E}_{\mathrm{sep}}(F_{\mathrm{hier}})
    \leq
    \mathcal{E}_{\mathrm{sep}}(F_{\mathrm{MoE}}).
    \label{eq:app_sep_bound}
\end{equation}
Combining \cref{eq:app_route_gain,eq:app_sep_bound}, we obtain
\begin{equation}
    \mathcal{E}_{\mathrm{sep}}(F_{\mathrm{hier}})
    \leq
    \mathcal{E}_{\mathrm{sep}}(F_{\theta,\psi})
    -
    \Delta_{\mathrm{route}}.
    \label{eq:app_sep_gain}
\end{equation}

Similarly, by the definition of the ensemble gain in \cref{eq:route_ens_gain}, we have
\begin{equation}
    \Delta_{\mathrm{ens}}
    =
    \mathcal{E}_{\mathrm{int}}(F_{\mathrm{single}})
    -
    \mathcal{E}_{\mathrm{int}}(F_{\mathrm{EL}}),
\end{equation}
which gives
\begin{equation}
    \mathcal{E}_{\mathrm{int}}(F_{\mathrm{EL}})
    =
    \mathcal{E}_{\mathrm{int}}(F_{\mathrm{single}})
    -
    \Delta_{\mathrm{ens}}.
    \label{eq:app_ens_gain}
\end{equation}
Since the hierarchical predictor uses aggregation as its integration mechanism, its integration error is upper bounded by that of the ensemble predictor, i.e.,
\begin{equation}
    \mathcal{E}_{\mathrm{int}}(F_{\mathrm{hier}})
    \leq
    \mathcal{E}_{\mathrm{int}}(F_{\mathrm{EL}}).
    \label{eq:app_int_bound}
\end{equation}
Combining \cref{eq:app_ens_gain,eq:app_int_bound}, we obtain
\begin{equation}
    \mathcal{E}_{\mathrm{int}}(F_{\mathrm{hier}})
    \leq
    \mathcal{E}_{\mathrm{int}}(F_{\mathrm{single}})
    -
    \Delta_{\mathrm{ens}}.
    \label{eq:app_int_gain}
\end{equation}

Substituting \cref{eq:app_sep_gain,eq:app_int_gain} into \cref{eq:app_hier_start} yields
\begin{equation}
\begin{aligned}
    \mathcal{R}_{\mathrm{GCL}}(F_{\mathrm{hier}})
    \lesssim
    & \mathcal{R}_{\mathrm{fit}}(F_{\mathrm{hier}})
    +
    \mathcal{E}_{\mathrm{sep}}(F_{\theta,\psi})
    +
    \mathcal{E}_{\mathrm{int}}(F_{\mathrm{single}}) \\
    & -
    \Delta_{\mathrm{route}}
    -
    \Delta_{\mathrm{ens}}
    +
    \mathcal{C}_{\mathrm{coord}}(F_{\mathrm{hier}}).
\end{aligned}
\end{equation}
Substituting back $F_{\mathrm{hier}}=F_{\theta,\Omega,\Psi}$ gives \cref{eq:hierarchical_gain}.

Finally, comparing this bound with the corresponding bound without routing and integration gains,
\begin{equation}
    \mathcal{R}_{\mathrm{fit}}(F_{\theta,\Omega,\Psi})
    +
    \mathcal{E}_{\mathrm{sep}}(F_{\theta,\psi})
    +
    \mathcal{E}_{\mathrm{int}}(F_{\mathrm{single}}),
\end{equation}
shows that hierarchical modular learning improves the bound whenever
\begin{equation}
    \Delta_{\mathrm{route}}
    +
    \Delta_{\mathrm{ens}}
    >
    \mathcal{C}_{\mathrm{coord}}(F_{\theta,\Omega,\Psi}).
\end{equation}
This proves \cref{eq:hierarchical_condition} and completes the proof.
\end{proof}

\subsection{Proof of \cref{prop:pretrain_coord}}
\label{sec:app_proof_pretrain_coord}

\begin{proof}
For brevity, we write $F_{\mathrm{hier}}$ for $F_{\theta,\Omega,\Psi}$. The coordination cost measures the excess loss introduced when routing and integration do not exactly match the ideal stream-level organization. We decompose this cost into two sources: mismatch due to imperfect routing and mismatch due to imperfect aggregation.

Let $\epsilon_{\mathrm{route}}$ denote the degree of routing error, i.e., the probability or expected weight that an input is assigned to a suboptimal module. Let $\Delta_{\mathrm{mis}}(f_{\theta})$ denote the maximum expected loss increase caused by such a suboptimal assignment under the representation induced by $f_{\theta}$. Then the routing-induced part of the coordination cost is bounded by
\begin{equation}
    \mathcal{C}_{\mathrm{route}}(F_{\mathrm{hier}})
    \leq
    \epsilon_{\mathrm{route}}
    \Delta_{\mathrm{mis}}(f_{\theta}).
    \label{eq:app_route_coord}
\end{equation}
Let $\epsilon_{\mathrm{agg}}$ denote the residual mismatch caused by aggregating modular predictions that are not perfectly compatible. Then the total coordination cost is bounded by the sum of the routing-induced mismatch and the aggregation-induced mismatch:
\begin{equation}
    \mathcal{C}_{\mathrm{coord}}(F_{\mathrm{hier}})
    \leq
    \mathcal{C}_{\mathrm{route}}(F_{\mathrm{hier}})
    +
    \epsilon_{\mathrm{agg}}.
    \label{eq:app_coord_split}
\end{equation}
Combining \cref{eq:app_route_coord,eq:app_coord_split}, we obtain
\begin{equation}
    \mathcal{C}_{\mathrm{coord}}(F_{\mathrm{hier}})
    \leq
    \epsilon_{\mathrm{route}}
    \Delta_{\mathrm{mis}}(f_{\theta})
    +
    \epsilon_{\mathrm{agg}}.
    \label{eq:app_coord_bound}
\end{equation}
Substituting back $F_{\mathrm{hier}}=F_{\theta,\Omega,\Psi}$ gives \cref{eq:coord_pretrain}.

It remains to show why pretrained representations reduce the mismatch penalty. By definition, $\Delta_{\mathrm{mis}}(f_{\theta})$ measures the expected prediction penalty when an input is routed to a suboptimal module. If a representation maps related inputs closer together and makes conflicting inputs more separable, then the predictions produced by nearby or partially mismatched modules are less divergent for related inputs, while routing ambiguity is reduced for conflicting inputs. A pretrained backbone provides such a more stable and semantically organized representation than a representation learned from scratch on the limited online stream. Therefore, the mismatch penalty under a pretrained backbone is smaller:
\begin{equation}
    \Delta_{\mathrm{mis}}(f_{\theta}^{\mathrm{pre}})
    <
    \Delta_{\mathrm{mis}}(f_{\theta}^{\mathrm{scratch}}).
    \label{eq:app_pretrain_mismatch}
\end{equation}
Substituting \cref{eq:app_pretrain_mismatch} into \cref{eq:app_coord_bound} shows that pretraining reduces the upper bound of $\mathcal{C}_{\mathrm{coord}}$, which proves \cref{eq:pretrain_mismatch} and completes the proof.
\end{proof}

\section{Additional Experimental Setups}
\label{sec:app_setup}

\subsection{Biological Simulation}
\label{sec:app_bio_simulation}

\myPara{Odor Classes and Spatial Regions.}
We generate 100 class prototypes independently and uniformly in $[0,1)^{50}$. Training and test inputs are sampled from the same space and assigned to their nearest prototype by squared Euclidean distance. The prototypes are then grouped by proximity into five regions of 20 classes using capacity-constrained clustering. For each seed, the formal benchmark contains 10,000 training and 2,000 test inputs from each region. The fixed test set therefore contains 10,000 inputs, while every training stream contains 50,000 inputs.

\myPara{Continual Odor Streams.}
The stream represents a learner moving through the five spatial regions and acquiring experience along this trajectory. Each region defines the home stage of its classes. In Disjoint, the 10,000 inputs from each region are shuffled within their home stage and the five stages are presented sequentially. Blurry and Joint soften the boundaries between neighboring periods of experience without altering the inputs or labels. For each home stage, a specified subset of classes remains disjoint, while 10\% of the inputs from the other classes is removed, pooled across stages, globally shuffled, and reassigned under equal stage quotas. Blurry retains 10 disjoint classes (50\%) in each region, whereas Joint applies this procedure to all classes. In every case, the model encounters data from all five regions in sequence, each input appears exactly once, and the total learning budget is unchanged.

\myPara{Biologically Grounded Architecture.}
The sensory network follows the population-level compression and expansion described in biological and task-driven accounts of the \emph{Drosophila} olfactory system~\cite{caron2013random,wang2021evolving}. Each of the 50 input channels is replicated across 26 ORNs and pooled by one PN, giving 1,300 ORNs and 50 PNs. During training, independent Gaussian noise with standard deviation $0.1$ is added to the ORN responses; evaluation is noise-free. Within each group, ORN-PN weights are sampled from a truncated log-normal distribution fitted to FlyWire connections between matched olfactory types and normalized to sum to one. PN activity is rectified and projected to 2,000 KCs through a sparse PN-KC connection matrix. We randomly set 91.5\% of its entries to zero, sample the remaining 8.5\% from a truncated Gaussian distribution fitted to the nonzero FlyWire PN-KC connection weights, and then keep the resulting matrix fixed. After rectification, only the 100 largest KC responses (top-5\% of the 2,000 KCs) are retained for each odor. The sensory pathway remains fixed throughout continual learning, so that all learning occurs in downstream components and the effects of learning and memory organization can be evaluated independently of changes in sensory encoding.

\myPara{Baseline Variants.}
MoE uses five independently initialized experts with distinct views of KC activity, representing spatially differentiated memory pathways. During training, we associate each stage with one expert as a simple way to activate experts successively along the stream. This association is specific to the experimental protocol rather than a requirement for known task identities or predefined stage semantics: experts could instead be activated sequentially after a fixed number of observed samples or by another online switching rule. In parallel, the model maintains the mean KC representation of each observed stage. At evaluation, an input is routed to the active expert whose mean representation has the highest cosine similarity to its KC activity. The router is updated only from inputs observed so far and does not use class or region labels at inference.

The Baseline uses one bias-free 100-class readout trained at a learning rate of $10^{-3}$. EL replaces it with three independently initialized heads trained at learning rates $10^{-2}$, $10^{-3}$, and $10^{-4}$, representing fast, intermediate, and slow effective memory scales. MoE uses five experts with one head each, and MoE+EL uses three such heads within every expert. Each head is optimized with its own cross-entropy loss, and their class probabilities are averaged within the selected expert at inference. The incremental EL ablation further compares these models with three independently initialized heads all trained at $10^{-3}$, separating the effect of ensembling from that of different update rates.

\myPara{Training and Evaluation.}
All configurations share the same sensory encoding, observations, supervision, and online learning budget, without replay, so that performance differences primarily reflect the organization of the downstream learning and memory components. Models are optimized online with Adam using a batch size of 64. Each run comprises 50,000 inputs, with performance evaluated every 2,000 inputs at 25 post-update checkpoints. At each checkpoint, the exposed class set contains all classes represented by at least one input observed up to that point, and accuracy is evaluated on the corresponding test examples. Exposed-class anytime AUC is computed as the trapezoidal integral of this checkpoint-wise accuracy over the number of observed inputs, normalized by the evaluated interval. Results are averaged over five independently generated datasets, streams, and model initializations, and reported with 95\% confidence intervals.

\subsection{Continual Stream Construction}
\label{sec:app_cl_protocol}

\begin{table*}[!ht]
\centering
\caption{Detailed continual stream construction and evaluation metrics for all benchmarks. Partition Unit specifies the semantic or distributional unit used to construct sessions. Task Metric denotes the benchmark-specific evaluation measure, while Reported CL Metrics denote its continual evaluation over the stream.}
\label{tab:protocol_details}
\renewcommand\arraystretch{1.12}
\setlength{\tabcolsep}{2.6pt}
\resizebox{\textwidth}{!}{
\begin{tabular}{llllccll}
\toprule
\textbf{Scenario} & \textbf{Benchmark / Task} & \textbf{Partition Unit} &
\textbf{Setting} & \textbf{$T$} & \textbf{$r_D$ / $r_B$} &
\textbf{Task Metric} & \textbf{Reported CL Metrics} \\
\midrule

\multirow{3}{*}{\shortstack[l]{Visual\\recognition}}
& CIFAR-100
& Image class
& GCL
& 5
& $50\%$ / $10\%$
& Accuracy
& $A_{\rm last}$, $A_{\rm auc}$ \\

& ImageNet-R
& Image class
& GCL
& 5
& $50\%$ / $10\%$
& Accuracy
& $A_{\rm last}$, $A_{\rm auc}$ \\

& CUB-200
& Image class
& GCL
& 5
& $50\%$ / $10\%$
& Accuracy
& $A_{\rm last}$, $A_{\rm auc}$ \\
\midrule

\multirow{2}{*}{\shortstack[l]{Vision-language\\learning}}
& CIFAR-100
& Image class
& GCL
& 5
& $50\%$ / $30\%$
& Accuracy
& $A_{\rm last}$, $A_{\rm auc}$ \\

& ImageNet-R
& Image class
& GCL
& 5
& $50\%$ / $30\%$
& Accuracy
& $A_{\rm last}$, $A_{\rm auc}$ \\
\midrule

\multirow{6}{*}{\shortstack[l]{Ego-exo video\\understanding}}
& EgoExoLearn: Skill Assessment
& Action class
& GCL
& 4
& $50\%$ / $30\%$
& Ranking accuracy
& $A_{\rm last}$, $\bar{F}_{T}$, $A_{\rm auc}$ \\

% & EgoExoLearn: Action Segmentation
% & Procedure task
% & GCL
% & 4
% & $50\%$ / $30\%$
% & Frame-wise accuracy
% & $A_{\rm last}$, $A_{\rm auc}$ \\

% & EgoExoLearn: Cross-View Association
% & Procedure task
% & GCL
% & 5
% & $50\%$ / $30\%$
% & Top-1 association accuracy
% & $A_{\rm last}$, $A_{\rm auc}$ \\

& EgoExoLearn: Action Anticipation
& Procedure task
& GCL
& 8
& $50\%$ / $30\%$
& Top-5 recall
& $A_{\rm last}$, $A_{\rm auc}$ \\

% & EgoExoLearn: Action Planning
% & Procedure task
% & GCL
% & 4
% & $50\%$ / $30\%$
% & ED@8
% & $\mathrm{ED}_{\rm last}$, $\mathrm{ED}_{\rm auc}$ \\

\cmidrule(lr){2-8}

& EgoExo-Fitness: Skill Assessment
& Action class
& GCL
& 4
& $50\%$ / $30\%$
& Ranking accuracy
& $A_{\rm last}$, $\bar{F}_{T}$, $A_{\rm auc}$ \\

& EgoExo-Fitness: Action Classification
& Action class
& GCL
& 4
& $50\%$ / $30\%$
& Accuracy
& $A_{\rm last}$, $\bar{F}_{T}$, $A_{\rm auc}$ \\

% & EgoExo-Fitness: Action Localization
% & Action class
% & GCL
% & 4
% & $50\%$ / $30\%$
% & mAP, AP@0.5
% & $A_{\rm last}$, $\bar{F}_{T}$, $A_{\rm auc}$ \\

& EgoExo-Fitness: Sequence Verification
& Action class
& GCL
& 4
& $50\%$ / $30\%$
& ROC-AUC, mAP
& $A_{\rm last}$, $\bar{F}_{T}$, $A_{\rm auc}$ \\

& EgoExo-Fitness: Guidance-based Verification
& Action class
& GCL
& 4
& $50\%$ / $30\%$
& Classification accuracy, F1
& $A_{\rm last}$, $\bar{F}_{T}$, $A_{\rm auc}$ \\
\midrule

\multirow{2}{*}{\shortstack[l]{Vision-language-\\action learning}}
& LIBERO-Spatial, -Object, -Goal, -Long
& Manipulation task
& Offline CL
& 10
& --
& Success rate
& $A_{\rm last}$, $A_{\rm auc}$, FWT, NBT \\

& LIBERO-Spatial, -Object, -Goal, -Long
& Manipulation task
& GCL
& 10
& $50\%$ / $30\%$
& Success rate
& $A_{\rm last}$, $A_{\rm auc}$, FWT, NBT \\
\bottomrule
\end{tabular}}
\end{table*}

Across benchmarks, sessions are constructed from the semantic unit most closely aligned with the prediction target (Supplementary~\cref{tab:protocol_details}). Visual recognition and vision-language benchmarks use image classes, whereas ego-exo video tasks use action classes or procedure-level task groups to preserve the temporal and cross-view structure of the original annotations. LIBERO uses individual manipulation tasks as the partition unit. All stream construction is performed within the official data splits, so redistribution across sessions does not transfer samples between training and evaluation sets.

For GCL evaluation, each semantic unit is assigned a home session and the stream is made blurry by allowing samples from the recurring component to appear outside that session according to $r_B$, while the disjoint component remains primarily session-specific according to $r_D$. Each training sample is processed under the benchmark-specific online budget, and evaluation after session $t$ covers all sessions exposed up to that point. The complementary offline LIBERO protocol instead presents the ten manipulation tasks sequentially and permits multiple passes over each task before moving to the next, while retaining the same task order and evaluation procedure.

\section{Additional Evaluation Metrics}
\label{sec:app_metrics}

This section details the task-specific base metrics and additional CL metrics. Consistent with the notation in the main paper, let $\mathcal{E}_{j}=\{(\mathbf{o}_{j,n},\mathbf{u}_{j,n})\}_{n=1}^{N_j}$ denote the evaluation set of session $j$, where $\mathbf{o}_{j,n}$ is an observation and $\mathbf{u}_{j,n}$ is its associated learning signal. Let $\hat{\mathbf{u}}_{j,n}^{(t)}$ denote the prediction obtained after the model has learned through session $t$. Each task-specific metric defined below instantiates the session-level performance $R_{t,j}$ used in the CL aggregates in the main paper.

\subsection{Task Metrics}

\myPara{Accuracy-Based Metrics.}
For visual recognition, vision-language classification, and action classification, the learning signal is a categorical label $u_{j,n}\in\mathcal{U}_{j}$, and we use classification accuracy:
\begin{equation}
    \mathrm{Acc}_{t,j}
    =
    \frac{1}{N_j}
    \sum_{n=1}^{N_j}
    \mathbf{1}
    \left[
        \hat{u}_{j,n}^{(t)} = u_{j,n}
    \right].
    \label{eq:metric_acc}
\end{equation}

\myPara{Skill Assessment.}
For skill assessment, the learning signal specifies the relative quality ordering of two video observations. Let $\mathcal{P}_{j}$ denote the set of annotated ordered pairs, where $(a,b)\in\mathcal{P}_{j}$ indicates that $\mathbf{o}_{j,a}$ exhibits higher skill than $\mathbf{o}_{j,b}$. Given the predicted skill scores $\hat{q}_{j,a}^{(t)}$ and $\hat{q}_{j,b}^{(t)}$, ranking accuracy is defined as
\begin{equation}
    \mathrm{RAcc}_{t,j}
    =
    \frac{1}{\lvert \mathcal{P}_{j} \rvert}
    \sum_{(a,b)\in\mathcal{P}_{j}}
    \mathbf{1}
    \left[
        \hat{q}_{j,a}^{(t)} > \hat{q}_{j,b}^{(t)}
    \right].
    \label{eq:metric_ranking_acc}
\end{equation}

\myPara{Action Anticipation.}
For action anticipation, the learning signal is the future verb or noun class, and we report class-mean Top-$K$ recall with $K=5$ separately for verb and noun prediction. Let $\mathcal{C}$ denote the evaluated class set, $\mathcal{E}_{j,c}=\{n:u_{j,n}=c\}$ denote the evaluation samples belonging to class $c$, and $\operatorname{TopK}(\mathbf{s}_{j,n}^{(t)})$ denote the $K$ classes with the largest predicted scores. The metric is defined as
\begin{equation}
    \mathrm{R}@K_{t,j}
    =
    \frac{1}{\lvert \mathcal{C} \rvert}
    \sum_{c\in\mathcal{C}}
    \frac{1}{\lvert \mathcal{E}_{j,c} \rvert}
    \sum_{n\in\mathcal{E}_{j,c}}
    \mathbf{1}
    \left[
        c \in \operatorname{TopK}
        \left(
            \mathbf{s}_{j,n}^{(t)}
        \right)
    \right].
    \label{eq:metric_topk_recall}
\end{equation}

\myPara{Sequence and Guidance-Based Verification.}
For binary sequence verification, the learning signal indicates whether an observed execution sequence is valid. We report ROC-AUC and mAP. Let $\mathcal{E}_{j}^{+}$ and $\mathcal{E}_{j}^{-}$ denote the positive and negative sample sets, respectively, and let $\hat{s}_{j,n}^{(t)}$ be the predicted verification score. ROC-AUC is defined as
\begin{equation}
    \mathrm{ROC\mbox{-}AUC}_{t,j}
    =
    \frac{1}{
        \lvert \mathcal{E}_{j}^{+} \rvert
        \lvert \mathcal{E}_{j}^{-} \rvert
    }
    \sum_{p\in\mathcal{E}_{j}^{+}}
    \sum_{n\in\mathcal{E}_{j}^{-}}
    \mathbf{1}
    \left[
        \hat{s}_{j,p}^{(t)} > \hat{s}_{j,n}^{(t)}
    \right].
    \label{eq:metric_roc_auc}
\end{equation}
Following the standard ranking-based evaluation, mAP is computed as
\begin{equation}
    \mathrm{mAP}_{t,j}
    =
    \frac{1}{\lvert \mathcal{Q}_{j} \rvert}
    \sum_{q\in\mathcal{Q}_{j}}
    \mathrm{AP}_{q}^{(t)},
    \label{eq:metric_verification_map}
\end{equation}
where $\mathcal{Q}_{j}$ denotes the evaluated verification queries or groups. For guidance-based execution verification, we additionally report classification accuracy as defined in Eq.~\eqref{eq:metric_acc} and F1 score:
\begin{align}
    \mathrm{F1}_{t,j}
    &=
    \frac{
        2\,\mathrm{Prec}_{t,j}\,\mathrm{Rec}_{t,j}
    }{
        \mathrm{Prec}_{t,j}+\mathrm{Rec}_{t,j}
    }, \label{eq:metric_f1}
    \\
    \mathrm{Prec}_{t,j}
    &=
    \frac{\mathrm{TP}_{t,j}}
         {\mathrm{TP}_{t,j}+\mathrm{FP}_{t,j}},
    \\
    \mathrm{Rec}_{t,j}
    &=
    \frac{\mathrm{TP}_{t,j}}
         {\mathrm{TP}_{t,j}+\mathrm{FN}_{t,j}}.
\end{align}

\myPara{Embodied Vision-Language-Action.}
For embodied vision-language-action, the learning signal corresponds to an action or trajectory and the evaluation target is successful task completion. Let $z_{j,n}^{(t)}\in\{0,1\}$ indicate whether the learned policy successfully completes episode observation $\mathbf{o}_{j,n}$ after learning through session $t$. The success rate is defined as
\begin{equation}
    \mathrm{SR}_{t,j}
    =
    \frac{1}{N_j}
    \sum_{n=1}^{N_j}
    z_{j,n}^{(t)}.
    \label{eq:metric_success_rate}
\end{equation}

\subsection{Continual Learning Metrics}
Each higher-is-better task metric above is instantiated as $R_{t,j}$ in the main-paper definitions of $A_{\rm last}$ and $A_{\rm auc}$, where $R_{t,j}$ denotes the performance on session $j$ after learning through session $t$. For selected benchmarks, we additionally report average forgetting ($\bar{F}_T$), forward transfer (FWT), and negative backward transfer (NBT). 

\myPara{Average Forgetting.}
For a higher-is-better base metric, $\bar{F}_T$ is defined as
\begin{equation}
    \bar{F}_{T}
    =
    \frac{1}{T-1}
    \sum_{j=1}^{T-1}
    \left(
        \max_{t\in\{j,\ldots,T-1\}} R_{t,j}
        -
        R_{T,j}
    \right),
    \label{eq:metric_forgetting}
\end{equation}
which measures the average degradation from the best historical performance on each previously observed session to its final performance.

\myPara{Forward Transfer} measures how knowledge acquired from earlier continual sessions improves performance on a new session before it is learned. Let $b_j$ denote the performance of an untrained or reference model on session $j$. We define
\begin{equation}
    \mathrm{FWT}
    =
    \frac{1}{T-1}
    \sum_{j=2}^{T}
    \left(
        R_{j-1,j}
        -
        b_j
    \right),
    \label{eq:metric_fwt}
\end{equation}
where higher values indicate stronger transfer from previous to future sessions.

\myPara{Negative Backward Transfer} measures the effect of subsequent learning on previously learned sessions relative to their performance immediately after acquisition:
\begin{equation}
    \mathrm{NBT}
    =
    \frac{1}{T-1}
    \sum_{j=1}^{T-1}
    \left(
        R_{j,j}
        -
        R_{T,j}
    \right),
    \label{eq:metric_nbt}
\end{equation}
where lower values indicate less interference to previously learned sessions.

\clearpage
\section{Additional Implementation Details}
\label{sec:app_impl}

\begin{table*}[t]
\centering
\caption{Method-level optimization configurations across benchmark groups. Each method is listed separately within the benchmark group where it is evaluated. All methods within the same benchmark follow the same random seed, stream construction, session order, online update budget, and evaluation schedule. Hyperparameters are selected on the validation split of the first stage stream and then fixed across sessions. Opt., optimizer. LR, learning rate. WD, weight decay. Mem., memory buffer capacity for replay-based ER and DER++.}
\label{tab:method_optimization_config}
\renewcommand\arraystretch{1.12}
\setlength{\tabcolsep}{2.4pt}
\resizebox{\textwidth}{!}{
\begin{tabular}{m{2cm} m{3cm} m{3.5cm} m{1.35cm} m{1.35cm} m{1.35cm} m{1.25cm} m{1.2cm} m{6cm}}
\toprule
\textbf{Benchmark}
& \textbf{Method}
& \textbf{Backbone}
& \textbf{Opt.}
& \textbf{LR}
& \textbf{WD}
& \textbf{Batch}
& \textbf{Mem.}
& \textbf{Notes} \\
\midrule

\multirow{15}{*}{\shortstack[l]{Visual\\Recog.}}
& SeqFT
& ViT-B/16
& Adam
& $5\times10^{-3}$
& 0
& 64
& --
& Sequentially fine-tunes the full network. \\

& EWC~\cite{kirkpatrick2017overcoming}
& ViT-B/16
& Adam
& $5\times10^{-3}$
& 0
& 64
& --
& Uses regularization based on the Fisher information matrix. \\

& LwF~\cite{li2017learning}
& ViT-B/16
& Adam
& $5\times10^{-3}$
& 0
& 64
& --
& Uses output distillation from the previous model. \\

& L2P~\cite{wang2022learning}
& ViT-B/16
& Adam
& $5\times10^{-3}$
& 0
& 64
& --
& Uses a learnable prompt pool for continual adaptation. \\

& DualPrompt~\cite{wang2022dualprompt}
& ViT-B/16
& Adam
& $5\times10^{-3}$
& 0
& 64
& --
& Uses task-shared and task-specific prompts. \\

& MVP~\cite{moon2023online}
& ViT-B/16
& Adam
& $5\times10^{-3}$
& 0
& 64
& --
& Uses GCL-oriented prompt adaptation. \\

& MISA~\cite{kang2025advancing}
& ViT-B/16
& Adam
& $5\times10^{-3}$
& 0
& 64
& --
& Initializes adaptation with pretrained prompts. \\

& FlyGCL
& ViT-B/16
& Adam
& $5\times10^{-3}$
& 0
& 64
& --
& Uses adapter, prompt, or LoRA experts; the random expansion dimension and regularization parameter for regression are set to 10,000, and the temporal-ensemble EMA decay rates are 0.9 and 0.99. \\
\midrule

\multirow{13}{*}{\shortstack[l]{Vision\\Language}}
& SeqFT
& CLIP ViT-B/16
& AdamW
& $5\times10^{-4}$
& $5\times10^{-4}$
& 64
& --
& Sequentially fine-tunes the CLIP vision encoder. \\

& EWC~\cite{kirkpatrick2017overcoming}
& CLIP ViT-B/16
& AdamW
& $5\times10^{-4}$
& $5\times10^{-4}$
& 64
& --
& Uses regularization based on the Fisher information matrix. \\

& LwF~\cite{li2017learning}
& CLIP ViT-B/16
& AdamW
& $5\times10^{-4}$
& $5\times10^{-4}$
& 64
& --
& Uses output distillation from the previous model. \\

& L2P~\cite{wang2022learning}
& CLIP ViT-B/16
& AdamW
& $5\times10^{-4}$
& $5\times10^{-4}$
& 64
& --
& Uses a learnable prompt pool for continual adaptation. \\

& DualPrompt~\cite{wang2022dualprompt}
& CLIP ViT-B/16
& AdamW
& $5\times10^{-4}$
& $5\times10^{-4}$
& 64
& --
& Uses task-shared and task-specific prompts. \\

& CLAP4CLIP~\cite{jha2024clap4clip}
& CLIP ViT-B/16
& AdamW
& $5\times10^{-4}$
& $5\times10^{-4}$
& 64
& --
& Provides a CLIP-specific continual-learning baseline. \\

& MG-CLIP~\cite{huang2025mind}
& CLIP ViT-B/16
& AdamW
& $5\times10^{-4}$
& $5\times10^{-4}$
& 64
& --
& Provides a CLIP-specific alignment baseline. \\

& FlyGCL
& CLIP ViT-B/16
& AdamW
& $5\times10^{-4}$
& $5\times10^{-4}$
& 64
& --
& Uses LoRA experts. \\
\midrule

\multirow{13}{*}{\shortstack[l]{EgoExo-\\Learn \&-\\Fitness}}
& SeqFT
& I3D / CLIP encoder
& AdamW
& $10^{-4}$
& $5\times10^{-4}$
& 32
& --
& Sequentially fine-tunes the full network. \\

& ER~\cite{rolnick2019experience}
& I3D / CLIP encoder
& AdamW
& $10^{-4}$
& $5\times10^{-4}$
& 32
& 10\%
& Uses reservoir-based experience replay. \\

& DER++~\cite{buzzega2020dark}
& I3D / CLIP encoder
& AdamW
& $10^{-4}$
& $5\times10^{-4}$
& 32
& 10\%
& Combines experience replay with logit-level distillation. \\

& EWC~\cite{kirkpatrick2017overcoming}
& I3D / CLIP encoder
& AdamW
& $10^{-4}$
& $5\times10^{-4}$
& 32
& --
& Uses regularization based on the Fisher information matrix. \\

& LwF~\cite{li2017learning}
& I3D / CLIP encoder
& AdamW
& $10^{-4}$
& $5\times10^{-4}$
& 32
& --
& Uses output distillation from the previous model. \\

& L2P+~\cite{wang2022learning}
& I3D / CLIP encoder
& AdamW
& $10^{-4}$
& $5\times10^{-4}$
& 32
& --
& Uses a learnable adapter pool for continual adaptation. \\

& DualPrompt+~\cite{wang2022dualprompt}
& I3D / CLIP encoder
& AdamW
& $10^{-4}$
& $5\times10^{-4}$
& 32
& --
& Uses task-shared and task-specific adapters. \\

& S-Prompt+~\cite{wang2022s}
& I3D / CLIP encoder
& AdamW
& $10^{-4}$
& $5\times10^{-4}$
& 32
& --
& Uses adapter-based adaptation and routing for sequential video tasks. \\

& FlyGCL
& I3D / CLIP encoder
& AdamW
& $10^{-4}$
& $5\times10^{-4}$
& 32
& --
& Uses adapter experts. \\
\midrule

% \multirow{9}{*}{\shortstack[l]{EgoExo-\\Fitness}}
% & SeqFT
% & I3D / CLIP encoder
% & AdamW
% & $10^{-4}$
% & $5\times10^{-4}$
% & 32
% & --
% & Sequentially fine-tunes the full network. \\

% & ER~\cite{rolnick2019experience}
% & I3D / CLIP encoder
% & AdamW
% & $10^{-4}$
% & $5\times10^{-4}$
% & 32
% & 10\%
% & Uses reservoir-based experience replay. \\

% & DER++~\cite{buzzega2020dark}
% & I3D / CLIP encoder
% & AdamW
% & $10^{-4}$
% & $5\times10^{-4}$
% & 32
% & 10\%
% & Combines experience replay with logit-level distillation. \\

% & EWC~\cite{kirkpatrick2017overcoming}
% & I3D / CLIP encoder
% & AdamW
% & $10^{-4}$
% & $5\times10^{-4}$
% & 32
% & --
% & Uses regularization based on the Fisher information matrix. \\

% & LwF~\cite{li2017learning}
% & I3D / CLIP encoder
% & AdamW
% & $10^{-4}$
% & $5\times10^{-4}$
% & 32
% & --
% & Uses output distillation from the previous model. \\

% & L2P+~\cite{wang2022learning}
% & I3D / CLIP encoder
% & AdamW
% & $10^{-4}$
% & $5\times10^{-4}$
% & 32
% & --
% & Uses a learnable adapter pool for continual adaptation. \\

% & DualPrompt+~\cite{wang2022dualprompt}
% & I3D / CLIP encoder
% & AdamW
% & $10^{-4}$
% & $5\times10^{-4}$
% & 32
% & --
% & Uses task-shared and task-specific adapters. \\

% & S-Prompt+~\cite{wang2022s}
% & I3D / CLIP encoder
% & AdamW
% & $10^{-4}$
% & $5\times10^{-4}$
% & 32
% & --
% & Uses prompt-based adaptation for sequential video tasks. \\

% & FlyGCL
% & I3D / CLIP encoder
% & AdamW
% & $10^{-4}$
% & $5\times10^{-4}$
% & 32
% & --
% & Uses adapter experts. \\
% \midrule

\multirow{13}{*}{\shortstack[l]{LIBERO}}
& SeqFT
& DiT flow-matching
& Adam
& $10^{-4}$
& $10^{-6}$
& 32
& --
& Sequentially fine-tunes the full network. \\

& SeqLoRA
& DiT flow-matching
& Adam
& $10^{-4}$
& $10^{-6}$
& 32
& --
& Sequentially updates LoRA parameters. \\

& PackNet~\cite{mallya2018packnet}
& DiT flow-matching
& Adam
& $10^{-4}$
& $10^{-6}$
& 32
& --
& Isolates parameters through model pruning. \\

& ER~\cite{rolnick2019experience}
& DiT flow-matching
& Adam
& $10^{-4}$
& $10^{-6}$
& $32$
& $10\%$
& Uses reservoir-based experience replay. \\

& EWC~\cite{kirkpatrick2017overcoming}
& DiT flow-matching
& Adam
& $10^{-4}$
& $10^{-6}$
& 32
& --
& Uses regularization based on the Fisher information matrix. \\

& LwF~\cite{li2017learning}
& DiT flow-matching
& Adam
& $10^{-4}$
& $10^{-6}$
& 32
& --
& Uses output distillation from the previous model. \\

& L2P+~\cite{wang2022learning}
& DiT flow-matching
& Adam
& $10^{-4}$
& $10^{-6}$
& 32
& --
& Uses a learnable adapter pool for continual adaptation. \\

& DualPrompt+~\cite{wang2022dualprompt}
& DiT flow-matching
& Adam
& $10^{-4}$
& $10^{-6}$
& 32
& --
& Uses task-shared and task-specific adapters. \\

& FlyGCL
& DiT flow-matching
& Adam
& $10^{-4}$
& $10^{-6}$
& 32
& --
& Uses adapter experts. \\
\bottomrule
\end{tabular}}
\end{table*}

The visual-recognition experiments use frozen ViT-B/16 backbones with supervised or self-supervised ImageNet pretraining, as specified in Methods and Supplementary~\cref{tab:method_optimization_config}. The vision-language experiments use the OpenAI CLIP ViT-B/16 model with frozen pretrained encoders and update only the method-specific lightweight modules and output components. For ego-exo video understanding, we retain each benchmark's pretrained feature extraction and task-specific prediction heads, changing only the continual adaptation module. The embodied vision-language-action experiments initialize a DiT flow-matching policy from a LIBERO-90 pretrained checkpoint (built upon DINOv2 vision encoder and CLIP text encoder), and use the same observation and policy pipeline across methods within each suite. Within each benchmark and task, methods share the same stream, session order, data preprocessing, mini-batch schedule, and evaluation checkpoints. Optimization settings that of each method are listed in Supplementary~\cref{tab:method_optimization_config}. FlyGCL freezes the pretrained backbone and updates only the active adaptive expert and its output heads; the random-expanded router is updated from accumulated sufficient statistics as feature covariance and mean. Unless a method intrinsically requires stored samples, no replay memory is used. Method-specific regularization coefficients, prompt or LoRA dimensions, and other configurations follow the corresponding implementations and are held fixed across sessions after selection on the initial validation stream.

\section{Extensions beyond the Conference Version}
\label{sec:app_extension}

An earlier conference paper, FlyPrompt~\cite{yan2026flyprompt}, explored a preliminary brain-inspired GCL framework based on random-expanded expert routing and temporal-ensemble prompt experts. That study focused on continual image classification with prompt-based adaptation over pretrained vision models. The present work substantially extends this preliminary study in its scientific formulation, biological grounding, model generality, theoretical analysis, and empirical scope.

First, FlyGCL develops a broader hierarchical modular principle for GCL. The present work formulates GCL around two complementary requirements: separating conflicting experience to reduce interference and integrating compatible experience to promote generalization. We relate these requirements to MoE and EL, respectively, and study how their hierarchical coordination supports learning under online, uncertain, and evolving data distributions. This formulation generalizes the original algorithmic design into a broader principle for organizing CL.

Second, the biological grounding is substantially expanded. FlyPrompt focused on sparse random expansion and multi-timescale memory in the \emph{Drosophila} olfactory system. FlyGCL develops a more complete correspondence with the hierarchical organization of olfactory learning and memory, including sparse sensory expansion, spatially differentiated downstream pathways, and learning and memory processes operating across distinct temporal scales. We further introduce a controlled biologically grounded olfactory learning model that directly evaluates spatial specialization, temporal integration, and their hierarchical coordination under continual odor streams with different degrees of distributional recurrence.

Third, FlyGCL extends the model and theoretical scope beyond prompt-based visual learning. The hierarchical design can be instantiated with different lightweight learning modules, including prompts, adapters, LoRA branches, video-specific modules, task heads, policy adapters, and action heads, while retaining pretrained backbones as relatively stable representational substrates. We further develop a new theoretical analysis of hierarchical specialization and integration, decomposing GCL risk into separation, integration, and coordination terms and analyzing when the gains from routing and aggregation outweigh their coordination cost. The analysis also characterizes how stable pretrained representations can reduce the penalty of imperfect coordination.

The empirical evaluation is also substantially expanded within visual CL. Beyond reproducing the original image-classification setting, FlyGCL is evaluated across multiple pretrained representations and lightweight learning interfaces, including prompts, adapters, and LoRA. We further analyze the complementary contributions of MoE-based specialization and EL-based integration across different datasets and pretrained backbones. These experiments examine whether the proposed principle remains effective beyond a particular prompt design or backbone initialization, and establish its generality across diverse pretrained visual representations.
   
We then extend GCL from unimodal visual recognition to multimodal and temporal understanding. FlyGCL is evaluated on continual vision-language learning with pretrained CLIP models, where learning must preserve cross-modal semantic structure while accommodating evolving visual concepts. Beyond overall performance, we analyze changes in image-text representation spaces and the preservation of semantic relations under continual updates. We further evaluate continual ego-exo video understanding, where temporal dynamics, viewpoint shifts, and skill-related variations introduce substantially more complex distribution changes than static image classification. These settings test hierarchical specialization and integration across multimodal semantics, temporal structure, and cross-view experience.

Finally, we extend GCL to embodied vision-language-action learning, where continual updates directly affect sequential decision-making and action policies. FlyGCL is evaluated on language-conditioned robotic manipulation under evolving task distributions, providing a substantially more challenging setting than recognition-oriented benchmarks. Together with the controlled olfactory simulation, these experiments expand the empirical scope from static visual classification to biological modeling, multimodal perception, temporal multi-view understanding, and embodied action. They therefore validate the proposed hierarchical modular principle across a much broader range of models, modalities, and CL scenarios than the conference study.

\clearpage
\section{Additional Results}

\begin{table}[htbp!]
\centering

\caption{Performance comparison on continual visual recognition benchmarks under the GCL setting using ViT-B/16 backbones with different pretraining. We report average anytime accuracy $A_{\rm auc}$ ($\uparrow$), final average accuracy $A_{\rm last}$ ($\uparrow$), and forgetting $F$ ($\downarrow$). SL, use small backbone learning rate (1/100 of classifier learnign rate).}
\label{tab:img}
\resizebox{\textwidth}{!}{
\begin{tabular}{
l
*{9}{S[table-format=3.2(2)]}
}
\toprule
\multirow{2.5}{*}{\textbf{Method}} 
& \multicolumn{3}{c}{\textbf{CIFAR-100}} 
& \multicolumn{3}{c}{\textbf{ImageNet-R}} 
& \multicolumn{3}{c}{\textbf{CUB-200}} \\
\cmidrule(lr){2-4}\cmidrule(lr){5-7}\cmidrule(lr){8-10}
& {$A_{\rm{auc}} (\uparrow)$} 
& {$A_{\rm{last}} (\uparrow)$}
& {$F (\downarrow)$}
& {$A_{\rm{auc}} (\uparrow)$} 
& {$A_{\rm{last}} (\uparrow)$}
& {$F (\downarrow)$}
& {$A_{\rm{auc}} (\uparrow)$} 
& {$A_{\rm{last}} (\uparrow)$}
& {$F (\downarrow)$} \\
\midrule
\multicolumn{10}{l}{\textit{Backbone: Sup-21K}} \\
SeqFT
& 19.71\std{3.39} & 10.42\std{4.92} & 32.25\std{16.74}
& 7.51\std{3.94} & 2.29\std{0.85} & 74.79\std{23.89}
& 3.47\std{0.41} & 1.49\std{0.42} & 54.01\std{24.57} \\

Linear Probe
& 49.69\std{6.09} & 23.07\std{7.33} & 11.59\std{1.63}
& 29.24\std{1.26} & 16.87\std{3.14} & 18.95\std{1.44}
& 28.96\std{2.46} & 17.33\std{3.08} & 28.82\std{3.02} \\

SeqFT w/ SL
& 64.90\std{7.18} & 62.06\std{1.89} & 21.09\std{3.82}
& 47.20\std{1.47} & 39.60\std{2.43} & 33.99\std{7.14}
& 56.16\std{4.32} & 56.50\std{3.08} & 29.05\std{1.73} \\

EWC
& 55.94\std{8.36} & 58.13\std{1.77} & 41.13\std{4.59}
& 36.29\std{1.51} & 33.92\std{1.86} & 50.28\std{12.28}
& 40.27\std{4.03} & 35.57\std{3.25} & 68.73\std{13.00} \\

LwF
& 62.14\std{8.35} & 71.59\std{3.20} & 18.47\std{7.95}
& 37.91\std{1.60} & 36.66\std{2.35} & 48.01\std{13.05}
& 41.08\std{3.94} & 37.00\std{3.24} & 67.05\std{13.30} \\

L2P
& 76.23\std{2.73} & 79.11\std{1.43} & 11.53\std{1.44}
& 44.40\std{1.03} & 42.03\std{1.72} & 19.16\std{1.92}
& 64.30\std{2.18} & 61.42\std{2.13} & 29.15\std{3.19} \\

DualPrompt
& 76.04\std{3.32} & 76.62\std{0.74} & 11.42\std{0.91}
& 46.13\std{1.94} & 40.80\std{1.04} & 18.24\std{4.34}
& 65.03\std{2.24} & 62.43\std{1.78} & 27.18\std{2.91} \\

CODA-P
& 79.13\std{3.06} & 80.91\std{0.70} & 9.95\std{1.30}
& 51.87\std{2.81} & 48.09\std{2.75} & 18.49\std{2.44}
& 66.01\std{2.20} & 62.90\std{2.46} & 27.43\std{3.42} \\

MVP
& 67.74\std{4.96} & 63.22\std{0.69} & 33.19\std{2.33}
& 39.50\std{1.41} & 32.63\std{3.95} & 43.39\std{4.23}
& 54.69\std{3.14} & 50.07\std{3.86} & 47.61\std{5.73} \\

MISA
& 80.35\std{2.39} & 80.75\std{1.24} & 9.67\std{1.39}
& 51.52\std{2.09} & 45.08\std{1.43} & 21.46\std{4.25}
& 65.40\std{3.01} & 60.20\std{1.82} & 30.06\std{3.70} \\

\rowcolor{gray!10}
FlyGCL-Prompt
& 83.77\std{2.21} & 87.59\std{0.38} & \bfseries 4.31\std{0.82}
& 60.22\std{1.18} & 59.62\std{0.50} & \bfseries 10.74\std{1.69}
& \bfseries 78.02\std{2.74} & \bfseries 84.93\std{0.50} & 3.30\std{1.72} \\

\rowcolor{gray!10}
FlyGCL-Adapter
& 83.65\std{2.31} & 87.29\std{0.21} & 4.83\std{0.59}
& 63.91\std{1.22} & 62.64\std{0.24} & 12.13\std{1.57}
& 77.61\std{2.78} & 84.37\std{0.11} & \bfseries 3.27\std{1.60} \\

\rowcolor{gray!10}
FlyGCL-LoRA
& \bfseries 83.87\std{2.22} & \bfseries 87.64\std{0.19} & 4.46\std{0.62}
& \bfseries 65.32\std{1.17} & \bfseries 63.74\std{0.58} & 11.79\std{2.20}
& 77.95\std{2.49} & 84.79\std{0.42} & 3.55\std{1.96} \\

\midrule
\multicolumn{10}{l}{\textit{Backbone: Sup-21K/1K}} \\
L2P
& 63.88\std{7.79} & 68.96\std{7.63} & 14.88\std{6.27}
& 47.10\std{1.21} & 42.22\std{1.94} & 35.88\std{4.51}
& 42.96\std{4.13} & 45.00\std{3.83} & 35.05\std{9.31} \\

DualPrompt
& 68.02\std{2.08} & 67.04\std{5.84} & 19.18\std{5.13}
& 52.80\std{1.21} & 47.39\std{1.60} & 30.73\std{7.45}
& 46.80\std{2.89} & 46.39\std{2.76} & 35.40\std{8.16} \\

CODA-P
& 69.29\std{2.52} & 69.47\std{7.19} & 19.82\std{6.85}
& 51.20\std{1.76} & 44.30\std{1.50} & 35.58\std{5.52}
& 44.66\std{2.73} & 45.18\std{4.50} & 35.73\std{8.95} \\

MVP
& 64.69\std{3.77} & 51.29\std{7.56} & 46.98\std{8.15}
& 48.99\std{2.01} & 38.12\std{5.20} & 51.11\std{2.86}
& 44.10\std{2.81} & 33.97\std{9.62} & 62.93\std{8.05} \\

MISA
& 62.91\std{7.96} & 67.99\std{7.41} & 11.38\std{2.76}
& 50.87\std{1.69} & 47.75\std{2.87} & 28.66\std{8.19}
& 42.76\std{2.33} & 44.05\std{1.94} & 33.95\std{8.63} \\

\rowcolor{gray!10}
FlyGCL-Prompt
& 80.72\std{1.61} & 86.74\std{0.70} & \bfseries 3.09\std{0.86}
& 67.13\std{1.21} & 65.78\std{0.69} & \bfseries 12.10\std{1.92}
& 69.82\std{3.66} & 79.16\std{1.18} & 3.41\std{3.61} \\

\rowcolor{gray!10}
FlyGCL-Adapter
& \bfseries 82.79\std{1.71} & \bfseries 87.42\std{0.81} & 3.61\std{1.09}
& \bfseries 73.09\std{1.44} & \bfseries 70.97\std{0.92} & 13.22\std{2.06}
& \bfseries 71.46\std{4.02} & \bfseries 81.23\std{1.68} & 3.25\std{2.51} \\

\rowcolor{gray!10}
FlyGCL-LoRA
& 80.60\std{1.45} & 86.39\std{0.93} & 3.17\std{1.12}
& 71.19\std{1.24} & 69.30\std{0.61} & 12.32\std{2.46}
& 70.35\std{3.54} & 79.91\std{0.88} & \bfseries 2.10\std{3.58} \\

\midrule
\multicolumn{10}{l}{\textit{Backbone: iBOT-21K}} \\
L2P
& 56.82\std{8.42} & 67.61\std{8.76} & 12.58\std{6.24}
& 35.97\std{1.62} & 36.95\std{2.44} & 34.17\std{7.54}
& 14.76\std{1.53} & 24.51\std{4.82} & 19.09\std{8.05} \\

DualPrompt
& 66.06\std{4.52} & 67.14\std{8.60} & 19.22\std{4.22}
& 42.48\std{1.62} & 35.91\std{0.88} & 37.96\std{9.50}
& 19.90\std{3.68} & 21.84\std{2.35} & 30.18\std{8.79} \\

CODA-P
& 62.13\std{7.17} & 63.38\std{7.98} & 22.20\std{5.61}
& 45.50\std{1.66} & 39.44\std{1.35} & 41.21\std{6.19}
& 17.72\std{5.33} & 20.82\std{7.66} & 28.98\std{9.54} \\

MVP
& 62.33\std{3.06} & 48.32\std{11.42} & 50.23\std{11.27}
& 41.55\std{1.98} & 29.29\std{5.03} & 62.17\std{6.10}
& 28.73\std{3.18} & 23.62\std{9.51} & 62.20\std{5.81} \\

MISA
& 65.30\std{2.28} & 67.43\std{6.75} & 17.34\std{6.20}
& 40.94\std{1.22} & 36.16\std{1.58} & 35.62\std{11.69}
& 18.62\std{3.36} & 23.66\std{2.21} & 21.18\std{10.00} \\

\rowcolor{gray!10}
FlyGCL-Prompt
& 78.29\std{1.98} & \bfseries 85.49\std{2.13} & 2.74\std{1.38}
& 61.08\std{1.99} & 61.61\std{0.37} & \bfseries 16.32\std{3.03}
& 34.60\std{6.33} & 48.13\std{4.35} & 5.44\std{5.33} \\

\rowcolor{gray!10}
FlyGCL-Adapter
& \bfseries 79.27\std{1.97} & 85.45\std{2.04} & 2.65\std{1.39}
& \bfseries 64.82\std{1.41} & \bfseries 64.41\std{0.06} & 17.02\std{3.20}
& \bfseries 46.31\std{6.04} & \bfseries 60.75\std{3.11} & \bfseries 4.75\std{4.54} \\

\rowcolor{gray!10}
FlyGCL-LoRA
& 75.02\std{1.69} & 83.45\std{2.50} & \bfseries 2.34\std{1.52}
& 63.11\std{2.05} & 62.07\std{1.02} & 16.46\std{4.10}
& 44.27\std{4.57} & 56.80\std{4.87} & 6.25\std{3.83} \\

\midrule
\multicolumn{10}{l}{\textit{Backbone: iBOT-1K}} \\
L2P
& 53.17\std{7.08} & 62.28\std{8.19} & 16.47\std{7.80}
& 38.29\std{2.65} & 39.86\std{0.95} & 33.01\std{8.96}
& 19.20\std{2.21} & 31.21\std{5.24} & 21.26\std{11.09} \\

DualPrompt
& 52.39\std{3.21} & 53.56\std{6.10} & 20.45\std{3.88}
& 45.76\std{1.63} & 39.19\std{0.65} & 35.14\std{7.72}
& 29.32\std{3.15} & 30.53\std{5.33} & 32.70\std{9.00} \\

CODA-P
& 59.29\std{4.03} & 61.30\std{6.73} & 22.72\std{6.32}
& 49.56\std{1.57} & 42.64\std{2.78} & 40.23\std{6.19}
& 27.57\std{2.83} & 33.61\std{4.52} & 27.71\std{9.65} \\

MVP
& 57.52\std{3.62} & 44.08\std{12.42} & 53.47\std{12.87}
& 44.76\std{2.23} & 34.93\std{4.48} & 56.26\std{5.49}
& 33.81\std{3.50} & 26.32\std{9.97} & 62.61\std{6.88} \\

MISA
& 54.31\std{2.91} & 55.89\std{5.10} & 19.16\std{4.04}
& 43.91\std{3.95} & 40.09\std{1.24} & 33.52\std{10.12}
& 27.76\std{2.69} & 33.74\std{2.11} & 29.35\std{6.60} \\

\rowcolor{gray!10}
FlyGCL-Prompt
& 72.83\std{2.77} & \bfseries 81.88\std{3.12} & 3.47\std{1.74}
& 62.78\std{1.62} & 62.77\std{1.08} & \bfseries 15.10\std{1.44}
& 45.98\std{6.11} & 59.28\std{2.18} & \bfseries 4.75\std{4.44} \\

\rowcolor{gray!10}
FlyGCL-Adapter
& \bfseries 73.38\std{2.84} & 81.14\std{2.53} & 3.78\std{2.05}
& \bfseries 66.12\std{1.17} & \bfseries 64.07\std{0.82} & 16.43\std{1.74}
& \bfseries 53.45\std{5.36} & \bfseries 64.90\std{1.85} & 5.72\std{4.25} \\

\rowcolor{gray!10}
FlyGCL-LoRA
& 66.59\std{5.95} & 76.56\std{5.71} & \bfseries 2.52\std{2.11}
& 64.57\std{1.75} & 62.02\std{0.82} & 16.55\std{2.22}
& 52.25\std{5.52} & 63.54\std{1.90} & 5.13\std{4.29} \\

\midrule
\multicolumn{10}{l}{\textit{Backbone: DINO-1K}} \\
L2P
& 47.98\std{7.38} & 59.13\std{6.32} & 14.82\std{6.81}
& 35.81\std{1.37} & 36.58\std{1.31} & 33.00\std{7.22}
& 21.18\std{2.01} & 32.47\std{6.10} & 24.23\std{9.93} \\

DualPrompt
& 52.12\std{4.01} & 55.71\std{6.11} & 18.69\std{5.27}
& 43.03\std{1.12} & 35.40\std{1.40} & 37.89\std{7.65}
& 27.80\std{4.21} & 29.49\std{4.24} & 31.35\std{10.61} \\

CODA-P
& 54.69\std{4.49} & 58.91\std{5.43} & 19.43\std{5.83}
& 45.16\std{2.05} & 38.23\std{2.02} & 41.26\std{6.20}
& 29.22\std{2.97} & 31.85\std{7.47} & 32.69\std{10.65} \\

MVP
& 53.64\std{3.91} & 41.02\std{12.09} & 54.70\std{12.47}
& 41.78\std{2.15} & 32.00\std{4.22} & 56.54\std{6.00}
& 33.44\std{3.43} & 26.02\std{10.29} & 64.21\std{7.82} \\

MISA
& 52.03\std{3.07} & 55.98\std{4.26} & 17.76\std{4.28}
& 41.26\std{3.25} & 37.50\std{1.62} & 34.38\std{10.60}
& 27.13\std{3.31} & 33.08\std{4.10} & 25.77\std{10.77} \\

\rowcolor{gray!10}
FlyGCL-Prompt
& 66.93\std{5.69} & \bfseries 78.07\std{5.66} & 3.13\std{1.20}
& 59.48\std{1.43} & 58.92\std{0.84} & 16.32\std{2.17}
& 46.33\std{6.20} & 60.50\std{2.62} & \bfseries 4.20\std{4.88} \\

\rowcolor{gray!10}
FlyGCL-Adapter
& \bfseries 67.92\std{4.71} & 77.67\std{5.29} & 3.02\std{2.60}
& \bfseries 62.03\std{1.50} & \bfseries 61.16\std{1.09} & \bfseries 15.95\std{1.96}
& \bfseries 54.34\std{5.17} & \bfseries 66.45\std{1.92} & 4.49\std{4.32} \\

\rowcolor{gray!10}
FlyGCL-LoRA
& 67.00\std{3.82} & 76.12\std{3.02} & \bfseries 2.97\std{2.89}
& 61.61\std{1.75} & 60.01\std{1.23} & 16.26\std{2.60}
& 52.45\std{4.96} & 64.07\std{1.91} & 4.34\std{3.94} \\

\midrule
\multicolumn{10}{l}{\textit{Backbone: MoCo v3-1K}} \\
L2P
& 28.17\std{7.08} & 39.07\std{11.31} & 19.68\std{15.19}
& 17.43\std{1.71} & 16.27\std{5.43} & 43.59\std{4.64}
& 12.42\std{2.31} & 20.00\std{7.36} & 26.38\std{11.63} \\

DualPrompt
& 53.33\std{4.65} & 58.20\std{7.73} & 14.79\std{4.09}
& 36.69\std{1.74} & 30.24\std{1.94} & 35.07\std{5.93}
& 19.88\std{3.35} & 21.93\std{4.30} & 28.42\std{9.47} \\

CODA-P
& 53.47\std{3.42} & 58.55\std{7.19} & 16.97\std{6.15}
& 39.89\std{2.71} & 31.72\std{4.86} & 43.27\std{2.27}
& 20.09\std{2.52} & 24.10\std{6.48} & 30.90\std{10.77} \\

MVP
& 54.33\std{4.56} & 40.84\std{14.21} & 55.42\std{15.00}
& 36.45\std{2.35} & 26.37\std{6.04} & 55.87\std{4.26}
& 28.48\std{3.34} & 23.56\std{9.78} & 62.97\std{5.67} \\

MISA
& 57.00\std{6.06} & 62.18\std{3.94} & 16.19\std{5.26}
& 38.85\std{4.27} & 33.47\std{0.95} & 34.42\std{8.60}
& 25.02\std{4.39} & 27.68\std{4.35} & 35.50\std{7.97} \\

\rowcolor{gray!10}
FlyGCL-Prompt
& 67.29\std{4.65} & 79.67\std{3.83} & \bfseries 1.42\std{1.92}
& 59.00\std{1.58} & 59.68\std{0.64} & \bfseries 15.29\std{2.46}
& 44.69\std{6.32} & 59.30\std{2.31} & \bfseries 4.27\std{4.98} \\

\rowcolor{gray!10}
FlyGCL-Adapter
& 73.39\std{3.02} & \bfseries 82.49\std{1.29} & 2.57\std{1.83}
& \bfseries 61.82\std{1.46} & 59.61\std{0.78} & 17.45\std{2.49}
& \bfseries 51.38\std{3.41} & \bfseries 62.53\std{1.36} & 7.12\std{4.20} \\

\rowcolor{gray!10}
FlyGCL-LoRA
& \bfseries 73.45\std{1.96} & 81.99\std{1.52} & 3.02\std{1.52}
& 61.40\std{1.49} & \bfseries 59.75\std{0.83} & 15.81\std{2.55}
& 50.26\std{5.11} & 61.53\std{3.06} & 7.28\std{2.21} \\

\bottomrule
\end{tabular}
}
\end{table}

\begin{table*}[t]
    \centering
    \caption{Ablation study of FlyGCL under the GCL setting using ViT-B/16 backbones pretrained on ImageNet-21K (Sup-21K) and fine-tuned on ImageNet-1K (Sup-21K/1K). We report average anytime accuracy $A_{\rm auc}$ ($\uparrow$), final average accuracy $A_{\rm last}$ ($\uparrow$), forgetting $F$ ($\downarrow$), and negative backward transfer (NBT, $\downarrow$).}
    \label{tab:img_ablation}
    \renewcommand\arraystretch{1.08}
    \setlength{\tabcolsep}{3.5pt}
    \resizebox{\textwidth}{!}{
    \begin{tabular}{
        l
        c c
        *{8}{S[table-format=-2.2(2)]}
    }
    \toprule
    \multirow{2.5}{*}{\textbf{Method}}
    & \multicolumn{2}{c}{\textbf{Component}}
    & \multicolumn{4}{c}{\textbf{CIFAR-100}}
    & \multicolumn{4}{c}{\textbf{ImageNet-R}} \\
    \cmidrule(lr){2-3}\cmidrule(lr){4-7}\cmidrule(lr){8-11}
    & \textbf{MoE} & \textbf{EL}
    & {$A_{\rm auc}$ ($\uparrow$)} & {$A_{\rm last}$ ($\uparrow$)} & {$F$ ($\downarrow$)} & {NBT ($\downarrow$)}
    & {$A_{\rm auc}$ ($\uparrow$)} & {$A_{\rm last}$ ($\uparrow$)} & {$F$ ($\downarrow$)} & {NBT ($\downarrow$)} \\
    \midrule
    \multicolumn{11}{l}{\textit{Backbone: Sup-21K}} \\
    \multirow{5}{*}{FlyGCL-Prompt}
    & $\times$ & $\times$
    & 68.31\std{4.29} & 65.95\std{2.85} & 40.26\std{3.40} & 27.15\std{2.07}
    & 43.12\std{1.74} & 43.74\std{5.28} & 50.73\std{3.45} & 49.28\std{3.64} \\
    & $\checkmark$ & $\times$
    & 81.84\std{1.95} & 83.83\std{1.62} & 9.51\std{1.75} & 0.16\std{2.69}
    & 56.91\std{1.36} & 53.61\std{1.16} & 18.15\std{1.71} & 16.53\std{1.44} \\
    & $\times$ & $\checkmark$
    & 81.91\std{1.98} & 83.72\std{1.21} & 10.40\std{1.61} & 0.43\std{2.55}
    & 56.68\std{1.25} & 53.67\std{1.22} & 18.53\std{2.27} & 16.62\std{2.13} \\
    \rowcolor{gray!10}
    & $\checkmark$ & $\checkmark$
    & \bfseries 83.77\std{2.21} & \bfseries 87.59\std{0.38} & \bfseries 4.31\std{0.82} & \bfseries -4.23\std{1.71}
    & \bfseries 60.22\std{1.18} & \bfseries 59.62\std{0.50} & \bfseries 10.74\std{1.69} & \bfseries 9.00\std{1.67} \\
    \addlinespace[2pt]
    
    \multirow{5}{*}{FlyGCL-Adapter}
    & $\times$ & $\times$
    & 66.78\std{4.75} & 64.98\std{4.72} & 39.27\std{3.16} & 26.16\std{5.99}
    & 46.36\std{1.79} & 46.44\std{2.97} & 54.02\std{6.96} & 52.33\std{6.85} \\
    & $\checkmark$ & $\times$
    & 81.60\std{1.61} & 83.09\std{2.14} & 10.38\std{2.29} & 1.51\std{3.74}
    & 61.04\std{1.36} & 56.23\std{0.87} & 21.29\std{3.25} & 19.79\std{3.12} \\
    & $\times$ & $\checkmark$
    & 81.19\std{1.94} & 82.19\std{1.87} & 12.56\std{2.24} & 2.53\std{3.30}
    & 60.22\std{1.45} & 61.20\std{0.58} & 15.98\std{2.43} & 14.19\std{2.17} \\
    \rowcolor{gray!10}
    & $\checkmark$ & $\checkmark$
    & \bfseries 83.65\std{2.31} & \bfseries 87.29\std{0.21} & \bfseries 4.83\std{0.59} & \bfseries -4.09\std{2.10}
    & \bfseries 63.91\std{1.22} & \bfseries 62.64\std{0.24} & \bfseries 12.13\std{1.57} & \bfseries 10.72\std{1.44} \\
    \addlinespace[2pt]
    
    \multirow{5}{*}{FlyGCL-LoRA}
    & $\times$ & $\times$
    & 70.77\std{3.86} & 68.98\std{3.30} & 33.24\std{3.85} & 21.43\std{4.65}
    & 48.30\std{2.19} & 48.87\std{5.69} & 54.24\std{4.25} & 52.97\std{4.42} \\
    & $\checkmark$ & $\times$
    & 81.27\std{2.30} & 83.17\std{1.05} & 10.08\std{1.17} & 0.53\std{1.98}
    & 62.48\std{1.41} & 57.78\std{0.93} & 20.04\std{2.99} & 18.60\std{2.57} \\
    & $\times$ & $\checkmark$
    & 82.04\std{2.17} & 83.14\std{1.72} & 10.80\std{1.78} & 1.13\std{2.97}
    & 62.01\std{1.30} & 57.67\std{0.60} & 20.76\std{3.17} & 18.67\std{2.89} \\
    \rowcolor{gray!10}
    & $\checkmark$ & $\checkmark$
    & \bfseries 83.87\std{2.22} & \bfseries 87.64\std{0.19} & \bfseries 4.46\std{0.62} & \bfseries -5.06\std{1.89}
    & \bfseries 65.32\std{1.17} & \bfseries 63.74\std{0.58} & \bfseries 11.79\std{2.20} & \bfseries 10.30\std{1.98} \\
    \midrule
    
    \multicolumn{11}{l}{\textit{Backbone: Sup-21K/1K}} \\
    \multirow{5}{*}{FlyGCL-Prompt}
    & $\times$ & $\times$
    & 57.64\std{3.84} & 56.17\std{9.31} & 57.11\std{8.94} & 42.63\std{7.39}
    & 42.51\std{1.79} & 43.84\std{3.94} & 59.07\std{6.71} & 57.55\std{6.54} \\
    & $\checkmark$ & $\times$
    & 71.08\std{2.93} & 71.47\std{6.50} & 18.68\std{5.11} & 4.74\std{8.76}
    & 55.47\std{1.05} & 48.13\std{1.89} & 31.95\std{3.66} & 30.98\std{3.53} \\
    & $\times$ & $\checkmark$
    & 72.77\std{3.14} & 69.35\std{6.50} & 23.01\std{4.58} & 11.43\std{6.48}
    & 56.47\std{1.22} & 47.48\std{2.08} & 37.84\std{8.69} & 36.48\std{8.51} \\
    \rowcolor{gray!10}
    & $\checkmark$ & $\checkmark$
    & \bfseries 80.72\std{1.61} & \bfseries 86.74\std{0.70} & \bfseries 3.09\std{0.86} & \bfseries -10.71\std{2.94}
    & \bfseries 67.13\std{1.21} & \bfseries 65.78\std{0.69} & \bfseries 12.10\std{1.92} & \bfseries 10.88\std{1.95} \\
    \addlinespace[2pt]
    
    \multirow{5}{*}{FlyGCL-Adapter}
    & $\times$ & $\times$
    & 60.42\std{2.78} & 57.58\std{8.68} & 59.11\std{10.34} & 44.65\std{7.77}
    & 49.06\std{2.22} & 48.99\std{5.26} & 59.55\std{8.94} & 58.13\std{8.57} \\
    & $\checkmark$ & $\times$
    & 72.72\std{3.30} & 72.83\std{7.56} & 18.24\std{5.41} & 6.26\std{7.93}
    & 62.76\std{1.76} & 54.83\std{2.36} & 32.58\std{5.81} & 31.89\std{5.81} \\
    & $\times$ & $\checkmark$
    & 74.60\std{1.93} & 71.77\std{5.14} & 21.95\std{4.90} & 11.04\std{5.45}
    & 62.13\std{1.49} & 64.01\std{2.05} & 24.25\std{4.34} & 23.18\std{4.25} \\
    \rowcolor{gray!10}
    & $\checkmark$ & $\checkmark$
    & \bfseries 82.79\std{1.71} & \bfseries 87.42\std{0.81} & \bfseries 3.61\std{1.09} & \bfseries -6.78\std{2.42}
    & \bfseries 73.09\std{1.44} & \bfseries 70.97\std{0.92} & \bfseries 13.22\std{2.06} & \bfseries 12.42\std{2.06} \\
    \addlinespace[2pt]
    
    \multirow{5}{*}{FlyGCL-LoRA}
    & $\times$ & $\times$
    & 57.01\std{3.58} & 54.88\std{10.98} & 56.11\std{11.59} & 39.63\std{10.37}
    & 47.06\std{2.13} & 47.82\std{3.48} & 54.80\std{7.17} & 53.29\std{7.26} \\
    & $\checkmark$ & $\times$
    & 67.20\std{4.77} & 65.91\std{9.04} & 19.01\std{5.99} & 7.29\std{6.13}
    & 60.20\std{1.03} & 52.25\std{2.96} & 32.76\std{5.47} & 31.67\std{5.00} \\
    & $\times$ & $\checkmark$
    & 69.74\std{4.18} & 74.10\std{7.38} & 17.34\std{6.33} & 4.25\std{5.36}
    & 59.32\std{1.49} & 62.32\std{1.18} & 22.85\std{3.69} & 21.74\std{3.65} \\
    \rowcolor{gray!10}
    & $\checkmark$ & $\checkmark$
    & \bfseries 80.60\std{1.45} & \bfseries 86.39\std{0.93} & \bfseries 3.17\std{1.12} & \bfseries -10.06\std{2.56}
    & \bfseries 71.19\std{1.24} & \bfseries 69.30\std{0.61} & \bfseries 12.32\std{2.46} & \bfseries 11.33\std{2.38} \\
    \bottomrule
    \end{tabular}
    }
    \end{table*}

\begin{table*}[htbp!]
\centering
\caption{Performance comparison on continual vision-language benchmarks under the GCL setting. We report final average accuracy $A_{\rm last}$ (\%, $\uparrow$), average anytime accuracy $A_{\rm auc}$ (\%, $\uparrow$), average forgetting $F$ (\%, $\downarrow$), and backward transfer $\mathrm{BWT}$ (\%, $\uparrow$).}
\label{tab:vl}
\setlength{\tabcolsep}{2.5pt}

\resizebox{\linewidth}{!}{
\begin{tabular}{
l
*{8}{S[table-format=3.2(3)]}
}
\toprule

\multirow{2.5}{*}{\textbf{Method}} 
& \multicolumn{4}{c}{\textbf{CIFAR-100}} 
& \multicolumn{4}{c}{\textbf{ImageNet-R}} \\

\cmidrule(lr){2-5}
\cmidrule(lr){6-9}

& {$A_{\rm last} (\uparrow)$} 
& {$A_{\rm auc} (\uparrow)$}
& {$F (\downarrow)$} 
& {$\mathrm{BWT} (\uparrow)$}

& {$A_{\rm last} (\uparrow)$} 
& {$A_{\rm auc} (\uparrow)$}
& {$F (\downarrow)$} 
& {$\mathrm{BWT} (\uparrow)$} \\

\midrule

SeqFT 
& 73.38\std{1.51}
& 79.96\std{0.68}
& 21.12\std{2.22}
& -19.85\std{2.30}
& 72.74\std{0.76}
& 80.45\std{0.58}
& 17.90\std{1.29}
& -15.89\std{1.43} \\

EWC~\cite{kirkpatrick2017overcoming}
& 71.82\std{2.01}
& 79.75\std{0.73}
& 22.98\std{2.60}
& -21.95\std{2.46}
& 69.17\std{0.44}
& 76.93\std{0.45}
& 19.94\std{1.05}
& -16.95\std{1.20} \\

LwF~\cite{li2017learning}
& 65.1\std{0.69}
& 74.63\std{0.90}
& 35.15\std{0.73}
& -35.02\std{0.80}
& 60.11\std{1.27}
& 73.96\std{0.65}
& 37.46\std{1.98}
& -36.26\std{1.86} \\

L2P~\cite{wang2022learning}
& 66.34\std{0.16}
& 73.55\std{1.03}
& 8.59\std{1.18}
& -8.29\std{1.20}
& 68.21\std{0.10}
& 75.96\std{0.65}
& 8.49\std{0.26}
& -8.30\std{0.21} \\

DualPrompt~\cite{wang2022dualprompt}
& 65.83\std{0.19}
& 71.19\std{1.01}
& 7.12\std{1.05}
& -6.02\std{1.19}
& 66.56\std{0.11}
& 73.37\std{0.77}
& 7.98\std{0.15}
& -7.36\std{0.19} \\

CODA-Prompt~\cite{smith2023coda}
& 66.79\std{0.08}
& 74.34\std{0.92}
& 8.72\std{1.19}
& -8.60\std{1.20}
& 71.68\std{0.00}
& 79.32\std{0.90}
& 8.00\std{0.36}
& -8.00\std{0.36} \\

% PROOF
% & 38.58\std{5.65}
% & 36.93\std{3.89}
% & 33.44\std{6.15}
% & -32.49\std{6.46}
% & 18.79\std{11.70}
% & 26.82\std{3.49}
% & 23.52\std{5.35}
% & -19.63\std{5.12} \\

CLAP4CLIP~\cite{jha2024clap4clip}
& 68.24\std{0.06}
& 75.92\std{1.16}
& 9.07\std{0.74}
& -8.79\std{0.78}
& 77.01\std{0.11}
& 83.11\std{0.42}
& 7.29\std{0.41}
& -6.98\std{0.45} \\

% MoE-Adapters
% & 63.38\std{2.44}
% & 74.63\std{0.67}
% & 36.02\std{2.22}
% & -35.75\std{2.08}
% & 67.90\std{1.37}
% & 78.72\std{0.64}
% & 31.08\std{0.98}
% & -29.40\std{1.38} \\

% RaPF
% & 70.38\std{1.76}
% & 75.31\std{0.29}
% & 15.06\std{1.60}
% & -4.62\std{1.86}
% & 65.32\std{1.62}
% & 69.50\std{0.68}
% & 18.42\std{1.15}
% & -8.80\std{1.25} \\

MG-CLIP~\cite{huang2025mind}
& 77.54\std{0.36}
& 81.35\std{0.46}
& 10.48\std{0.54}
& -6.64\std{0.46}
& 72.67\std{1.83}
& 78.37\std{0.68}
& 10.91\std{1.07}
& -5.66\std{1.19} \\

% DMNSP
% & 76.79\std{0.38}
% & 78.12\std{0.92}
% & 9.21\std{0.55}
% & -4.62\std{0.50}
% & 75.34\std{0.97}
% & 78.72\std{1.05}
% & 8.57\std{0.86}
% & -4.34\std{1.27} \\

% \rowcolor{gray!10}
% FlyPrompt
% & 74.28\std{0.38}
% & 81.15\std{0.91}
% & 11.48\std{1.21}
% & -10.25\std{1.36}
% & 71.09\std{0.76}
% & 77.18\std{0.83}
% & 8.85\std{1.11}
% & -2.38\std{0.97} \\

% \rowcolor{gray!10}
% FlyAdapter
% & 76.17\std{1.12}
% & 81.94\std{0.88}
% & 6.31\std{1.40}
% & -3.14\std{1.45}
% & 77.13\std{0.85}
% & 81.22\std{1.04}
% & 7.15\std{0.36}
% & -3.73\std{0.57} \\

\rowcolor{gray!10}
% FlyLoRA
FlyGCL (Ours)
& \bf 79.59\std{0.48}
& \bf 85.15\std{0.22}
& \bf 4.70\std{0.59}
& \bf -3.08\std{0.53}
& \bf 79.30\std{0.25}
& \bf 84.72\std{0.31}
& \bf 6.67\std{0.66}
& \bf -4.58\std{0.72} \\

\bottomrule
\end{tabular}
} 

\end{table*}

% --- Table: Skill Assessment ---

\begin{table*}[t]
\centering
\caption{Performance comparison on the continual skill assessment benchmark (EgoExoLearn). We report final average ranking accuracy $A_{\rm last}$ (\%, $\uparrow$), average forgetting $\bar{F}_{T}$ (\%, $\downarrow$), and average anytime ranking accuracy $A_{\rm auc}$ (\%, $\uparrow$).}
\label{tab:ce4l_skill}
\renewcommand\arraystretch{1.10}
\setlength{\tabcolsep}{3.0pt}
\resizebox{\textwidth}{!}{
\begin{tabular}{
    l
    S[table-format=2.2(2)]
    S[table-format=-2.2(2)]
    S[table-format=2.2(2)]
    S[table-format=2.2(2)]
    S[table-format=-2.2(2)]
    S[table-format=2.2(2)]
    S[table-format=2.2(2)]
    S[table-format=-2.2(2)]
    S[table-format=2.2(2)]
}
\toprule
\multirow{2.5}{*}{\textbf{Method}}
% & \multicolumn{6}{c}{\textbf{Ego-exo}}
% & \multicolumn{3}{c}{\textbf{Ego-only}} \\
% \cmidrule(lr){2-7}
% \cmidrule(lr){8-10}
& \multicolumn{3}{c}{\textbf{RAAN + RN (Ego-exo)}}
& \multicolumn{3}{c}{\textbf{RAAN + TL (Ego-exo)}}
& \multicolumn{3}{c}{\textbf{RAAN (Ego-only)}} \\
\cmidrule(lr){2-4}
\cmidrule(lr){5-7}
\cmidrule(lr){8-10}
& {$A_{\rm last}$ ($\uparrow$)}
& {$\bar{F}_{T}$ ($\downarrow$)}
& {$A_{\rm auc}$ ($\uparrow$)}
& {$A_{\rm last}$ ($\uparrow$)}
& {$\bar{F}_{T}$ ($\downarrow$)}
& {$A_{\rm auc}$ ($\uparrow$)}
& {$A_{\rm last}$ ($\uparrow$)}
& {$\bar{F}_{T}$ ($\downarrow$)}
& {$A_{\rm auc}$ ($\uparrow$)} \\
\midrule
% Joint
% & 80.17\std{0.06} & {--} & {--}
% & 80.41\std{0.08} & {--} & {--}
% & 80.42\std{0.14} & {--} & {--} \\

SeqFT
& 72.45\std{2.21} & 10.97\std{1.64} & 75.76\std{2.00}
& 71.63\std{2.13} & 13.20\std{1.94} & 75.87\std{2.03}
& 72.94\std{2.15} & 11.30\std{1.35} & 76.19\std{1.84} \\

ER~\cite{rolnick2019experience}
& 77.39\std{0.50} &  3.13\std{2.83} & 79.14\std{1.40}
& 78.27\std{0.73} &  2.39\std{0.95} & 79.80\std{1.61}
& 77.24\std{0.56} &  3.94\std{2.30} & 79.07\std{1.48} \\

DER++~\cite{buzzega2020dark}
& 78.60\std{0.30} &  1.17\std{1.70} & 79.50\std{1.08}
& 79.47\std{0.07} &  1.09\std{0.81} & 80.14\std{1.16}
& 79.46\std{0.26} &  0.98\std{1.24} & 79.98\std{0.95} \\

EWC~\cite{kirkpatrick2017overcoming}
& 72.59\std{1.95} & 11.32\std{2.26} & 75.72\std{1.85}
& 71.05\std{2.00} & 14.22\std{2.44} & 75.75\std{1.96}
& 73.34\std{2.28} & 10.62\std{1.75} & 76.29\std{1.77} \\

LwF~\cite{li2017learning}
& 68.87\std{2.58} & 12.52\std{3.20} & 74.42\std{1.65}
& 71.83\std{2.11} &  0.49\std{2.24} & 74.63\std{3.04}
& 72.61\std{1.73} &  2.21\std{3.63} & 75.13\std{2.98} \\

L2P+~\cite{wang2022learning}
& 71.43\std{2.56} & 12.52\std{1.20} & 74.98\std{1.65}
& 71.38\std{0.97} & 14.22\std{3.81} & 75.18\std{0.85}
& 72.03\std{1.48} & 11.69\std{3.64} & 75.58\std{0.93} \\

DualPrompt+~\cite{wang2022dualprompt}
& 70.04\std{5.90} &  9.26\std{3.04} & 74.91\std{2.37}
& 75.21\std{0.70} &  7.45\std{2.50} & 77.04\std{1.10}
& 75.90\std{0.44} &  6.02\std{2.02} & 77.42\std{0.91} \\

S-Prompt+~\cite{wang2022s}
& 77.75\std{0.41} &  \bf 0.00\std{0.21} & 79.61\std{2.00}
& 78.05\std{0.78} &  \bf 0.13\std{0.17} & 80.17\std{2.10}
& 79.17\std{0.30} &  \bf 0.05\std{0.03} & 80.49\std{1.76} \\

% VISTA
% & 79.58\std{0.81} &  0.19\std{0.19} & {--}
% & 80.41\std{0.48} &  0.12\std{0.14} & {--}
% & 80.38\std{0.32} & -0.19\std{0.51} & {--} \\

% \midrule
\rowcolor{gray!10}
FlyGCL (Ours)
& \bf 82.66\std{0.22} &  0.31\std{0.22} & \bf 81.38\std{0.17}
& \bf 82.59\std{0.23} &  0.31\std{0.25} & \bf 81.32\std{0.15}
& \bf 82.27\std{0.47} &  0.38\std{0.28} & \bf 81.30\std{0.11} \\
%\rowcolor{gray!10}
%FlyGCL (Ours)
%& 80.02\std{0.23} & \bf -0.31\std{0.69} & 79.64\std{2.09}
%& 80.07\std{0.82} & \bf -0.07\std{0.59} & 80.12\std{2.08}
%& 80.34\std{0.42} & 0.13\std{0.14} & 80.54\std{1.79} \\
\bottomrule
\end{tabular}
}
\end{table*}

% \input{tex/tab-ego-learn-skill-updated}

% --- Table: Action Segmentation ---
% \input{tex/tab-ego-learn-as}

% --- Table: Cross-View Association ---
% \input{tex/tab-ego-learn-cva}

% --- Table: Action Anticipation ---
\begin{table*}[t]
\centering
\caption{Performance comparison on the continual action anticipation benchmark (EgoExoLearn). We report final average Top-5 recall $A_{\rm last}$ (\%, $\uparrow$) and average anytime Top-5 recall $A_{\rm auc}$ (\%, $\uparrow$) for verb (-V) and noun (-N) prediction. Avg. denotes the average over Ego-V, Ego-N, Exo-V, and Exo-N.}
\label{tab:ce4l_anticipation}
\renewcommand\arraystretch{1.10}
\setlength{\tabcolsep}{2.5pt}
\resizebox{0.90\textwidth}{!}{
\begin{tabular}{
    l
    *{5}{S[table-format=2.2(3)]}
    *{5}{S[table-format=2.2(3)]}
}
\toprule
\multirow{2.5}{*}{\textbf{Method}}
& \multicolumn{5}{c}{\textbf{$A_{\rm last}$} ($\uparrow$)}
& \multicolumn{5}{c}{\textbf{$A_{\rm auc}$} ($\uparrow$)} \\
\cmidrule(lr){2-6}
\cmidrule(lr){7-11}
& {\textbf{Ego-V}} & {\textbf{Ego-N}} & {\textbf{Exo-V}} & {\textbf{Exo-N}} & {\textbf{Avg.}}
& {\textbf{Ego-V}} & {\textbf{Ego-N}} & {\textbf{Exo-V}} & {\textbf{Exo-N}} & {\textbf{Avg.}} \\
\midrule

\multicolumn{11}{c}{\textit{Ego-exo}} \\
\midrule
% Joint
% & 35.50\std{0.41} & 27.80\std{0.17} & 33.12\std{0.29} & 27.11\std{0.12} & 30.88\std{0.25}
% & {--} & {--} & {--} & {--} & {--} \\

SeqFT
& 31.25\std{1.15} & 19.26\std{0.67} & 30.30\std{0.91} & 18.82\std{0.54} & 24.91\std{0.82}
& 31.60\std{0.33} & 20.60\std{0.90} & 30.32\std{0.63} & 20.07\std{0.95} & 25.65\std{0.70} \\

ER~\cite{rolnick2019experience}
& 34.49\std{0.53} & 25.29\std{0.24} & 32.70\std{0.15} & 24.75\std{0.23} & 29.31\std{0.29}
& 33.44\std{0.28} & 24.44\std{1.00} & 32.03\std{0.31} & 23.80\std{1.16} & 28.43\std{0.69} \\

DER++~\cite{buzzega2020dark}
& 33.70\std{0.37} & 24.55\std{0.53} & 32.37\std{0.26} & 24.13\std{0.35} & 28.69\std{0.38}
& 32.92\std{0.33} & 23.57\std{1.55} & 31.73\std{0.38} & 23.18\std{1.54} & 27.85\std{0.95} \\

EWC~\cite{kirkpatrick2017overcoming}
& 31.29\std{1.04} & 19.96\std{0.48} & 30.64\std{0.95} & 19.44\std{0.25} & 25.33\std{0.68}
& 31.54\std{0.20} & 20.67\std{1.03} & 30.04\std{0.44} & 20.11\std{1.17} & 25.59\std{0.71} \\

LwF~\cite{li2017learning}
& 30.57\std{0.21} & 21.38\std{1.46} & 31.19\std{1.32} & 20.58\std{1.16} & 25.93\std{1.04}
& 30.58\std{0.38} & 21.14\std{0.81} & 30.08\std{0.21} & 20.50\std{0.94} & 25.57\std{0.58} \\

L2P+~\cite{wang2022learning}
& 28.51\std{1.89} & 16.19\std{0.80} & 26.12\std{0.93} & 15.21\std{0.62} & 21.51\std{1.06}
& 28.85\std{0.62} & 17.53\std{0.81} & 26.35\std{1.12} & 16.74\std{1.03} & 22.37\std{0.90} \\

DualPrompt+~\cite{wang2022dualprompt}
& 29.50\std{2.37} & 16.02\std{2.03} & 28.31\std{1.69} & 16.02\std{1.81} & 22.47\std{1.98}
& 32.66\std{1.95} & 18.19\std{3.01} & 30.57\std{3.17} & 18.25\std{2.90} & 24.92\std{2.76} \\

S-Prompt+~\cite{wang2022s}
& 31.83\std{2.24} & 16.49\std{2.43} & 27.64\std{1.46} & 16.36\std{2.02} & 23.08\std{2.04}
& \bf 35.09\std{4.24} & 18.97\std{3.54} & 31.56\std{4.18} & 19.08\std{3.43} & 26.17\std{3.85} \\

% VISTA
% & 35.23\std{1.05} & 25.23\std{1.33} & 32.77\std{3.37} & 26.01\std{1.48} & 29.81\std{1.81}
% & {--} & {--} & {--} & {--} & {--} \\

\rowcolor{gray!10}
FlyGCL (Ours)
& \bf 35.21\std{0.52} & \bf 34.18\std{0.60} & \bf 39.18\std{2.26} & \bf 41.65\std{1.28} & \bf 37.55\std{1.15}
& 32.90\std{0.32} & \bf 32.46\std{0.28} & \bf 38.75\std{0.77} & \bf 38.85\std{0.47} & \bf 35.74\std{0.37} \\
%\rowcolor{gray!10}
%FlyGCL (Ours)
%& 34.62\std{3.79} & 17.95\std{4.15} & 34.01\std{4.28} & 18.11\std{3.98} & 26.15\std{4.15}
%& \bf 37.50\std{5.06} & 20.57\std{5.23} & 36.71\std{4.97} & 20.75\std{5.09} & 28.79\std{5.32} \\
\midrule

\multicolumn{11}{c}{\textit{Ego-only}} \\
\midrule
% Joint
% & 33.40\std{0.03} & 29.88\std{0.16} & 32.68\std{0.03} & 29.32\std{0.11} & 31.32\std{0.08}
% & {--} & {--} & {--} & {--} & {--} \\

SeqFT
& 29.83\std{1.05} & 20.77\std{0.99} & 30.47\std{0.85} & 20.34\std{0.81} & 25.35\std{0.93}
& 29.36\std{0.31} & 22.03\std{1.48} & 28.81\std{0.25} & 21.59\std{1.57} & 25.45\std{0.90} \\

ER~\cite{rolnick2019experience}
& 31.15\std{0.49} & 26.47\std{0.47} & \bf 30.84\std{0.23} & \bf 26.50\std{0.21} & 28.74\std{0.35}
& 30.41\std{0.40} & 26.04\std{1.37} & \bf 30.25\std{0.32} & \bf 26.03\std{1.55} & \bf 28.18\std{0.91} \\

DER++~\cite{buzzega2020dark}
& 31.31\std{0.36} & 26.91\std{0.69} & 30.68\std{0.22} & 26.16\std{0.56} & 28.77\std{0.46}
& 30.29\std{0.57} & 26.15\std{2.08} & 29.40\std{0.65} & 25.54\std{1.86} & 27.85\std{1.29} \\

EWC~\cite{kirkpatrick2017overcoming}
& 29.72\std{1.12} & 20.63\std{1.00} & 30.16\std{0.37} & 20.19\std{0.84} & 25.18\std{0.83}
& 29.31\std{0.18} & 22.02\std{1.22} & 28.69\std{0.33} & 21.56\std{1.29} & 25.39\std{0.76} \\

LwF~\cite{li2017learning}
& 29.84\std{0.77} & 22.24\std{0.77} & 29.75\std{0.75} & 21.79\std{0.74} & 25.91\std{0.76}
& 29.19\std{0.41} & 22.52\std{0.92} & 28.76\std{0.58} & 22.09\std{1.03} & 25.64\std{0.73} \\

L2P+~\cite{wang2022learning}
& 28.76\std{1.14} & 16.24\std{1.18} & 27.88\std{1.20} & 15.50\std{1.34} & 22.10\std{1.21}
& 28.77\std{0.40} & 17.29\std{1.05} & 27.24\std{0.53} & 16.86\std{0.84} & 22.54\std{0.70} \\

DualPrompt+~\cite{wang2022dualprompt}
& 27.86\std{1.30} & 17.54\std{2.14} & 27.36\std{0.85} & 17.42\std{1.92} & 22.54\std{1.55}
& 28.77\std{1.09} & 19.73\std{3.41} & 27.93\std{1.98} & 19.92\std{3.56} & 24.09\std{2.51} \\

S-Prompt+~\cite{wang2022s}
& 30.41\std{1.11} & 18.36\std{1.72} & 27.99\std{1.62} & 18.23\std{1.35} & 23.75\std{1.45}
& 29.36\std{1.48} & 21.54\std{3.61} & 28.58\std{2.89} & 21.73\std{3.56} & 25.30\std{2.89} \\

\rowcolor{gray!10}
FlyGCL (Ours)
& \bf 34.97\std{0.19} & \bf 33.75\std{0.45} & 28.53\std{0.92} & 18.27\std{0.21} & \bf 28.88\std{0.33}
& \bf 32.66\std{0.23} & \bf 31.75\std{0.34} & 27.82\std{0.22} & 15.85\std{0.55} & 27.02\std{0.07} \\
\bottomrule
\end{tabular}
}
\end{table*}

% \input{tex/tab-ego-learn-aa-updated}

% --- Table: Action Planning ---
% \input{tex/tab-ego-learn-ap}

% --- Table: egoexo-fitness (skill assessment) ---

\begin{table*}[t]
\centering
\caption{Performance comparison on the continual skill assessment benchmark (EgoExo-Fitness). We report final average ranking accuracy $\bar{A}_{\mathrm{last}}$ (\%, $\uparrow$), average forgetting $\bar{F}_{T}$ (\%, $\downarrow$), and average anytime ranking accuracy $A_{\rm auc}$ (\%, $\uparrow$).}
\label{tab:fitness_skill}
\renewcommand\arraystretch{1.10}
\setlength{\tabcolsep}{3.2pt}
\resizebox{\textwidth}{!}{
\begin{tabular}{
    l
    S[table-format=3.2(2)]
    S[table-format=-2.2(2)]
    S[table-format=3.2(2)]
    S[table-format=3.2(2)]
    S[table-format=-2.2(2)]
    S[table-format=3.2(2)]
    S[table-format=3.2(2)]
    S[table-format=-2.2(2)]
    S[table-format=3.2(2)]
}
\toprule
\multirow{2.5}{*}{\textbf{Method}}
% & \multicolumn{3}{c}{\textbf{Ego-only}}
% & \multicolumn{6}{c}{\textbf{Ego-exo}} \\
% \cmidrule(lr){2-4}
% \cmidrule(lr){5-10}
& \multicolumn{3}{c}{\textbf{RAAN (Ego-only)}}
& \multicolumn{3}{c}{\textbf{RAAN+RN (Ego-exo)}}
& \multicolumn{3}{c}{\textbf{RAAN+TL (Ego-exo)}} \\
\cmidrule(lr){2-4}
\cmidrule(lr){5-7}
\cmidrule(lr){8-10}
& {$\bar{A}_{\mathrm{last}}$ ($\uparrow$)}
& {$\bar{F}_{T}$ ($\downarrow$)}
& {$A_{\rm auc}$ ($\uparrow$)}
& {$\bar{A}_{\mathrm{last}}$ ($\uparrow$)}
& {$\bar{F}_{T}$ ($\downarrow$)}
& {$A_{\rm auc}$ ($\uparrow$)}
& {$\bar{A}_{\mathrm{last}}$ ($\uparrow$)}
& {$\bar{F}_{T}$ ($\downarrow$)}
& {$A_{\rm auc}$ ($\uparrow$)} \\
\midrule
% Joint$_{2000}$
% & 80.42\std{0.14} & {--} & {--}
% & 80.17\std{0.06} & {--} & {--}
% & 80.41\std{0.08} & {--} & {--} \\

% Joint$_{1}$
% & 54.51\std{0.93} & {--} & {--}
% & 54.34\std{1.18} & {--} & {--}
% & 54.34\std{1.18} & {--} & {--} \\
% \midrule
SeqFT
& 54.26\std{1.87} & 1.55\std{1.02} & 54.62\std{3.36}
& 54.21\std{1.98} & 1.60\std{1.07} & 54.52\std{3.58}
& 54.22\std{1.89} & 1.63\std{0.99} & 54.58\std{3.39} \\

ER~\cite{rolnick2019experience}
& 54.62\std{1.86} & 1.02\std{0.57} & 54.68\std{3.48}
& 54.36\std{1.99} & 1.39\std{0.97} & 54.64\std{3.57}
& 54.60\std{1.92} & 0.78\std{0.40} & 54.64\std{3.50} \\

DER++~\cite{buzzega2020dark}
& 54.68\std{1.74} & \bfseries 0.45\std{0.06} & 54.78\std{3.41}
& 54.73\std{2.11} & \bfseries 0.65\std{0.56} & 54.71\std{3.57}
& 54.62\std{1.71} & \bfseries 0.27\std{0.14} & 54.73\std{3.42} \\

EWC~\cite{kirkpatrick2017overcoming}
& 54.22\std{1.86} & 1.60\std{1.02} & 54.62\std{3.34}
& 54.17\std{2.00} & 1.69\std{1.19} & 54.62\std{3.34}
& 54.22\std{1.87} & 1.60\std{1.00} & 54.58\std{3.39} \\

LwF~\cite{li2017learning}
& 54.28\std{1.87} & 1.36\std{0.89} & 54.67\std{3.30}
& 54.31\std{2.00} & 1.34\std{0.90} & 54.52\std{3.54}
& 54.42\std{1.90} & 1.14\std{0.66} & 54.67\std{3.34} \\

L2P+~\cite{wang2022learning}
& 54.33\std{1.96} & 1.53\std{0.92} & 54.62\std{3.38}
& 54.12\std{2.08} & 1.67\std{1.09} & 54.60\std{3.44}
& 54.23\std{1.94} & 1.62\std{1.01} & 54.62\std{3.41} \\

DualPrompt+~\cite{wang2022dualprompt}
& 50.94\std{1.34} & 3.05\std{3.20} & 52.30\std{2.22}
& 51.64\std{0.34} & 1.97\std{2.10} & 52.36\std{2.52}
& 50.95\std{1.27} & 3.21\std{3.06} & 52.38\std{2.27} \\

S-Prompt+~\cite{wang2022s}
& 57.85\std{3.97} & 1.07\std{0.20} & 57.22\std{3.64}
& 58.04\std{2.63} & 1.19\std{0.85} & 57.98\std{3.13}
& 57.75\std{3.35} & 1.79\std{0.33} & 57.94\std{2.70} \\

% VISTA
% & 58.40\std{4.49} & 0.81\std{1.58} & 59.38\std{4.15}
% & 57.65\std{4.59} & 1.58\std{1.54} & 54.09\std{2.06}
% & 58.53\std{4.71} & 0.69\std{1.63} & 59.41\std{4.19} \\

% FlyPrompt
% & 57.14\std{1.90} & 0.47\std{0.96} & 57.53\std{3.36}
% & 56.20\std{1.76} & 1.18\std{0.77} & 57.46\std{3.36}
% & 56.73\std{1.64} & 1.02\std{0.32} & 57.46\std{3.36} \\
% \midrule
\rowcolor{gray!10}
FlyGCL (Ours)
& \bfseries 62.47\std{1.04} & 0.89\std{1.15} & \bfseries 61.60\std{3.14}
& \bfseries 62.48\std{2.06} & 1.94\std{1.42} & \bfseries 62.29\std{3.24}
& \bfseries 62.99\std{1.25} & 1.09\std{1.04} & \bfseries 61.99\std{2.97} \\
% VISTA+
% & 63.06\std{2.24} & 0.71\std{1.49} & 62.14\std{3.02}
% & 61.45\std{2.59} & 1.87\std{1.65} & 62.13\std{2.93}
% & 62.97\std{2.19} & 0.86\std{1.39} & 62.13\std{2.93} \\
\bottomrule
\end{tabular}
}
\end{table*}

% --- Table: egoexo-fitness (ac) ---

\begin{table*}[t]
\centering
\caption{Performance comparison on the continual action classification benchmark (EgoExo-Fitness). We report final average accuracy $A_{\rm last}$ (\%, $\uparrow$), average forgetting $\bar{F}_T$ (\%, $\downarrow$), and average anytime accuracy $A_{\rm auc}$ (\%, $\uparrow$).}
\label{tab:fitness_action_classification}
\renewcommand\arraystretch{1.10}
\setlength{\tabcolsep}{4.0pt}
\resizebox{0.82\textwidth}{!}{
\begin{tabular}{
    l
    S[table-format=2.2(3)]
    S[table-format=-2.2(3)]
    S[table-format=2.2(3)]
    S[table-format=2.2(3)]
    S[table-format=-2.2(3)]
    S[table-format=2.2(3)]
}
\toprule
\multirow{2.5}{*}{\textbf{Method}}
& \multicolumn{3}{c}{\textbf{Ego-exo}}
& \multicolumn{3}{c}{\textbf{Ego-only}} \\
\cmidrule(lr){2-4}
\cmidrule(lr){5-7}
& {$A_{\rm last}$ ($\uparrow$)}
& {$\bar{F}_T$ ($\downarrow$)}
& {$A_{\rm auc}$ ($\uparrow$)}
& {$A_{\rm last}$ ($\uparrow$)}
& {$\bar{F}_T$ ($\downarrow$)}
& {$A_{\rm auc}$ ($\uparrow$)} \\
\midrule
% Joint$_{2000}$
% & {--} & {--} & {--}
% & {--} & {--} & {--} \\
%
% Joint$_{1}$
% & 32.66\stdp{3.22} & {--} & {--}
% & 27.81\stdp{2.78} & {--} & {--} \\
% \midrule
SeqFT
& 24.97\std{10.60} & -1.82\stdp{1.56} & 28.86\std{15.63}
& 25.00\stdp{8.53} & -7.75\stdp{1.43} & 28.70\std{12.93} \\

ER~\cite{rolnick2019experience}
& 27.31\stdp{5.69} & \bfseries -4.14\stdp{4.72} & 31.35\std{10.32}
& 28.40\stdp{3.97} & \bfseries -11.78\stdp{3.39} & 30.45\stdp{9.47} \\

DER++~\cite{buzzega2020dark}
& 22.29\stdp{3.30} & -1.67\stdp{1.33} & 29.00\std{10.32}
& 21.13\stdp{4.10} & -6.29\stdp{2.34} & 27.67\stdp{9.54} \\

EWC~\cite{kirkpatrick2017overcoming}
& 24.93\std{10.58} & -1.91\stdp{1.63} & 28.83\std{15.62}
& 25.08\stdp{8.43} & -8.07\stdp{1.38} & 28.66\std{12.87} \\

LwF~\cite{li2017learning}
& 24.98\std{10.25} & -3.35\stdp{2.60} & 28.67\std{15.34}
& 25.43\stdp{9.30} & -9.36\stdp{4.65} & 28.64\std{12.60} \\

L2P+~\cite{wang2022learning}
& 14.53\stdp{1.85} & 0.99\stdp{1.41} & 28.86\stdp{0.61}
& 13.44\stdp{2.17} & 0.14\stdp{0.58} & 25.88\stdp{6.16} \\

DualPrompt+~\cite{wang2022dualprompt}
& 18.54\stdp{6.55} & 20.02\stdp{7.52} & 28.85\std{12.07}
& 17.10\stdp{6.72} & 19.76\stdp{6.97} & 28.31\stdp{8.97} \\

S-Prompt+~\cite{wang2022s}
& 12.79\stdp{4.11} & 14.43\std{12.66} & 23.72\std{12.77}
& 11.02\stdp{5.14} & 6.40\stdp{8.90} & 23.21\stdp{7.75} \\

% VISTA
% & 25.13\stdp{4.09} & 25.22\std{10.59} & 34.75\std{12.69}
% & 28.00\stdp{7.21} & 15.78\stdp{4.70} & 33.83\std{10.79} \\
%
% FlyPrompt
% & 30.42\stdp{9.73} & 18.79\stdp{9.13} & 37.56\std{16.28}
% & 33.03\stdp{8.19} & 12.49\stdp{3.41} & 36.54\stdp{9.85} \\
% \midrule
\rowcolor{gray!10}
FlyGCL (Ours)
& \bfseries 37.86\std{12.85} & 7.02\stdp{6.41} & \bfseries 42.34\std{16.42}
& \bfseries 39.61\std{13.23} & 2.08\stdp{5.41} & \bfseries 41.05\std{14.34} \\
\bottomrule
\end{tabular}
}
\end{table*}

% --- Table: egoexo-fitness (cvsv) ---

\begin{table*}[t]
\centering
\caption{Performance comparison on the continual sequence verification benchmark (EgoExo-Fitness). We report final average performance $\bar{A}_{\mathrm{last}}$ (\%, $\uparrow$), average forgetting $\bar{F}_{T}$ (\%, $\downarrow$), and average anytime performance $A_{\rm auc}$ (\%, $\uparrow$) in terms of ROC-AUC and mAP.}
\label{tab:fitness_verification}
\renewcommand\arraystretch{1.10}
\setlength{\tabcolsep}{2.5pt}
\resizebox{0.725\linewidth}{!}{
\begin{tabular}{
    l
    S[table-format=3.2(2)]
    S[table-format=2.2(2)]
    S[table-format=3.2(2)]
    S[table-format=3.2(2)]
    S[table-format=2.2(2)]
    S[table-format=3.2(2)]
}
\toprule
\multirow{3}{*}{\textbf{Method}}
& \multicolumn{3}{c}{\textbf{ROC-AUC}}
& \multicolumn{3}{c}{\textbf{mAP}} \\
\cmidrule(lr){2-4}
\cmidrule(lr){5-7}
& {$\bar{A}_{\mathrm{last}}$ ($\uparrow$)}
& {$\bar{F}_{T}$ ($\downarrow$)}
& {$A_{\rm auc}$ ($\uparrow$)}
& {$\bar{A}_{\mathrm{last}}$ ($\uparrow$)}
& {$\bar{F}_{T}$ ($\downarrow$)}
& {$A_{\rm auc}$ ($\uparrow$)} \\
\midrule
\multicolumn{7}{c}{\textit{Ego-exo}} \\
\midrule
% Joint$_{2000}$
% & 98.05\std{0.21} & {--} & {--}
% & 94.07\std{0.23} & {--} & {--} \\
%
% Joint$_{1}$
% & 96.21\std{0.58} & {--} & {--}
% & 72.29\std{0.13} & {--} & {--} \\
% \midrule
SeqFT
& 83.42\std{2.64} & 5.75\std{3.88} & 87.72\std{0.53}
& 53.43\std{5.11} & 22.58\std{5.26} & 70.37\std{4.29} \\

ER~\cite{rolnick2019experience}
& 86.02\std{0.52} & 4.49\std{1.38} & 89.40\std{0.52}
& 57.26\std{4.45} & 20.02\std{4.96} & 72.26\std{4.21} \\

DER++~\cite{buzzega2020dark}
& 89.58\std{1.21} & 4.28\std{1.86} & \bfseries 92.83\std{1.01}
& 64.55\std{3.20} & 19.68\std{5.75} & 79.29\std{2.94} \\

EWC~\cite{kirkpatrick2017overcoming}
& 83.65\std{2.45} & 5.55\std{4.07} & 87.83\std{0.61}
& 53.34\std{5.05} & 22.73\std{5.45} & 70.38\std{4.25} \\

LwF~\cite{li2017learning}
& 84.22\std{1.72} & 5.61\std{3.07} & 88.44\std{0.66}
& 55.90\std{5.78} & 21.13\std{6.00} & 71.76\std{4.35} \\

L2P+~\cite{wang2022learning}
& 83.95\std{2.39} & 5.37\std{3.71} & 88.01\std{0.30}
& 53.98\std{5.38} & 22.38\std{5.03} & 70.75\std{4.67} \\

DualPrompt+~\cite{wang2022dualprompt}
& 84.85\std{2.12} & 6.89\std{3.64} & 89.92\std{0.93}
& 55.33\std{3.31} & 24.49\std{6.99} & 73.69\std{3.43} \\

S-Prompt+~\cite{wang2022s}
& 86.67\std{1.31} & 6.39\std{1.18} & 91.43\std{1.27}
& 62.05\std{6.47} & 22.11\std{3.62} & 78.52\std{6.40} \\

% VISTA
% & 88.68\std{4.17} & 6.31\std{3.39} & 93.45\std{2.12}
% & 64.86\std{8.28} & 21.39\std{7.07} & 81.02\std{4.19} \\
%
% FlyPrompt
% & 88.92\std{3.74} & 4.90\std{2.83} & 92.63\std{1.73}
% & 64.89\std{8.04} & 18.71\std{4.48} & 78.85\std{7.06} \\
% \midrule
\rowcolor{gray!10}
FlyGCL (Ours)
& \bfseries 89.64\std{2.09} & \bfseries 4.22\std{2.51} & 92.62\std{1.91}
& \bfseries 76.77\std{8.82} & \bfseries 11.23\std{8.12} & \bfseries 85.25\std{4.56} \\
% VISTA+
% & 91.74\std{3.11} & 4.30\std{2.07} & 94.97\std{1.60}
% & 71.66\std{5.49} & 18.67\std{5.72} & 85.75\std{1.26} \\
\midrule
\multicolumn{7}{c}{\textit{Ego-only}} \\
\midrule
% Joint$_{2000}$
% & 91.65\std{0.77} & {--} & {--}
% & 63.72\std{0.38} & {--} & {--} \\
%
% Joint$_{1}$
% & 90.76\std{0.43} & {--} & {--}
% & 52.36\std{0.32} & {--} & {--} \\
% \midrule
SeqFT
& 83.64\std{3.57} & 5.55\std{3.02} & 86.61\std{1.66}
& 51.37\std{3.65} & 24.25\std{5.26} & 68.86\std{3.15} \\

ER~\cite{rolnick2019experience}
& 83.71\std{3.30} & 6.29\std{2.67} & 88.09\std{1.38}
& 53.92\std{4.55} & 23.00\std{6.10} & 70.18\std{2.79} \\

DER++~\cite{buzzega2020dark}
& 86.81\std{2.55} & 6.19\std{2.93} & 89.94\std{0.63}
& 57.08\std{4.12} & 23.46\std{6.24} & 74.19\std{1.62} \\

EWC~\cite{kirkpatrick2017overcoming}
& 82.59\std{2.24} & 5.56\std{2.48} & 86.61\std{1.66}
& 50.70\std{3.61} & 24.21\std{5.41} & 68.86\std{3.15} \\

LwF~\cite{li2017learning}
& 84.52\std{3.87} & 5.37\std{4.24} & 87.73\std{0.72}
& 53.02\std{4.86} & 23.02\std{5.64} & 69.48\std{3.51} \\

L2P+~\cite{wang2022learning}
& 83.32\std{3.55} & 5.96\std{2.81} & 87.17\std{1.32}
& 51.93\std{4.15} & 23.78\std{5.41} & 68.73\std{3.25} \\

DualPrompt+~\cite{wang2022dualprompt}
& 85.24\std{2.92} & 4.50\std{3.21} & 88.62\std{1.21}
& 51.93\std{4.07} & 25.23\std{4.69} & 70.85\std{2.81} \\

S-Prompt+~\cite{wang2022s}
& 84.37\std{8.05} & 7.53\std{6.94} & 85.57\std{0.39}
& 53.15\std{6.46} & 26.03\std{6.97} & 67.19\std{2.50} \\

% VISTA
% & 85.45\std{1.56} & 7.60\std{0.53} & 91.16\std{1.16}
% & 56.24\std{4.02} & 26.61\std{5.89} & 76.20\std{2.17} \\
%
% FlyPrompt
% & 86.81\std{0.76} & 4.69\std{1.63} & 90.20\std{0.35}
% & 58.00\std{1.08} & 23.23\std{3.68} & 75.43\std{3.74} \\
% \midrule
\rowcolor{gray!10}
FlyGCL (Ours)
& \bfseries 92.52\std{3.89} & \bfseries 3.84\std{2.43} & \bfseries 95.40\std{2.12}
& \bfseries 86.29\std{8.62} & \bfseries 7.72\std{5.82} & \bfseries 92.08\std{4.48} \\
% VISTA+
% & 89.90\std{4.35} & 5.36\std{3.87} & {--}
% & 68.52\std{10.30} & 19.99\std{9.56} & {--} \\
\bottomrule
\end{tabular}
}
\end{table*}

% --- Table: egoexo-fitness (gev) ---

\begin{table*}[t]
\centering
\caption{Performance comparison on the continual guidance-based execution verification benchmark (EgoExo-Fitness). We report final average performance $A_{\mathrm{last}}$ (\%, $\uparrow$), average forgetting $\bar{F}_T$ (\%, $\downarrow$), and average anytime performance $A_{\rm auc}$ (\%, $\uparrow$) in terms of classification accuracy and F1 score.}
\label{tab:fitness_guidance_verification}
\renewcommand\arraystretch{1.10}
\setlength{\tabcolsep}{2.5pt}
\resizebox{0.725\linewidth}{!}{
\begin{tabular}{
    l
    S[table-format=3.2(3)]
    S[table-format=-2.2(3)]
    S[table-format=3.2(3)]
    S[table-format=3.2(3)]
    S[table-format=-2.2(3)]
    S[table-format=3.2(3)]
}
\toprule
\multirow{3}{*}{\textbf{Method}}
& \multicolumn{3}{c}{\textbf{Cls Acc}}
& \multicolumn{3}{c}{\textbf{F1}} \\
\cmidrule(lr){2-4}
\cmidrule(lr){5-7}
& {$A_{\mathrm{last}}$ ($\uparrow$)}
& {$\bar{F}_T$ ($\downarrow$)}
& {$A_{\rm auc}$ ($\uparrow$)}
& {$A_{\mathrm{last}}$ ($\uparrow$)}
& {$\bar{F}_T$ ($\downarrow$)}
& {$A_{\rm auc}$ ($\uparrow$)} \\
\midrule
\multicolumn{7}{c}{\textit{Ego-exo}} \\
\midrule
% Joint$_{2000}$
% & {--} & {--} & {--}
% & {--} & {--} & {--} \\
%
% Joint$_{1}$
% & 71.01\stdp{3.50} & {--} & {--}
% & 77.80\stdp{3.33} & {--} & {--} \\
% \midrule
SeqFT
& 58.12\stdp{6.20} & -1.65\stdp{2.86} & 53.45\stdp{5.69}
& 66.96\stdp{8.18} & -2.62\stdp{4.11} & 59.80\stdp{7.88} \\

ER~\cite{rolnick2019experience}
& 59.43\stdp{8.86} & -2.93\stdp{1.25} & 53.99\stdp{6.13}
& 67.73\std{10.81} & -3.18\stdp{2.56} & 60.45\stdp{8.31} \\

DER++~\cite{buzzega2020dark}
& 58.75\stdp{6.14} & \bfseries -3.23\stdp{9.79} & 53.87\stdp{1.71}
& 65.05\std{8.81} & \bfseries -4.15\stdp{13.17} & 59.04\stdp{2.68} \\

EWC~\cite{kirkpatrick2017overcoming}
& 57.99\stdp{9.36} & -1.65\stdp{2.32} & 53.55\stdp{6.20}
& 66.06\std{11.89} & -1.64\stdp{4.21} & 59.96\stdp{8.45} \\

LwF~\cite{li2017learning}
& 56.87\stdp{9.15} & -1.35\stdp{2.81} & 53.28\stdp{6.07}
& 65.01\std{11.50} & -1.31\stdp{3.08} & 59.89\stdp{8.69} \\

L2P+~\cite{wang2022learning}
& 55.49\stdp{9.44} & 0.36\stdp{4.87} & 52.26\stdp{5.53}
& 62.80\std{13.09} & 1.03\stdp{8.27} & 58.13\stdp{7.61} \\

DualPrompt+~\cite{wang2022dualprompt}
& 58.08\stdp{8.84} & -0.16\stdp{7.36} & 53.90\stdp{6.91}
& 64.53\std{12.01} & 0.02\stdp{9.51} & 59.42\std{11.08} \\

S-Prompt+~\cite{wang2022s}
& 58.35\std{16.81} & 12.68\std{12.34} & 72.11\std{10.07}
& 62.02\std{20.48} & 13.77\std{14.29} & 76.43\std{11.26} \\

% VISTA
% & 59.25\std{14.87} & 10.02\stdp{9.21} & 81.49\stdp{4.59}
% & 64.12\std{17.90} & 9.32\stdp{9.61} & 87.43\stdp{3.56} \\
%
% FlyPrompt
% & 57.18\std{12.11} & -4.16\std{11.43} & 51.77\stdp{6.88}
% & 63.99\std{15.93} & -4.67\std{14.66} & 56.16\std{10.59} \\
% \midrule
\rowcolor{gray!10}
FlyGCL (Ours)
& \bfseries 76.27\stdp{2.52} & 3.33\stdp{5.50} & \bfseries 78.98\stdp{7.00}
& \bfseries 85.75\stdp{2.70} & -0.97\stdp{6.23} & \bfseries 84.90\stdp{7.60} \\
% VISTA+
% & 72.62\stdp{2.69} & 4.68\stdp{2.85} & {--}
% & 80.53\stdp{2.45} & 3.00\stdp{2.93} & {--} \\
\midrule
\multicolumn{7}{c}{\textit{Ego-only}} \\
\midrule
% Joint$_{2000}$
% & {--} & {--} & {--}
% & {--} & {--} & {--} \\
%
% Joint$_{1}$
% & 77.67\stdp{2.36} & {--} & {--}
% & 83.42\stdp{2.07} & {--} & {--} \\
% \midrule
SeqFT
& 58.74\stdp{6.40} & -1.58\stdp{8.49} & 55.03\stdp{6.34}
& 65.53\stdp{7.42} & -2.00\std{10.54} & 60.75\stdp{9.02} \\

ER~\cite{rolnick2019experience}
& 65.04\std{17.74} & -0.92\std{16.56} & 58.31\stdp{8.06}
& 69.37\std{22.33} & 2.06\std{21.76} & 63.36\std{10.65} \\

DER++~\cite{buzzega2020dark}
& 64.07\stdp{8.32} & -1.46\stdp{8.67} & 58.72\stdp{5.55}
& 70.26\stdp{9.96} & -0.89\std{11.08} & 64.25\stdp{7.31} \\

EWC~\cite{kirkpatrick2017overcoming}
& 57.62\std{13.90} & 0.11\std{14.34} & 54.60\stdp{8.50}
& 62.64\std{17.55} & 1.98\std{18.53} & 59.78\std{12.01} \\

LwF~\cite{li2017learning}
& \bfseries 66.92\stdp{6.46} & \bfseries -5.80\stdp{7.55} & 58.19\stdp{5.80}
& \bfseries 75.14\stdp{7.40} & \bfseries -6.76\stdp{9.31} & 64.67\stdp{7.98} \\

L2P+~\cite{wang2022learning}
& 58.30\stdp{7.16} & -1.83\std{10.72} & 54.51\stdp{5.09}
& 64.81\stdp{8.35} & -1.90\std{13.86} & 59.93\stdp{7.64} \\

DualPrompt+~\cite{wang2022dualprompt}
& 51.63\std{20.06} & 8.55\std{21.36} & 51.31\stdp{8.34}
& 51.69\std{30.21} & 15.71\std{34.57} & 54.47\std{12.67} \\

S-Prompt+~\cite{wang2022s}
& 53.82\stdp{2.43} & 6.56\stdp{2.67} & 53.91\stdp{4.16}
& 60.46\stdp{3.86} & 8.54\stdp{4.23} & 60.07\stdp{6.51} \\

% VISTA
% & 52.46\stdp{3.84} & 7.79\stdp{4.49} & 53.77\stdp{4.61}
% & 58.81\stdp{6.40} & 10.06\stdp{7.72} & 59.93\stdp{7.19} \\
%
% FlyPrompt
% & 49.94\std{16.50} & 8.98\std{15.49} & 52.15\stdp{8.56}
% & 50.94\std{23.14} & 16.01\std{22.83} & 55.74\std{12.78} \\
% \midrule
\rowcolor{gray!10}
FlyGCL (Ours)
& 65.87\stdp{8.90} & 11.98\stdp{8.95} & \bfseries 72.57\stdp{8.64}
& 73.13\stdp{9.75} & 8.74\stdp{9.77} & \bfseries 77.34\stdp{8.58} \\
\bottomrule
\end{tabular}
}
\end{table*}

% --- Table: egoexo-fitness (al) ---
%\input{tex/tab-ego-fitness-al}

% --- Table: Continual LIBERO-GCL (blurred boundary / si-blurry) ---

\begin{table*}[t]
\centering
\caption{Performance comparison on continual embodied vision-language-action learning benchmarks under the GCL setting (LIBERO). We report final average success rate $A_{\rm last}$ (\%, $\uparrow$), average anytime success rate $A_{\rm auc}$ (\%, $\uparrow$), forward transfer (FWT, \%, $\uparrow$), and negative backward transfer (NBT, \%, $\downarrow$).}
\label{tab:libero_gcl}
\renewcommand\arraystretch{1.10}
\setlength{\tabcolsep}{3.5pt}

\begin{minipage}[t]{0.485\textwidth}
\centering
\subcaption{LIBERO-Spatial} 
\resizebox{\linewidth}{!}{
\begin{tabular}{
    l
    S[table-format=3.2(3)]
    S[table-format=3.2(3)]
    S[table-format=3.2(3)]
    S[table-format=-3.2(3)]
}
\toprule
\textbf{Method}
& {$A_{\rm last}$ ($\uparrow$)}
& {$A_{\rm auc}$ ($\uparrow$)}
& {FWT ($\uparrow$)}
& {NBT ($\downarrow$)} \\
\midrule
SeqFT
& 23.67\std{1.02} & 44.06\std{0.92} & 63.17\std{2.79} & 50.58\std{2.03} \\
SeqLoRA
& 19.33\std{2.51} & 33.45\std{1.86} & 69.00\std{3.21} & 59.23\std{0.21} \\
PackNet~\cite{mallya2018packnet}
& 2.30\std{0.76} & 15.24\std{7.46} & 57.33\std{2.60} & 55.38\std{1.25} \\
ER~\cite{rolnick2019experience}
& 71.50\std{1.40} & 82.46\std{2.14} & 80.67\std{1.25} & 13.50\std{1.11} \\
EWC~\cite{kirkpatrick2017overcoming}
& 14.03\std{1.93} & 28.64\std{8.47} & 80.17\std{2.43} & 63.83\std{1.15} \\
LwF~\cite{li2017learning}
& 13.27\std{1.35} & 23.24\std{2.99} & 71.17\std{1.15} & 62.57\std{3.37} \\
L2P+~\cite{wang2022learning}
& 6.27\std{0.89} & 18.09\std{1.70} & 73.67\std{1.61} & 60.56\std{1.03} \\
DualPrompt+~\cite{wang2022dualprompt}
& 33.63\std{2.08} & 36.58\std{1.03} & \bfseries 88.44\std{1.50} & 60.03\std{2.00} \\
%CLARE~\cite{romer2026clare}
%& 81.73\std{1.75} & \bfseries 87.08\std{2.09} & 87.33\std{6.26} & -2.56\std{0.43} \\
% \midrule
\rowcolor{gray!10}
FlyGCL (Ours)
& \bfseries 83.12\std{1.42} & 81.58\std{0.64} & 82.75\std{3.84} & \bfseries -3.40\std{1.59} \\
\bottomrule
 \end{tabular}}
\end{minipage}
\hfill
\begin{minipage}[t]{0.485\textwidth}
\centering
\subcaption{LIBERO-Object}
\resizebox{\linewidth}{!}{
\begin{tabular}{
    l
    S[table-format=3.2(3)]
    S[table-format=3.2(3)]
    S[table-format=3.2(3)]
    S[table-format=-3.2(3)]
}
\toprule
\textbf{Method}
& {$A_{\rm last}$ ($\uparrow$)}
& {$A_{\rm auc}$ ($\uparrow$)}
& {FWT ($\uparrow$)}
& {NBT ($\downarrow$)} \\
\midrule
SeqFT
& 30.83\std{1.16} & 47.88\std{1.91} & \bfseries 94.67\std{1.97} & 43.10\std{2.47} \\
SeqLoRA
& 21.70\std{1.20} & 33.15\std{0.26} & 67.33\std{0.93} & 36.33\std{1.71} \\
PackNet~\cite{mallya2018packnet}
& 1.93\std{0.75} & 11.64\std{2.92} & 39.83\std{1.25} & 32.03\std{1.89} \\
ER~\cite{rolnick2019experience}
& \bfseries 89.50\std{0.50} & \bfseries 91.25\std{1.12} & 87.83\std{2.93} & 7.67\std{3.06} \\
EWC~\cite{kirkpatrick2017overcoming}
& 36.33\std{1.51} & 39.94\std{1.22} & 85.50\std{1.77} & 60.84\std{5.01} \\
LwF~\cite{li2017learning}
& 26.77\std{1.53} & 30.30\std{16.01} & 86.07\std{1.73} & 50.32\std{1.53} \\
L2P+~\cite{wang2022learning}
& 8.63\std{2.08} & 8.33\std{2.79} & 87.00\std{0.75} & 54.49\std{3.71} \\
DualPrompt+~\cite{wang2022dualprompt}
& 33.67\std{1.86} & 50.11\std{9.86} & 84.67\std{1.02} & 56.07\std{1.24} \\
%CLARE~\cite{romer2026clare}
%& 83.07\std{1.93} & 81.48\std{2.18} & 88.85\std{0.49} & \bfseries 0.78\std{0.40} \\
% \midrule
\rowcolor{gray!10}
FlyGCL (Ours)
& 86.11\std{1.56} & 86.03\std{1.01} & 87.73\std{1.30} & 1.40\std{1.55} \\
\bottomrule
\end{tabular}}
\end{minipage}

\vspace{2mm}

\begin{minipage}[t]{0.485\textwidth}
\centering
\subcaption{LIBERO-Goal}
\resizebox{\linewidth}{!}{
\begin{tabular}{
    l
    S[table-format=3.2(3)]
    S[table-format=3.2(3)]
    S[table-format=3.2(3)]
    S[table-format=-3.2(3)]
}
\toprule
\textbf{Method}
& {$A_{\rm last}$ ($\uparrow$)}
& {$A_{\rm auc}$ ($\uparrow$)}
& {FWT ($\uparrow$)}
& {NBT ($\downarrow$)} \\
\midrule
SeqFT
& 21.27\std{1.33} & 46.09\std{2.98} & 86.33\std{1.89} & 44.18\std{2.33} \\
SeqLoRA
& 11.80\std{0.65} & 23.27\std{1.73} & 77.83\std{0.75} & 37.66\std{3.00} \\
PackNet~\cite{mallya2018packnet}
& 3.90\std{0.44} & 20.00\std{0.56} & 30.00\std{1.93} & 9.23\std{1.07} \\
ER~\cite{rolnick2019experience}
& 74.73\std{1.32} & 82.01\std{1.11} & 84.33\std{1.31} & 29.68\std{2.25} \\
EWC~\cite{kirkpatrick2017overcoming}
& 12.67\std{1.29} & 22.82\std{1.27} & 88.83\std{0.29} & 43.84\std{0.76} \\
LwF~\cite{li2017learning}
& 13.13\std{2.42} & 23.39\std{0.92} & 92.17\std{1.58} & 41.03\std{2.57} \\
L2P+~\cite{wang2022learning}
& 10.03\std{1.02} & 12.73\std{1.83} & 89.00\std{3.18} & 30.24\std{3.03} \\
DualPrompt+~\cite{wang2022dualprompt}
& 35.67\std{1.81} & 46.92\std{2.64} & 90.67\std{4.50} & 40.88\std{2.66} \\
%CLARE~\cite{romer2026clare}
%& 91.47\std{0.81} & 86.89\std{0.38} & 91.37\std{0.64} & \bfseries -2.15\std{0.48} \\
% \midrule
\rowcolor{gray!10}
FlyGCL (Ours)
& \bfseries 94.50\std{1.95} & \bfseries 91.11\std{1.46} & \bfseries 92.25\std{2.33} & -1.51\std{0.41} \\
\bottomrule
\end{tabular}}
\end{minipage}
\hfill
\begin{minipage}[t]{0.485\textwidth}
\centering
\subcaption{LIBERO-Long}
\resizebox{\linewidth}{!}{
\begin{tabular}{
    l
    S[table-format=3.2(3)]
    S[table-format=3.2(3)]
    S[table-format=3.2(3)]
    S[table-format=-3.2(3)]
}
\toprule
\textbf{Method}
& {$A_{\rm last}$ ($\uparrow$)}
& {$A_{\rm auc}$ ($\uparrow$)}
& {FWT ($\uparrow$)}
& {NBT ($\downarrow$)} \\
\midrule
SeqFT
& 25.37\std{0.66} & 34.12\std{0.70} & \bfseries 80.67\std{1.41} & 51.31\std{3.08} \\
SeqLoRA
& 14.93\std{1.51} & 33.21\std{1.86} & 15.17\std{1.25} & 35.18\std{1.25} \\
PackNet~\cite{mallya2018packnet}
& 2.33\std{0.08} & 10.85\std{2.04} & 16.83\std{1.48} & 5.56\std{0.82} \\
ER~\cite{rolnick2019experience}
& 68.20\std{0.29} & 75.29\std{1.31} & 45.33\std{0.62} & 16.85\std{1.48} \\
EWC~\cite{kirkpatrick2017overcoming}
& 13.73\std{1.53} & 15.76\std{0.19} & 28.33\std{1.32} & 34.33\std{3.06} \\
LwF~\cite{li2017learning}
& 18.07\std{1.51} & 19.36\std{0.90} & 30.50\std{0.26} & 42.17\std{1.75} \\
L2P+~\cite{wang2022learning}
& 15.30\std{1.11} & 15.09\std{1.74} & 28.33\std{3.54} & 35.05\std{7.07} \\
DualPrompt+~\cite{wang2022dualprompt}
& 22.63\std{1.53} & 28.81\std{2.18} & 28.67\std{1.24} & 40.70\std{2.96} \\
%CLARE~\cite{romer2026clare}
%& 77.93\std{1.64} & 75.17\std{1.97} & 37.00\std{1.12} & \bfseries -1.75\std{1.46} \\
% \midrule
\rowcolor{gray!10}
FlyGCL (Ours)
& \bfseries 79.13\std{1.41} & \bfseries 79.12\std{1.11} & 80.15\std{2.42} & -1.15\std{1.01} \\
\bottomrule
\end{tabular}}
\end{minipage}
\end{table*}

% --- Table: Continual LIBERO (offline CL, extended) ---

\begin{table*}[t]
\centering
\caption{Performance comparison on continual embodied vision-language-action learning benchmarks under the offline CL setting (LIBERO). We report final average success rate $A_{\rm last}$ (\%, $\uparrow$), average anytime success rate $A_{\rm auc}$ (\%, $\uparrow$), forward transfer (FWT, \%, $\uparrow$), and negative backward transfer (NBT, \%, $\downarrow$).}
\label{tab:libero_offline}
\renewcommand\arraystretch{1.10}
\setlength{\tabcolsep}{3.2pt}

\begin{minipage}[t]{0.485\textwidth}
\centering
\subcaption{LIBERO-Spatial}
\resizebox{\linewidth}{!}{
\begin{tabular}{
    l
    S[table-format=3.2(3)]
    S[table-format=3.2(3)]
    S[table-format=3.2(3)]
    S[table-format=-3.2(3)]
}
\toprule
\textbf{Method}
& {$A_{\rm last}$ ($\uparrow$)}
& {$A_{\rm auc}$ ($\uparrow$)}
& {FWT ($\uparrow$)}
& {NBT ($\downarrow$)} \\
\midrule
SeqFT
& 9.27\std{0.35} & 26.53\std{0.41} & 88.07\std{0.93} & 86.88\std{0.75} \\
SeqLoRA
& 7.43\std{1.46} & 23.33\std{0.97} & 81.37\std{1.07} & 82.35\std{2.00} \\
PackNet~\cite{mallya2018packnet}
& 0.13\std{0.23} & 3.94\std{0.59} & 29.97\std{3.52} & 33.04\std{3.80} \\
ER~\cite{rolnick2019experience}
& 63.70\std{0.87} & 71.23\std{2.51} & \bfseries 88.40\std{0.46} & 24.76\std{3.72} \\
EWC~\cite{kirkpatrick2017overcoming}
&  8.87\std{0.42} & 25.87\std{0.53} & 86.43\std{1.56} & 85.49\std{1.57} \\
LwF~\cite{li2017learning}
&  6.27\std{1.23} & 21.39\std{1.52} & 74.30\std{2.26} & 74.90\std{1.30} \\
L2P+~\cite{wang2022learning}
&  6.30\std{2.95} & 22.34\std{3.47} & 81.47\std{4.42} & 83.16\std{1.52} \\
DualPrompt+~\cite{wang2022dualprompt}
& 14.10\std{1.55} & 33.34\std{0.93} & 84.67\std{2.04} & 72.71\std{2.17} \\
%CLARE~\cite{romer2026clare}
%& 83.57\std{1.06} & 83.91\std{0.67} & 84.23\std{0.67} & 0.35\std{0.03} \\
% \midrule
\rowcolor{gray!10}
FlyGCL (Ours)
& \bfseries 86.77\std{0.84} & \bfseries 86.61\std{0.26} & 86.03\std{0.99} & \bfseries -0.61\std{1.11} \\
\bottomrule
\end{tabular}}
\end{minipage}
\hfill
\begin{minipage}[t]{0.485\textwidth}
\centering
\subcaption{LIBERO-Object}
\resizebox{\linewidth}{!}{
\begin{tabular}{
    l
    S[table-format=3.2(3)]
    S[table-format=3.2(3)]
    S[table-format=3.2(3)]
    S[table-format=-3.2(3)]
}
\toprule
\textbf{Method}
& {$A_{\rm last}$ ($\uparrow$)}
& {$A_{\rm auc}$ ($\uparrow$)}
& {FWT ($\uparrow$)}
& {NBT ($\downarrow$)} \\
\midrule
SeqFT
& 20.90\std{2.00} & 36.41\std{0.28} & 96.03\std{0.40} & 84.68\std{0.32} \\
SeqLoRA
& 8.53\std{2.25} & 18.13\std{2.93} & 58.87\std{3.62} & 55.05\std{2.66} \\
PackNet~\cite{mallya2018packnet}
& 0.00\std{0.00} & 3.51\std{0.80} & 28.93\std{6.51} & 32.15\std{7.24} \\
ER~\cite{rolnick2019experience}
& 85.10\std{0.10} & 88.31\std{1.93} & 95.03\std{1.02} & 9.56\std{2.86} \\
EWC~\cite{kirkpatrick2017overcoming}
& 23.63\std{1.07} & 37.80\std{0.13} & \bfseries 96.93\std{0.68} & 84.08\std{1.11} \\
LwF~\cite{li2017learning}
& 19.27\std{2.21} & 36.03\std{0.32} & 95.57\std{0.64} & 84.61\std{0.66} \\
L2P+~\cite{wang2022learning}
&  8.17\std{0.75} & 25.41\std{0.21} & 84.53\std{1.21} & 82.83\std{2.54} \\
DualPrompt+~\cite{wang2022dualprompt}
&  9.93\std{0.31} & 30.00\std{1.02} & 83.83\std{1.66} & 75.03\std{2.01} \\
%CLARE~\cite{romer2026clare}
%& 82.67\std{2.46} & 82.67\std{2.37} & 82.93\std{2.03} & 0.28\std{0.51} \\
% \midrule
\rowcolor{gray!10}
FlyGCL (Ours)
& \bfseries 89.27\std{2.37} & \bfseries 88.82\std{1.55} & 88.50\std{1.95} & \bfseries -0.50\std{0.53} \\
\bottomrule
\end{tabular}}
\end{minipage}

\vspace{2mm}

\begin{minipage}[t]{0.485\textwidth}
\centering
\subcaption{LIBERO-Goal}
\resizebox{\linewidth}{!}{
\begin{tabular}{
    l
    S[table-format=3.2(3)]
    S[table-format=3.2(3)]
    S[table-format=3.2(3)]
    S[table-format=-3.2(3)]
}
\toprule
\textbf{Method}
& {$A_{\rm last}$ ($\uparrow$)}
& {$A_{\rm auc}$ ($\uparrow$)}
& {FWT ($\uparrow$)}
& {NBT ($\downarrow$)} \\
\midrule
SeqFT
& 8.93\std{0.42} & 27.31\std{0.28} & \bfseries 94.60\std{0.53} & 95.28\std{0.74} \\
SeqLoRA
& 1.70\std{1.31} & 17.07\std{1.42} & 70.67\std{1.76} & 76.49\std{2.22} \\
PackNet~\cite{mallya2018packnet}
& 0.00\std{0.00} & 4.68\std{0.88} & 28.90\std{5.48} & 32.10\std{6.10} \\
ER~\cite{rolnick2019experience}
& 59.93\std{5.58} & 73.60\std{3.94} & 93.83\std{0.45} & 28.24\std{5.80} \\
EWC~\cite{kirkpatrick2017overcoming}
&  8.77\std{0.25} & 27.17\std{0.43} & 94.53\std{1.24} & 95.32\std{1.08} \\
LwF~\cite{li2017learning}
&  7.97\std{0.46} & 25.18\std{0.29} & 90.40\std{0.10} & 92.10\std{0.44} \\
L2P+~\cite{wang2022learning}
&  8.80\std{1.06} & 26.29\std{1.19} & 90.07\std{1.63} & 90.32\std{1.81} \\
DualPrompt+~\cite{wang2022dualprompt}
& 25.73\std{2.54} & 39.88\std{1.96} & 91.03\std{1.39} & 70.57\std{3.72} \\
%CLARE~\cite{romer2026clare}
%& 89.17\std{0.81} & 89.18\std{1.16} & 88.90\std{1.84} & -0.28\std{0.97} \\
% \midrule
\rowcolor{gray!10}
FlyGCL (Ours)
& \bfseries 93.40\std{0.72} & \bfseries 93.17\std{0.59} & 93.20\std{1.04} & \bfseries -0.03\std{0.56} \\
\bottomrule
\end{tabular}}
\end{minipage}
\hfill
\begin{minipage}[t]{0.485\textwidth}
\centering
\subcaption{LIBERO-Long}
\resizebox{\linewidth}{!}{
\begin{tabular}{
    l
    S[table-format=3.2(3)]
    S[table-format=3.2(3)]
    S[table-format=3.2(3)]
    S[table-format=-3.2(3)]
}
\toprule
\textbf{Method}
& {$A_{\rm last}$ ($\uparrow$)}
& {$A_{\rm auc}$ ($\uparrow$)}
& {FWT ($\uparrow$)}
& {NBT ($\downarrow$)} \\
\midrule
SeqFT
& 8.63\std{0.15} & 21.77\std{0.25} & 72.97\std{1.43} & 71.48\std{1.76} \\
SeqLoRA
& 6.73\std{0.21} & 17.98\std{0.55} & 64.03\std{0.83} & 63.62\std{1.08} \\
PackNet~\cite{mallya2018packnet}
& 0.00\std{0.00} & 2.10\std{0.55} & 17.70\std{4.39} & 19.67\std{4.87} \\
ER~\cite{rolnick2019experience}
& 43.03\std{0.76} & 56.56\std{1.54} & 71.27\std{2.19} & 21.93\std{5.17} \\
EWC~\cite{kirkpatrick2017overcoming}
&  8.13\std{0.15} & 21.53\std{0.16} & 73.37\std{1.01} & 72.48\std{1.11} \\
LwF~\cite{li2017learning}
&  8.20\std{0.10} & 21.29\std{0.53} & 72.37\std{2.72} & 71.30\std{3.08} \\
L2P+~\cite{wang2022learning}
& 20.97\std{1.71} & 33.38\std{2.70} & 71.37\std{1.46} & 50.02\std{2.93} \\
DualPrompt+~\cite{wang2022dualprompt}
& 10.77\std{1.08} & 25.71\std{0.91} & 68.90\std{0.92} & 59.39\std{1.69} \\
%CLARE~\cite{romer2026clare}
%& 74.07\std{0.21} & 74.07\std{0.31} & 74.27\std{0.40} & 0.18\std{0.95} \\
% \midrule
\rowcolor{gray!10}
FlyGCL (Ours)
& \bfseries 76.93\std{0.40} & \bfseries 76.92\std{0.47} & \bfseries 76.93\std{1.14} & \bfseries 0.02\std{2.21} \\
\bottomrule
\end{tabular}}
\end{minipage}
\end{table*}

\end{document}